\documentclass{article}

\usepackage[eandd,preprint]{neurips_2026}

\usepackage[utf8]{inputenc}
\usepackage[T1]{fontenc}
\usepackage{hyperref}
\usepackage{url}
\usepackage{booktabs}
\usepackage{tabularx}
\usepackage{subcaption}   
\usepackage{placeins}

\usepackage{amsfonts}
\usepackage{nicefrac}
\usepackage{microtype}
\usepackage{xcolor}
\usepackage{graphicx}
\usepackage{enumitem}
\usepackage{xspace}
\usepackage{algorithm}
\usepackage{algpseudocode}
\usepackage{tcolorbox}
\usepackage{tikz}  
\usetikzlibrary{positioning}
\usetikzlibrary{calc}
\usepackage{amsfonts}
\usepackage{amsmath}
\usepackage{amsthm}
\usepackage{enumitem}
\usetikzlibrary{positioning, arrows.meta}

\usepackage{amsmath, amssymb, amsthm}
\newtheorem{theorem}{Theorem}[section]

\newtheorem{proposition}[theorem]{Proposition}

\theoremstyle{definition}

\newcommand{\sysname}{\textsc{Isomorph}\xspace}

\usepackage{xcolor}

\title{ISOMORPH: A Supply Chain Digital Twin for\\
       Simulation, Dataset Generation, and Forecasting Benchmarks}

\author{
  Zhizhen Zhang\thanks{Authors contributed equally} \\
  University of Massachusetts Amherst \\
  \texttt{zhizhenzhang@umass.edu} \\
  \And
  Hyemin Gu\footnotemark[1] \\
  University of Massachusetts Amherst \\
  \texttt{hgu@umass.edu} \\
  \And
  Benjamin J.\ Zhang \\
  University of North Carolina \\
  \texttt{bjz@unc.edu} \\
  \And
  Daniel Elenius \\
  SRI International \\
  \texttt{daniel.elenius@sri.com} \\
  \And
  Michael Tyrrell \\
  SRI International \\
  \texttt{michael.tyrrell@sri.com} \\
  \And
  Theo J.\ Bourdais \\
  California Institute of Technology \\
  \texttt{theo.bourdais@caltech.edu} \\
  \And
  Houman Owhadi \\
  California Institute of Technology \\
  \texttt{owhadi@caltech.edu} \\
  \And
  Markos A.\ Katsoulakis\footnotemark[2]\thanks{Corresponding authors} \\
  University of Massachusetts Amherst \\
  \texttt{markos@umass.edu} \\
  \And
  Tuhin Sahai\footnotemark[2] \\
  SRI International \\
  \texttt{tuhin.sahai@sri.com} \\
}

\begin{document}
\maketitle

\begin{abstract}
Open time-series forecasting (TSF) benchmarks cover
retail, energy, weather, and traffic, but supply-chain
logistics remains underserved. We introduce \sysname{},
the first public digital twin of a multi-echelon
logistics network with fully interpretable,
user-configurable parameters and modular topology,
demand process, and control rules. The simulator
advances a directed routing graph in discrete time:
demand arrives at the destination, is served from stock
or recorded as backlog, and triggers replenishment
through the network. 
By construction, the state vector
tracks per-node on-hand inventory together with
outstanding orders, in-transit shipments, and a smoothed
demand estimate across the network, so the dynamics
close as a Markov chain on a high-dimensional but
tractable state space whose transition kernel acts
linearly on the empirical distribution of the state. 
The released data reproduces the
bullwhip effect---the
standard empirical signature of supply-chain
dynamics---at empirically consistent
magnitudes, and three conservation laws structurally encoded in the Markov chain serve as verification tools
when users extend the simulator with their own control
rules. We release datasets at two catalogue scales ($C{=}50$
and $C{=}200$) with a $33$-rollout scenario library at $C{=}50$,
exhibiting dynamics largely absent from fixed TSF
benchmarks: variance amplification, cascading
bottlenecks, regime shifts, and cross-channel coupling
through shared macro shocks.
Zero-shot evaluation of three foundation models (Chronos,
Moirai, TimesFM) against three in-domain-trained baselines
(ARIMA, ETS, PatchTST) spans four target observables: demand
and three network-dependent quantities: backlog, fill rate,
and edge utilization. Comparison with ETTh1, Electricity, and
Weather under a common protocol shows that the \sysname{}
targets introduce forecasting regimes that differ from those
represented in standard real-world TSF benchmarks. This
positions \sysname{} as a complementary, regenerable
logistics-domain benchmark with explicit network-state targets
and controlled variation in the underlying data-generating
system. The same pairing produces forecast confidence bands
across an ensemble of scenario configurations, providing
a form of forward UQ from parameter uncertainty.
This demonstrates how \sysname{} can serve as a configurable
digital-twin testbed for evaluating and developing foundation
models' uncertainty quantification under controlled variation
in the underlying system configuration.
Code (MIT) is released at 
\url{https://github.com/tuhinsahai/ISOMORPH}. An interactive demo is available at
\url{https://huggingface.co/spaces/HyeminGu/ISOMORPH-demo},
where users can stress-test the 13-node U.S. network
by selecting from three disruption scenarios and
observing the resulting dynamics in real time. 

\end{abstract}

\paragraph{Acknowledgements.}
This material is based upon work of the authors supported by the Defense Advanced Research Projects Agency (DARPA) under Agreement No. HR00112590112. Approved for public release; distribution is unlimited.

\tableofcontents

\section{Introduction}
\label{sec:intro}
 
Public time-series forecasting (TSF) benchmarks cover retail, energy,
weather, and traffic fairly extensively. However, supply chain logistics, one of the
most critical applied forecasting domains, remains underserved: no public
benchmark pairs demand series with the network state that produces them. The closest existing resources address the domain in a limited manner. SupplyGraph
\citep{wasi2024supplygraph} provides eight months of per-product
demand with a static product relationship graph from a single
company. However, it has no logistics network characteristics that include inventory state, transportation edges, or capacities. Deepbullwhip \citep{arief2026deepbullwhip} simulates a four-stage serial chain to study how demand swings amplify upstream,
but releases no supply chain forecasting dataset. M5 \citep{m5comp} aggregates retail demand hierarchically by store and category. It records what customers bought, without the warehouses, shipments, or inventory behind those sales. Across all
three, the same piece is missing: demand series paired with the
topology, capacities, and control decisions that determine whether
the network can deliver.

Therefore, to close this gap, one needs a different kind
of resource. A real-world dataset, even from a willing
operator, observes one realization of a logistics
network. Capacities, routing rules, and control decisions
sit implicitly inside the recorded series; they cannot be
inspected, modified, or replayed under alternative
settings. A simulator-based dataset makes those rules
explicit: the network, the demand process, and the
control policy are part of the deliverable, and the
dataset is one realization of a configuration that other
users can reproduce, rescale, or modify.
Adjacent domains have used this perspective.
BuildingsBench \citep{emami2023buildingsbench} pairs an
EnergyPlus simulator with a curated load-forecasting
benchmark for buildings, the most closely related work.
No equivalent exists for supply-chain logistics.

This paper introduces \sysname{} (\emph{\textbf{I}nterpretable \textbf{S}upply-chain \textbf{O}perations: a \textbf{M}arkov \textbf{O}perator for \textbf{R}outing and \textbf{P}acking in \textbf{H}igh dimensions}), a digital twin of a multi echelon logistics network. The user supplies a directed graph of source, intermediate, and destination nodes with edge transit times and per edge container capacities; the simulator runs the network
forward at a user chosen time resolution, with shipments routed by Dijkstra, packed greedily into containers, and replenished under an $(s,S)$ policy---each stocking location orders up to a target level $S$ as soon as its on-hand inventory falls to or below a reorder point $s$ (\S\ref{app:replenish}). The simulator's parameters fall into two groups: \emph{structural parameters} (topology, catalogue size $C$,
time resolution, and horizon), which fix what dataset is being generated; and \emph{scenario knobs} (demand structure, inventory, edge transport), which
govern how the network operates within that structure. We release a family of datasets generated from one set
of structural parameters---a 13-node U.S.\ network at
daily resolution (three sources, nine warehouses across
five tiers, and a destination served by two last-mile
edges) at two cardinalities ($C{=}50$ and
$C{=}200$)---together with rollouts from perturbed
scenario knobs.
The simulator is fully open-source and modular: any user can swap the topology, time resolution, demand process, or control rules to generate releases tailored to their own setting. 
Each release includes the full state trajectory of the network, covering per item demand and on hand inventory as well as per edge shipments and utilization, so quantities like per edge saturation events and destination fill rate drops can be recovered directly from the released data without rerunning the simulator.

A central challenge in simulating supply-chain dynamics
is that the system mixes continuously evolving variables
(inventory levels, smoothed demand) with abrupt regime
changes (demand shocks, capacity bottlenecks,
replenishment triggers), producing nonlinear,
history-dependent behavior. Without the right
state-space formulation, the Markov property fails and
conservation guarantees become unverifiable. Our
approach is to formulate the network as a Markov chain
on the right state space: unlike existing resources that
are limited to per-node demand or
inventory~\citep{wasi2024supplygraph, naik2025bullode},
the state vector carries outstanding orders, in-transit
shipments, and a smoothed demand estimate alongside
per-node inventory. The resulting state dimension is at least 
$C(3|\mathcal{N}|+1)$ (for $|\mathcal{N}|$ nodes and
catalogue size $C$). This closes the dynamics to be Markovian with a
transition kernel that acts linearly on the state
distribution, and structurally encodes three conservation
laws---the discrete realization of coupled transport and
queueing equations---directly into the transition
function (\S\ref{sec:state-space-consequences}),
enabling tractable simulation over arbitrary horizons,
topologies, catalogue sizes, and scenario knobs. The full ISOMORPH release provides datasets at catalogue
sizes $C=50$ and $C=200$, corresponding to more than
$2{,}000$ and $8{,}000$  state dimensions,
respectively.

To demonstrate \sysname's use as a forecasting benchmark,
we evaluate three pretrained TSF foundation models
(Chronos, Moirai, TimesFM) in a zero-shot setting, together
with three in-domain-trained baselines (ARIMA, ETS,
PatchTST) on four targets: demand, and three network-state quantities that the
released trajectories expose---backlog at the destination, fill rate, and
edge utilization. The first three are $50$-channel item-level series, the last
a $16$-channel edge-level series (\S\ref{sec:foundation-protocol}). Forecasting performance varies substantially across these targets
(Table~\ref{tab:networkstate-summary}).

Against ETTh1, Electricity, and Weather under the same protocol,
the \sysname{} targets show higher normalized forecast errors for several
network-state quantities while edge utilization remains comparable to
ETTh1 (Table~\ref{tab:crossbench-ci}). These results indicate that
\sysname{} provides forecasting regimes that differ from those represented
in standard real-world TSF benchmarks. This complements existing benchmark
suites by extending their coverage to a regenerable logistics domain with
explicit network-state targets and controlled variation in the underlying
data-generating system.

\paragraph{Contributions}
\begin{enumerate}[leftmargin=1.4em,itemsep=2pt,topsep=2pt]

  \item \textbf{Simulator.} We introduce \sysname{} (\S\ref{sec:twin}), the first
public digital twin of a multi-echelon logistics network
designed for TSF benchmarking, with every parameter--- demand, inventory, transport, routing, and control --- physically interpretable and user-configurable. The simulator is open-source and modular: topology, time resolution, demand process, replenishment policy, and routing rule are swappable without modifying the simulator code. Every parameter is documented with defaults in App.~\ref{app:param-values}. The simulator closely captures the day-to-day operational logic of large supply chain entities.

  \item \textbf{Empirical and structural validation.} We validate the released data on two complementary axes. \emph{Empirically}, the data reproduce the bullwhip effect at amplification levels within the cross-industry distribution reported by \citet{cachon2007search}. \emph{Structurally}, three pathwise conservation laws -- per-node mass, global mass, and backlog conservation (Proposition~\ref{prop:conservation}) -- are encoded in the transition function by the right state space design. When users add their own control rules to the simulator, these conservation laws---or adapted versions---can verify fidelity of the modified implementation.

  \item \textbf{Released datasets with complex supply-chain dynamics.} We release two baseline datasets generated by \sysname{} (\S\ref{sec:datasets}) at cardinalities $C{=}50$ and $C{=}200$ ($T{=}7{,}300$ time units at daily resolution, i.e.\ $20$ years, with chronological $70/15/15$ splits and loaders)  on a
13-node U.S. network, together with a $33$-rollout scenario library at
$C{=}50$ built around that baseline, of which $29$ form the forecasting
evaluation cohort and $20$ the forward-UQ ensemble. The generated trajectories exhibit coupled
supply-chain dynamics, including variance amplification along the network
(bullwhip), congestion that concentrates on individual edges as capacity is
tightened, long-horizon regime shifts in demand, and cross-channel coupling
through a shared macro shock $G(t)$ that simultaneously affects all items
and propagates through the network.

   \item \textbf{Forward UQ on supply-chain dynamics.} The
  pairing of \sysname{} with zero-shot foundation models
  enables forward uncertainty quantification (UQ) under
  perturbations that standard TSF datasets cannot expose:
  \emph{parameter uncertainty} (numerical values of
  scenario knobs such as demand-shock intensities and
  drift coefficients) and \emph{model uncertainty}
  (structural choices such as distributional families
  and replenishment policies). We instantiate the first, performing forward UQ
  over an ensemble of $K{=}20$ scenario
  configurations drawn from the cohort (\S\ref{sec:uq}) and assessing the
  resulting predictive intervals with empirical coverage and
  weighted CRPS on all four targets. This establishes \sysname{} as a configurable digital-twin testbed for evaluating and developing foundation models' uncertainty quantification under controlled variation in the underlying system configuration.

\end{enumerate}

\section{Related work}
\label{sec:related}

\subsection{Time-series forecasting benchmarks}
\label{sec:related-tsf}

Standard TSF benchmarks primarily consist of fixed real-world datasets spanning domains such as electricity, weather, retail, and traffic. Representative suites include \textbf{ETT}~\citep{zhou2021informer}, \textbf{M4} and \textbf{M5}~\citep{m4comp,m5comp}, \textbf{Monash}~\citep{godahewa2021monash}, and the recent \textbf{GIFT-Eval}~\citep{aksu2024gifteval}. While these benchmarks provide broad coverage of forecasting tasks, they expose datasets rather than the underlying data-generating processes, limiting the ability to systematically regenerate series under controlled changes to the generating conditions. Synthetic approaches such as \textbf{KernelSynth}~\citep{ansari2024chronos} and \textbf{ForecastPFN}~\citep{dooley2023forecastpfn} instead enable parameterized generation of time series, providing greater control over the characteristics of the resulting data. However, these approaches use generic statistical or learned generative mechanisms rather than an explicit simulator of a particular operational system. \sysname{} complements these existing benchmarks by providing a regenerable forecasting benchmark generated from an explicit multi-echelon supply-chain network, with controllable parameters governing the underlying system dynamics.

\subsection{Supply-chain simulators and datasets}
\label{sec:related-sc}

Publicly available supply-chain time-series datasets remain limited. \textbf{SupplyGraph}~\citep{wasi2024supplygraph} provides SKU-level observations from a consumer-goods company across four supply-chain stages---Sales Order, Delivery to Distributor, Factory Issue, and Production---covering \(C=41\) product categories. SupplyGraph is a static observational release: its product-association graph encodes shared product groups, sub-groups, plants, or storage locations rather than a transportation network with explicit transit times and capacities, and its underlying data-generating process is not exposed. Consequently, its generating mechanisms cannot be directly compared one-to-one with those of \sysname{}. Instead, we use SupplyGraph as a real-world reference to assess the statistical coverage of \sysname{}, comparing observable distributional and temporal properties in App.~\ref{app:supplygraph}. This comparison serves as a sanity check rather than a validation of identical dynamics: it asks whether the scenario library of \sysname{} can reproduce statistical regimes observed in a real supply chain. At the same time, \sysname{} is complementary to such observational releases, as its controllable generating process produces both time-series observations and the underlying network-state trajectories and system variables that are unavailable in SupplyGraph.

Open-source supply-chain simulators provide controllable alternatives to
observational datasets, but target different simulation and learning
objectives. \textbf{deepbullwhip}~\citep{arief2026deepbullwhip} provides a
serial-chain simulator and bullwhip benchmark, \textbf{OR-Gym}~\citep{hubbs2020orgym}
provides a serial-chain environment for reinforcement-learning-based
inventory control, and \textbf{BULL-ODE}~\citep{naik2025bullode} studies demand,
inventory, and ordering dynamics in a single-node continuous-time setting.
In contrast, \sysname{} is designed to generate  multi-item,
multi-echelon network trajectories for time-series forecasting research,
with a configurable directed network and user-controlled demand and network
parameters. Forecasting targets can be defined at the level of individual
items and/or network nodes and edges, enabling forecasting tasks across
different levels of the supply-chain system. In our forecasting experiments,
demand, backlog, and fill rate are modeled as 50-channel item-level targets,
while edge utilization is modeled as a 16-channel edge-level target
(Tables~\ref{tab:combined-ci-demand}, \ref{tab:combined-ci-backlog},
\ref{tab:combined-ci-fill_rate}, and \ref{tab:combined-ci-edge_util}). Table~\ref{tab:related-general}
compares these frameworks with \sysname{} in terms of primary use, topology,
catalogue size, time model, and structural closure. The comparison highlights
that \sysname{} supports a configurable directed network, multiple catalogue
items, user-chosen temporal resolution, and a closed state representation
over nodes, edges, and items, whereas the released configurations of the
related frameworks consider more restricted network and catalogue settings. Simpler supply-chain configurations can be obtained as reductions of
\sysname{}, allowing direct comparison with the serial-chain, single-item settings of deepbullwhip and OR-Gym. The full
structural comparison and the corresponding reductions are given in
Appendix~\ref{app:related-detailed}.

\begin{table}[!htb]
\centering
\footnotesize
\caption{High-level comparison of \sysname with three related open
supply-chain simulation and learning frameworks. Entries describe the default
or headline released instantiation of each work; the full structural
comparison (demand process, ordering and control, lead times, capacity and
routing, analytical structure, scenario generation) is
Table~\ref{tab:related-detailed} in App.~\ref{app:related-detailed}, which
also reports the full state vector of each system.
``Structural closure'' asks whether the released state is rich enough to make
the dynamics Markovian and represent the system's material-flow balances.}
\label{tab:related-general}
\setlength{\tabcolsep}{3pt}
\renewcommand{\arraystretch}{1.1}
\scriptsize
\begin{tabularx}{\linewidth}{@{}l *{4}{>{\raggedright\arraybackslash}X}@{}}
\toprule
Axis
 & \textbf{\sysname} (this work)
 & \textbf{BULL-ODE}~\citep{naik2025bullode}
 & \textbf{deepbullwhip}~\citep{arief2026deepbullwhip}
 & \textbf{OR-Gym}~\citep{hubbs2020orgym} \\
\midrule
Primary use
 & Simulator $+$ forecasting dataset
 & Learning study (NODE vs.\ UDE)
 & Simulator $+$ bullwhip benchmark
 & RL environment for inventory control \\
Topology
 & \textbf{Directed graph} (user-configurable); released 13-node U.S.\ network with 16 edges
 & \textbf{Single node}
 & \textbf{Serial chain}; v0.3.0 extends to DAGs via \texttt{SupplyChainGraph}
 & \textbf{Serial chain} \\
Catalogue $C$
 & Multi-item; released at $\mathbf{C{=}50}$ and $\mathbf{C{=}200}$
 & $\mathbf{C{=}1}$
 & $\mathbf{C{=}1}$
 & $\mathbf{C{=}1}$ \\
Time model
 & Discrete, user-chosen frequency; released daily with $T{=}7{,}300$
 & Continuous-time ODE ($\Delta t{=}0.2$); $T{=}30$
 & Discrete, weekly; $T{\le}5{,}200$ tested
 & Discrete, daily; $T{=}30$ \\
Structural closure
 & Markov over nodes, edges, and items, with routing and packing internal to
   the transition; supports three pathwise conservation identities
   (Prop.~\ref{prop:conservation})
 & Markov on a single node in continuous time; scalar conservation
   $\dot I = O - D$
 & Markov over nodes; per-node material balances and a delta-method CV
   filtering result
 & Markov over nodes; per-node material balances, no explicit conservation
   framing \\
\bottomrule
\end{tabularx}
\end{table}

\subsection{Related digital twins as ML dataset generators}
\label{sec:related-dt}

Mechanistic simulators, that is, simulators that advance a system by applying its
governing rules step by step, rather than sampling from a fitted or learned
law,  have been released as public machine-learning benchmarks in other
domains. We restrict attention to those whose task is forecasting the future
state of an operating system, which is the setting of this paper;
simulator-generated benchmarks for surrogate modeling, where the target is a
solution operator rather than a trajectory, are a separate line of work.
\textbf{BuildingsBench}~\citep{emami2023buildingsbench}
couples the EnergyPlus building simulator with a benchmark for short-term
electricity load forecasting. \textbf{ClimateBench}~\citep{watsonparris2022climatebench},
in contrast, evaluates machine-learning models trained to emulate outputs of
CMIP6 climate models, representing a surrogate-modeling task rather than
time-series forecasting. \sysname{} applies this simulator--benchmark
paradigm to multi-echelon supply chains, while additionally enabling
zero-shot foundation-model evaluation under controlled perturbations of the
underlying system parameters. To our knowledge, \sysname{} is the first
framework to provide this combination of controllable multi-echelon logistics
simulation, forecasting benchmark generation, and parameter-perturbation
evaluation.

\section{The \sysname digital twin}
\label{sec:twin}

\subsection{Overview}
\sysname simulates the flow of \emph{multiple item types} (a catalogue of~$C$ goods)
through \emph{a directed network} of factories, intermediate warehouses, and a
customer-facing destination, advancing the state of every location and link
forward in \emph{discrete time} at user-prescribed time resolution
(day/hour/minute).
In a single rollout, random customer demand arrives at the
destination each step, is served from available stock, and
triggers replenishment through the network under an $(s,S)$ policy---order up
to level $S$ when on-hand inventory reaches the reorder point $s$
(\S\ref{app:replenish})---and the cycle repeats.
At the end of each step~$t$ the simulator collects the state
of the entire network---stock levels at each location for each
of the~$C$ items, unfilled orders, in-transit shipments, and a
smoothed demand estimate---into a single vector~$\xi_t$
(formal definition in \S\ref{app:formal},
Eq.~\ref{eq:state-vector}).

The network is a user-supplied directed graph
$\mathcal{G} = (\mathcal{N}, \mathcal{E})$ whose locations are
embedded in two-dimensional geography (e.g.\ U.S.\ cities);
however, the metric on this graph is defined by user-set
integer transit times~$\tau_e$ on each link, not by geographic
distance.
Locations split into three roles: factories
$\mathcal{N}_{\mathrm{src}}$ with unbounded upstream supply,
intermediate warehouses $\mathcal{N}_{\mathrm{int}}$, and a
customer-facing destination~$d^\star$.
Each directed link $e = (u,v)$ carries, in addition
to~$\tau_e$, a per-container volume~$V_e$ and a per-step
container count~$K_e$, giving volume capacity $K_e V_e$ per
step.
Dijkstra's algorithm routes shipments on the weight
$\tau_e / (K_e V_e)$, so that capacity enters path selection alongside transit
time, and the user controls both the topology and the effective
distances by choosing which locations to connect and what
$\tau_e$, $K_e$ and $V_e$ to assign
(Figure~\ref{fig:pgm}(a);
released values in Appendix~\ref{app:params:edges},
Table~\ref{tab:edge-params}).

The dynamics is modeled as a \emph{Markov chain}: the sequence
of state variables $\{\xi_t\}$ satisfies
$\xi_{t+1} = \Psi(\xi_t, y_t, L_t)$,
where~$\Psi$ is a deterministic transition function driven by
two random inputs---exogenous Poisson customer demand~$y_t$
and Gaussian factory lead times~$L_t$, the latter drawn
conditioned on~$\xi_t$ at source locations whose stock falls
below a reorder threshold. 
The demand~$y_t$ is generated from a five-component intensity
combining seasonality, drift, bursts, and a shared macro shock
(\S\ref{app:demand}). Figure~\ref{fig:pgm}(b) shows the Markov chain unrolled
over three time slices.
The physical rules that constitute~$\Psi$ are organized into
two groups: replenishment and service rules
(\S\ref{app:replenish}), which govern when each
location reorders stock and how the destination serves or
backlogs demand; and inter-warehouse flows
(\S\ref{app:flows}), which specify source arrivals,
Dijkstra-based dispatch, and inter-warehouse pulls, see Figure~\ref{fig:psi-flow}.
The state space, the random inputs, the transition map~$\Psi$,
and the resulting transition kernel are specified in
\S\ref{app:formal}
(Eqs.~\ref{eq:one-step-update}--\ref{eq:kernel}).

\begin{figure}[t]
\centering
\begin{minipage}[c]{0.36\linewidth}
\centering
\IfFileExists{figures/network_map.pdf}
  {\includegraphics[width=\linewidth]{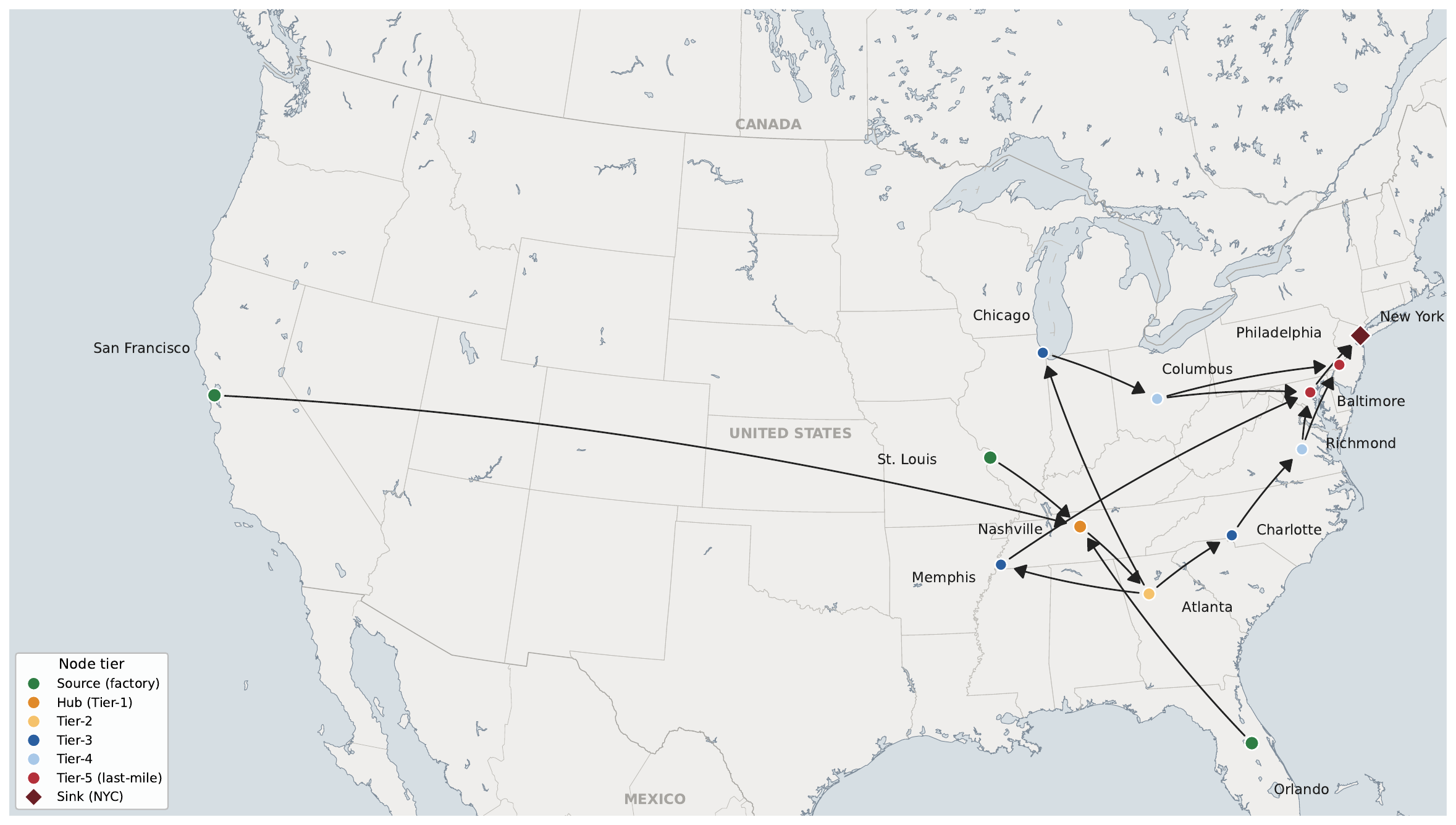}}
  {\fbox{\parbox[c][3.6cm][c]{0.95\linewidth}{\centering\itshape\small
   [network\_map.pdf placeholder:\\3 sources $\to$ Nashville $\to$ Atlanta\\
   $\to$ 5-tier graph $\to$ Phila/Balt $\to$ NYC]}}}\\[2pt]
{\small (a) Released network}
\end{minipage}\hfill
\begin{minipage}[c]{0.62\linewidth}
\centering
\resizebox{\linewidth}{!}{
\begin{tikzpicture}[
  >={Stealth[length=2mm]},
  font=\small,
  rv/.style={circle, draw=violet!60!black, fill=violet!20, line width=0.4pt,
             minimum size=0.65cm, inner sep=0pt},
  det/.style={circle, draw=orange!70!black, fill=orange!30, line width=0.4pt,
              minimum size=0.65cm, inner sep=0pt},
  src/.style={circle, draw=black!50, fill=gray!30, line width=0.4pt,
              minimum size=0.5cm, inner sep=0pt},
  hub/.style={circle, draw=blue!60!black, fill=blue!20, line width=0.4pt,
              minimum size=0.5cm, inner sep=0pt},
  ware/.style={circle, draw=blue!60!black, fill=blue!20, line width=0.4pt,
               minimum size=0.5cm, inner sep=0pt},
  dest/.style={circle, draw=blue!80!black, fill=blue!50, line width=0.4pt,
               minimum size=0.55cm, inner sep=0pt},
  edge/.style={->, line width=0.4pt, color=black!60},
  inputedge/.style={->, line width=0.4pt, color=violet!70!black},
  trans/.style={->, line width=0.7pt, color=black!70},
]

\draw[draw=blue!50!black, dashed, rounded corners=4pt, line width=0.4pt]
  (-4.7, -0.8) rectangle (-1.3, 2.2);
\draw[draw=blue!50!black, dashed, rounded corners=4pt, line width=0.4pt]
  ( 0.1, -0.8) rectangle ( 3.5, 2.2);
\draw[draw=blue!50!black, dashed, rounded corners=4pt, line width=0.4pt]
  ( 4.9, -0.8) rectangle ( 8.3, 2.2);

\node[font=\footnotesize, color=blue!60!black, anchor=north west]
  at (-4.6, 2.1) {\textit{$\xi(t-1)$}};
\node[font=\footnotesize, color=blue!60!black, anchor=north west]
  at ( 0.2, 2.1) {\textit{$\xi(t)$}};
\node[font=\footnotesize, color=blue!60!black, anchor=north west]
  at ( 5.0, 2.1) {\textit{$\xi(t+1)$}};

\node[font=\small, color=black!60] at (-3.0, 4.6) {\textit{Step $t-1$}};
\node[font=\small, color=black!60] at ( 1.8, 4.6) {\textit{Step $t$}};
\node[font=\small, color=black!60] at ( 6.6, 4.6) {\textit{Step $t+1$}};

\node[det] (lam0) at (-3.9, 3.6) {$\lambda$};
\node[rv]  (y0)   at (-3.0, 3.6) {$y$};
\node[rv]  (l0)   at (-2.1, 3.6) {$L$};
\draw[edge] (lam0) -- (y0);

\node[det] (lam1) at ( 0.9, 3.6) {$\lambda$};
\node[rv]  (y1)   at ( 1.8, 3.6) {$y$};
\node[rv]  (l1)   at ( 2.7, 3.6) {$L$};
\draw[edge] (lam1) -- (y1);

\node[det] (lam2) at ( 5.7, 3.6) {$\lambda$};
\node[rv]  (y2)   at ( 6.6, 3.6) {$y$};
\node[rv]  (l2)   at ( 7.5, 3.6) {$L$};
\draw[edge] (lam2) -- (y2);


\node[src]  (s0) at (-4.2, 1.5) {};
\node[hub]  (h0) at (-3.4, 1.5) {};
\node[ware] (w0) at (-2.6, 1.5) {};
\node[dest] (d0) at (-2.0, 0.2) {};
\node[font=\scriptsize, color=black!70] at (-4.2, 1.0) {src};
\node[font=\scriptsize, color=black!70] at (-3.4, 1.0) {hub};
\node[font=\scriptsize, color=black!70] at (-2.6, 1.0) {w};
\node[font=\scriptsize, color=black!70] at (-2.0,-0.3) {$d^\star$};
\draw[edge] (s0) -- (h0);
\draw[edge] (h0) -- (w0);
\draw[edge] (w0) -- (d0);

\node[src]  (s1) at ( 0.6, 1.5) {};
\node[hub]  (h1) at ( 1.4, 1.5) {};
\node[ware] (w1) at ( 2.2, 1.5) {};
\node[dest] (d1) at ( 2.8, 0.2) {};
\node[font=\scriptsize, color=black!70] at ( 0.6, 1.0) {src};
\node[font=\scriptsize, color=black!70] at ( 1.4, 1.0) {hub};
\node[font=\scriptsize, color=black!70] at ( 2.2, 1.0) {w};
\node[font=\scriptsize, color=black!70] at ( 2.8,-0.3) {$d^\star$};
\draw[edge] (s1) -- (h1);
\draw[edge] (h1) -- (w1);
\draw[edge] (w1) -- (d1);

\node[src]  (s2) at ( 5.4, 1.5) {};
\node[hub]  (h2) at ( 6.2, 1.5) {};
\node[ware] (w2) at ( 7.0, 1.5) {};
\node[dest] (d2) at ( 7.6, 0.2) {};
\node[font=\scriptsize, color=black!70] at ( 5.4, 1.0) {src};
\node[font=\scriptsize, color=black!70] at ( 6.2, 1.0) {hub};
\node[font=\scriptsize, color=black!70] at ( 7.0, 1.0) {w};
\node[font=\scriptsize, color=black!70] at ( 7.6,-0.3) {$d^\star$};
\draw[edge] (s2) -- (h2);
\draw[edge] (h2) -- (w2);
\draw[edge] (w2) -- (d2);

\draw[trans] (-1.3, 0.7) -- (0.1, 0.7)
  node[midway, above, font=\footnotesize] {$\Psi$};
\draw[trans] ( 3.5, 0.7) -- (4.9, 0.7)
  node[midway, above, font=\footnotesize] {$\Psi$};

\coordinate (target1) at (0.3, 2.0);
\coordinate (target2) at (5.1, 2.0);

\draw[inputedge] (y0) to[bend right=8] (target1);
\draw[inputedge] (l0) to[bend right=8] (target1);
\draw[inputedge] (y1) to[bend right=8] (target2);
\draw[inputedge] (l1) to[bend right=8] (target2);

\begin{scope}[shift={(-3.5, -1.8)}]
  \node[rv] at (0, 0) {};
  \node[font=\footnotesize, anchor=west, color=black!70] at (0.35, 0)
    {random external};

  \node[det] at (3.0, 0) {};
  \node[font=\footnotesize, anchor=west, color=black!70] at (3.35, 0)
    {deterministic};

  \draw[draw=blue!50!black, dashed, rounded corners=2pt, line width=0.4pt]
    (5.45, -0.2) rectangle (5.95, 0.2);
  \node[font=\footnotesize, anchor=west, color=black!70] at (6.05, 0)
    {Markov state $\xi$ (network of nodes)};
\end{scope}

\end{tikzpicture}
}\\[4pt]
{\small (b) Markov chain dynamics across three time slices}
\end{minipage}
\caption{The \textsc{ISOMORPH} model. \textbf{(a)} Released
network: 13 U.S.\ cities (3 sources, 9 intermediate
warehouses across 5 tiers, and destination NewYork)
connected by 16 directed edges, each carrying a user-set
transit time~$\tau_e$; the map shows geographic placement,
but the graph structure and routing are defined by the
edges and their transit times, not by geographic distance.
\textbf{(b)} Dynamic Bayesian network across three time
slices. Top row: random inputs---demand~$y_t$ (violet,
drawn from deterministic intensity~$\lambda_t$, orange)
and source lead times~$L_t$ (violet). Bottom row (dashed
boxes): Markov state~$\xi_t$ transitioned by~$\Psi$
(Algorithm~\ref{alg:day}), see also Fig.~\ref{fig:psi-flow} Each slice shows a representative subset of the
$|\mathcal{N}| = 13$ nodes (cities).}
\label{fig:pgm}
\end{figure}

\begin{figure}[t]
\centering
\resizebox{\textwidth}{!}{
\begin{tikzpicture}[
  >=Latex,
  every node/.style={font=\footnotesize},
  procstep/.style={
    draw, rounded corners=2pt, align=center,
    text width=20mm, minimum height=16mm,
    inner sep=3pt, fill=gray!8
  },
  rndstep/.style={
    procstep, fill=violet!12, draw=violet!55!black, line width=0.6pt
  },
  rinput/.style={
    draw=violet!55!black, rounded corners=2pt, align=center,
    text width=26mm, minimum height=7mm,
    fill=violet!5, font=\scriptsize, inner sep=2pt
  },
  state/.style={
    draw=black!55, dashed, rounded corners=2pt,
    minimum width=11mm, minimum height=14mm, align=center,
    font=\footnotesize\itshape
  },
  flow/.style={->, thick},
  rflow/.style={->, thick, violet!55!black}
]

\node[state] (xin) {$\xi_t$};

\node[procstep, right=3mm of xin]   (s1) {\textbf{1.} Receive\\scheduled arrivals\\at non-destination\\nodes};
\node[procstep, right=1.5mm of s1]  (s2) {\textbf{2.} Receive\\arrivals at $d^\star$;\\clear backlog};
\node[procstep, right=1.5mm of s2]  (s3) {\textbf{3.} Reset edge\\containers};
\node[rndstep,  right=1.5mm of s3]  (s4) {\textbf{4.} Draw $y_t$;\\update smoothed\\demand;\\serve at $d^\star$};
\node[procstep, right=1.5mm of s4]  (s5) {\textbf{5.} Dispatch\\$\to d^\star$ via\\Dijkstra + pack};
\node[procstep, right=1.5mm of s5]  (s6) {\textbf{6.} Inter-WH\\pull on\\real edges};
\node[rndstep,  right=1.5mm of s6]  (s7) {\textbf{7.} Source orders;\\draw lead times $L_t$};

\node[state, right=3mm of s7] (xout) {$\xi_{t+1}$};

\draw[flow] (xin) -- (s1);
\draw[flow] (s1) -- (s2);
\draw[flow] (s2) -- (s3);
\draw[flow] (s3) -- (s4);
\draw[flow] (s4) -- (s5);
\draw[flow] (s5) -- (s6);
\draw[flow] (s6) -- (s7);
\draw[flow] (s7) -- (xout);

\node[rinput, above=10mm of s4] (yt) {demand $y_t$};
\node[rinput, above=10mm of s7] (lt) {source lead times $L_t$};
\draw[rflow] (yt) -- (s4);
\draw[rflow] (lt) -- (s7);

\end{tikzpicture}
}
\caption{One-step transition map $\Psi: \xi_t \mapsto \xi_{t+1}$ as a fixed sequence of seven sub-steps (Algorithm~\ref{alg:day}). Random external inputs (\textcolor{violet!55!black}{violet}, top) enter at exactly two sub-steps---demand $y_t$ at sub-step (4) and source lead times $L_t$ at sub-step (7); the remaining five sub-steps are deterministic bookkeeping given $(\xi_t, y_t, L_t)$. Sub-steps (1) and (2) integrate inbound arrivals due at $t$, (3) replenishes per-step edge capacity, (4) consumes the demand draw and updates the destination's stock and backlog, (5) and (6) move units along edges via Dijkstra routing and greedy first-fit packing, and (7) places fresh source orders with stochastic lead times.}
\label{fig:psi-flow}
\end{figure}

The base configuration holds three properties fixed within a run---stationary
topology, fixed control policy, and deterministic edge transit
(\S\ref{app:replenish})---and under them three pathwise conservation
identities, in two families (mass, and demand events), hold on every sample
path; they are stated and proved in \S\ref{app:properties}
(Prop.~\ref{prop:conservation}).

Those three assumptions are sufficient for the proof but stronger than the
identities require, which matters because App.~\ref{app:extensions} relaxes two of
them in the released simulator. Inspecting
\eqref{eq:per-node-mass}--\eqref{eq:backlog-thm}: the two mass identities refer
to receipts, dispatches, and in-transit counts alone---never to time-invariance
of $\mathcal{G}$ or to the distribution of transit times---and the backlog
identity refers to the destination service rule. Accordingly, dynamic edge
cut and restore (which relaxes stationary topology and generates the disruption
dynamics of \S\ref{sec:datasets-validation:disruption}) and stochastic edge transit both
leave all three identities intact. What is actually load-bearing is narrower:
that no sub-step discards a unit---which is why the finite-warehouse-capacity
extension holds excess arrivals on the inbound edge instead of dropping
them---and that demand events are not coupled across items, which is why
substitution is the one entry in Table~\ref{tab:extension-future} marked as
breaking \eqref{eq:backlog-thm}. Fixed control policy is the one assumption
still in force in every released run. We have not re-proved
Prop.~\ref{prop:conservation} under the relaxed assumptions; the closure column
of App.~\ref{app:extensions} is an argument about the state vector, and the
archived extension rollouts satisfy the identities numerically.

The simulator is realized as three algorithms in
\S\ref{app:params}.
The top-level procedure is Algorithm~\ref{alg:day}
(\S\ref{app:day}), a seven-stage per-step loop
that defines~$\Psi$, see Figure~\ref{fig:psi-flow}; it calls two subroutines:
\begin{itemize}[leftmargin=1.4em,itemsep=2pt,topsep=2pt]
  \item Algorithm~\ref{alg:pack}
    (\S\ref{app:pack}): greedy first-fit packing of
    units into finite-capacity containers along a shipping
    path.
  \item Algorithm~\ref{alg:demand}
    (\S\ref{app:demand-alg}): sampling of the
    five-component demand intensity tensor.
\end{itemize}

\paragraph{Design rationale for the modeled dynamics.}

The components of \sysname{} are individually well-known: $(s,S)$
replenishment, which is optimal for the single-location dynamic inventory
problem under fixed ordering costs~\citep{scarf1960}; exponential smoothing
for the demand estimate that drives proactive shipping; shortest-path routing
via Dijkstra's algorithm~\citep{dijkstra1959}; and greedy first-fit packing
into containers. The design of \sysname{} lies in their sequential coupling
across warehouses and items within a single Markovian formulation on a
directed graph (Algorithm~\ref{alg:day}). This formulation follows a common
principle: realistic network behavior is introduced without substantially
enlarging the dynamic state, allowing \sysname{} to retain a compact
Markov representation while supporting multi-echelon and multi-item
dynamics. The specific structural choices underlying this principle,
including capacity and routing, pathwise conservation, lead-time structure,
and proactive shipping, are detailed in Appendix~\ref{app:related-detailed}.
Additional mechanisms can be treated as scenario-level extensions;
App.~\ref{app:extensions} discusses their implementation and configuration.

\subsection{Mathematical model}
\label{app:math}

This section specifies the simulator as a mathematical object: the Markov chain structure (\S\ref{app:formal}), the demand intensity (\S\ref{app:demand}), the replenishment and service rules (\S\ref{app:replenish}), and the inter-warehouse flows (\S\ref{app:flows}). 

\subsubsection{Markov chain structure}
\label{app:formal}

The simulator is a Markov chain on a hybrid state space: a deterministic one-step map $\Psi$, driven by two random external inputs (exogenous demand and source lead times), induces a transition kernel as the average of $\Psi$ over their joint law, yielding an equivalent Dynamic Bayesian Network (DBN) representation (Figure~\ref{fig:pgm}).

\textbf{Graph and node attributes.} Let $\mathcal{G} = (\mathcal{N}, \mathcal{E})$ be a directed graph with $|\mathcal{N}| = 13$. The nodes split into three groups: source nodes $\mathcal{N}_{\mathrm{src}}$, intermediate nodes $\mathcal{N}_{\mathrm{int}}$, and a single destination $d^\star$, see Figure \ref{fig:pgm}. Each edge $e = (u,v) \in \mathcal{E}$ carries a transit time $\tau_e \in \mathbb{Z}_{\geq 1}$ (time units), a per-container volume $V_e \in \mathbb{R}_{>0}$, and a per-step container count $K_e \in \mathbb{Z}_{\geq 0}$. The per-step volume capacity on $e$ is $C_e := K_e V_e$. 

In the released instantiation (\S\ref{sec:datasets}), $\mathcal{N}_{\mathrm{src}} = \{\text{San Francisco}, \text{St.\,Louis}, \text{Orlando}\}$, $d^\star = \text{NewYork}$, and $\mathcal{N}_{\mathrm{int}}$ contains the nine intermediate warehouses; the per-edge values $(\tau_e, V_e, K_e)$ are listed in Table~\ref{tab:edge-params} (\S\ref{app:params:edges}).

\textbf{Items and time.} The item set is $\mathcal{I} = \{1, \ldots, C\}$ with per-item unit volumes $v_i \in \mathbb{R}_{>0}$. Time is discrete with one step per user-chosen time unit, $t \in \{0, 1, \ldots, T-1\}$ and horizon $T$. The released runs use a \emph{daily} step and $T = 7{,}300$ (20 years).

\textbf{State variables.} 
The state at the end of time unit $t$ is the vector
\begin{equation}
    \xi_t = \big(\mathrm{OH}_t,\ B_t,\ \mathrm{Out}_t,\ \mathrm{IT}_t,\ \tilde{\lambda}_t\big) \in \mathcal{X},
    \label{eq:state-vector}
\end{equation}
with the following five components:
\begin{itemize}
    \item On-hand inventory $\mathrm{OH}_t = (\mathrm{OH}^{n,i}_t)_{n \in \mathcal{N},\, i \in \mathcal{I}} \in \mathbb{Z}_{\geq 0}^{|\mathcal{N}| \cdot C}$, where $\mathrm{OH}^{n,i}_t$ is the number of units of item $i$ held at node $n$.
    \item Backlog $B_t = (B^{n,i}_t)_{n \in \mathcal{N},\, i \in \mathcal{I}} \in \mathbb{Z}_{\geq 0}^{|\mathcal{N}| \cdot C}$, where $B^{n,i}_t$ is the number of demand units owed by node $n$ for item $i$, which vanishes at every node other than $d^\star$.
    \item Outstanding source orders $\mathrm{Out}_t = (\mathrm{Out}^{n,i}_t)_{n \in \mathcal{N},\, i \in \mathcal{I}}$, where each $\mathrm{Out}^{n,i}_t \in \{\emptyset\} \cup (\mathbb{Z}_{>0} \times \mathbb{Z}_{>0})$ is either empty or a pair $(\tau, q)$ recording $q$ units due to arrive at time unit $\tau$.
    \item Scheduled destination arrivals $\mathrm{IT}_t \in \mathcal{M}\big(\mathbb{Z}_{>0} \times \mathcal{I} \times \mathbb{Z}_{>0}\big)$, a finite collection of triples $(\tau, i, q)$ recording shipments en route to $d^\star$.
    \item Smoothed demand estimate $\tilde{\lambda}_t = (\tilde{\lambda}^i_t)_{i \in \mathcal{I}} \in \mathbb{R}_{>0}^C$, an exponentially weighted moving average of recent observed demand, updated by $\tilde{\lambda}^i_{t+1} = \alpha y_{i,t} + (1-\alpha)\tilde{\lambda}^i_t$ with $\alpha = 0.05$.
\end{itemize}

\textbf{Hybrid state space.} The state $\xi_t$ combines four discrete-valued components and one continuous-valued component. The discrete part lives in
\begin{equation*}
    \mathcal{Q} = \mathbb{Z}_{\geq 0}^{|\mathcal{N}| \cdot C} \times \mathbb{Z}_{\geq 0}^{|\mathcal{N}| \cdot C} \times \big(\{\emptyset\} \cup \mathbb{Z}_{>0} \times \mathbb{Z}_{>0}\big)^{|\mathcal{N}| \cdot C} \times \mathrm{Lists}\big(\mathbb{Z}_{>0} \times \mathcal{I} \times \mathbb{Z}_{>0}\big),
\end{equation*}
a countable set comprising integer counts for on-hand inventory and backlog ($|\mathcal{N}| \cdot C$ entries each), optional integer pairs for outstanding source orders ($|\mathcal{N}| \cdot C$ entries), and a finite list of triples $(\tau, i, q)$ recording shipments scheduled to arrive at $d^\star$ at time unit $\tau$. The continuous part is the smoothed demand vector $\tilde{\lambda}_t \in \mathbb{R}_{>0}^C$, which takes values in a $C$-dimensional positive-real space. The full state space is the hybrid product
\begin{equation}
    \mathcal{X} = \mathcal{Q} \times \mathbb{R}_{>0}^C,
    \label{eq:state-space}
\end{equation}
on which probability distributions are defined in the standard way for hybrid spaces (the Borel product of the discrete $\sigma$-algebra on $\mathcal{Q}$ with the Borel $\sigma$-algebra on $\mathbb{R}_{>0}^C$).

\textbf{Per-edge in-transit volumes (derived quantity).} Per-edge in-transit counts are not part of the state variables above. We write $q^{u \to v}_{i,t}$ for the cumulative count of item-$i$ units on edge $(u,v)$ at the end of time unit $t$, computed as the sum of all dispatches that have been routed onto $(u,v)$ but whose realized arrival time is strictly later than $t$. This is a derived quantity, reconstructible from the dispatch records produced inside the per-step update (Algorithm~\ref{alg:pack}, \S\ref{app:pack}); it is referenced in Proposition~\ref{prop:conservation}.

\begin{figure}[t]
\centering
\begin{tikzpicture}[
  >={Stealth[length=2mm]},
  font=\small,
  rv/.style={circle, draw=violet!60!black, fill=violet!20, line width=0.4pt,
             minimum size=0.65cm, inner sep=0pt},
  det/.style={circle, draw=orange!70!black, fill=orange!30, line width=0.4pt,
              minimum size=0.65cm, inner sep=0pt},
  src/.style={circle, draw=black!50, fill=gray!30, line width=0.4pt,
              minimum size=0.5cm, inner sep=0pt},
  hub/.style={circle, draw=blue!60!black, fill=blue!20, line width=0.4pt,
              minimum size=0.5cm, inner sep=0pt},
  ware/.style={circle, draw=blue!60!black, fill=blue!20, line width=0.4pt,
               minimum size=0.5cm, inner sep=0pt},
  dest/.style={circle, draw=blue!80!black, fill=blue!50, line width=0.4pt,
               minimum size=0.55cm, inner sep=0pt},
  edge/.style={->, line width=0.4pt, color=black!60},
  inputedge/.style={->, line width=0.4pt, color=violet!70!black},
  trans/.style={->, line width=0.7pt, color=black!70},
]

\draw[draw=blue!50!black, dashed, rounded corners=4pt, line width=0.4pt]
  (-4.7, -0.8) rectangle (-1.3, 2.2);
\draw[draw=blue!50!black, dashed, rounded corners=4pt, line width=0.4pt]
  ( 0.1, -0.8) rectangle ( 3.5, 2.2);
\draw[draw=blue!50!black, dashed, rounded corners=4pt, line width=0.4pt]
  ( 4.9, -0.8) rectangle ( 8.3, 2.2);

\node[font=\footnotesize, color=blue!60!black, anchor=north west]
  at (-4.6, 2.1) {\textit{$\xi(t-1)$}};
\node[font=\footnotesize, color=blue!60!black, anchor=north west]
  at ( 0.2, 2.1) {\textit{$\xi(t)$}};
\node[font=\footnotesize, color=blue!60!black, anchor=north west]
  at ( 5.0, 2.1) {\textit{$\xi(t+1)$}};

\node[font=\small, color=black!60] at (-3.0, 4.6) {\textit{Step $t-1$}};
\node[font=\small, color=black!60] at ( 1.8, 4.6) {\textit{Step $t$}};
\node[font=\small, color=black!60] at ( 6.6, 4.6) {\textit{Step $t+1$}};

\node[det] (lam0) at (-3.9, 3.6) {$\lambda$};
\node[rv]  (y0)   at (-3.0, 3.6) {$y$};
\node[rv]  (l0)   at (-2.1, 3.6) {$L$};
\draw[edge] (lam0) -- (y0);

\node[det] (lam1) at ( 0.9, 3.6) {$\lambda$};
\node[rv]  (y1)   at ( 1.8, 3.6) {$y$};
\node[rv]  (l1)   at ( 2.7, 3.6) {$L$};
\draw[edge] (lam1) -- (y1);

\node[det] (lam2) at ( 5.7, 3.6) {$\lambda$};
\node[rv]  (y2)   at ( 6.6, 3.6) {$y$};
\node[rv]  (l2)   at ( 7.5, 3.6) {$L$};
\draw[edge] (lam2) -- (y2);

\node[src]  (s0) at (-4.2, 1.5) {};
\node[hub]  (h0) at (-3.4, 1.5) {};
\node[ware] (w0) at (-2.6, 1.5) {};
\node[dest] (d0) at (-2.0, 0.2) {};
\node[font=\scriptsize, color=black!70] at (-4.2, 1.0) {src};
\node[font=\scriptsize, color=black!70] at (-3.4, 1.0) {hub};
\node[font=\scriptsize, color=black!70] at (-2.6, 1.0) {w};
\node[font=\scriptsize, color=black!70] at (-2.0,-0.3) {$d^\star$};
\draw[edge] (s0) -- (h0);
\draw[edge] (h0) -- (w0);
\draw[edge] (w0) -- (d0);

\node[src]  (s1) at ( 0.6, 1.5) {};
\node[hub]  (h1) at ( 1.4, 1.5) {};
\node[ware] (w1) at ( 2.2, 1.5) {};
\node[dest] (d1) at ( 2.8, 0.2) {};
\node[font=\scriptsize, color=black!70] at ( 0.6, 1.0) {src};
\node[font=\scriptsize, color=black!70] at ( 1.4, 1.0) {hub};
\node[font=\scriptsize, color=black!70] at ( 2.2, 1.0) {w};
\node[font=\scriptsize, color=black!70] at ( 2.8,-0.3) {$d^\star$};
\draw[edge] (s1) -- (h1);
\draw[edge] (h1) -- (w1);
\draw[edge] (w1) -- (d1);

\node[src]  (s2) at ( 5.4, 1.5) {};
\node[hub]  (h2) at ( 6.2, 1.5) {};
\node[ware] (w2) at ( 7.0, 1.5) {};
\node[dest] (d2) at ( 7.6, 0.2) {};
\node[font=\scriptsize, color=black!70] at ( 5.4, 1.0) {src};
\node[font=\scriptsize, color=black!70] at ( 6.2, 1.0) {hub};
\node[font=\scriptsize, color=black!70] at ( 7.0, 1.0) {w};
\node[font=\scriptsize, color=black!70] at ( 7.6,-0.3) {$d^\star$};
\draw[edge] (s2) -- (h2);
\draw[edge] (h2) -- (w2);
\draw[edge] (w2) -- (d2);

\draw[trans] (-1.3, 0.7) -- (0.1, 0.7)
  node[midway, above, font=\footnotesize] {$\Psi$};
\draw[trans] ( 3.5, 0.7) -- (4.9, 0.7)
  node[midway, above, font=\footnotesize] {$\Psi$};

\coordinate (target1) at (0.3, 2.0);
\coordinate (target2) at (5.1, 2.0);

\draw[inputedge] (y0) to[bend right=8] (target1);
\draw[inputedge] (l0) to[bend right=8] (target1);
\draw[inputedge] (y1) to[bend right=8] (target2);
\draw[inputedge] (l1) to[bend right=8] (target2);

\begin{scope}[shift={(-3.5, -1.8)}]
  \node[rv] at (0, 0) {};
  \node[font=\footnotesize, anchor=west, color=black!70] at (0.35, 0)
    {random external};

  \node[det] at (3.0, 0) {};
  \node[font=\footnotesize, anchor=west, color=black!70] at (3.35, 0)
    {deterministic};

  \draw[draw=blue!50!black, dashed, rounded corners=2pt, line width=0.4pt]
    (5.45, -0.2) rectangle (5.95, 0.2);
  \node[font=\footnotesize, anchor=west, color=black!70] at (6.05, 0)
    {Markov state $\xi$ (network of nodes)};
\end{scope}

\end{tikzpicture}

\caption{The \textsc{ISOMORPH} model as a dynamic Bayesian network across three time slices. \emph{Random external inputs} (top row, violet): demand $y_t \sim \mathrm{Poisson}(\lambda_{i,t})$ and source lead times $L_t$ are the sources of randomness in the chain, and no other component is random. The intensity $\lambda$ (orange) is deterministic given the simulator seed. \emph{Markov state} (bottom row, dashed boxes): the state $\xi_t$ is the joint state of the simulator, transitioned by the deterministic map $\Psi$ defined in Algorithm~\ref{alg:day}. The illustrated nodes inside each slice are organized by the partition $\mathcal{N} = \mathcal{N}_{\mathrm{src}} \sqcup \mathcal{N}_{\mathrm{int}} \sqcup \{d^\star\}$ from \S\ref{app:formal}: a representative source, two intermediate nodes (a hub and a downstream warehouse), and the destination. We do not draw all 13 cities of the released topology in each slice; the figure shows the role of each group, not the full graph. The geographic instantiation ($\mathcal{N}_{\mathrm{src}} = \{\text{San Francisco, St.\,Louis, Orlando}\}$, $d^\star = \text{NewYork}$, and the nine intermediate nodes) is in Figure~\ref{fig:pgm} of the main text.}
\label{fig:pgm:appendix}
\end{figure}

\paragraph{Random external inputs.}
\label{app:exogenous}

At each step $t$, the per-step update reads two random vectors $(y_t, L_t)$ that are external inputs, supplied to the system rather than computed from its current state. They are not part of the state.

\textbf{Demand.} The demand vector $y_t = (y_{1,t}, \ldots, y_{C,t}) \in \mathbb{Z}_{\geq 0}^C$ has independent Poisson components,
\begin{equation}
    y_{i,t} \sim \mathrm{Poisson}(\lambda_{i,t}),
    \label{eq:poisson-demand}
\end{equation}
with the draws independent across both $i$ and $t$. The intensity sequence $(\lambda_{i,t})_{i \in \mathcal{I},\, t \in \{0, \ldots, T-1\}}$ is constructed in \S\ref{app:demand} (see \eqref{eq:intensity}) and \S\ref{app:demand-alg} (algorithmic construction). For the present subsection, $(\lambda_{i,t})$ is treated as a fixed input to the chain.

\textbf{Source lead times.} Define the set of source-item pairs that need a fresh order on a state $\xi$ as
\begin{equation}
    \mathcal{T}(\xi) = \big\{(n,i) \in \mathcal{N}_{\mathrm{src}} \times \mathcal{I} \;:\; \mathrm{OH}^{n,i} < s^{n,i},\ \mathrm{Out}^{n,i} = \emptyset\big\},
    \label{eq:trigger-set}
\end{equation}
where the parameters $(s^{n,i}, \mu^{n,i})$ are specified in \S\ref{app:replenish}. For each $(n,i) \in \mathcal{N}_{\mathrm{src}} \times \mathcal{I}$, define
\begin{equation}
    L^{n,i}_t = \max\!\big(1,\ \lceil X^{n,i}_t \rceil\big), \qquad X^{n,i}_t \sim \mathcal{N}\!\big(\mu^{n,i},\ (0.2\,\mu^{n,i})^2\big).
    \label{eq:lead-time}
\end{equation}
Writing $\Phi$ for the standard-Gaussian cumulative distribution function, the integer $L^{n,i}_t$ takes value $\ell$ with probability
\begin{equation}
    \mathbb{P}\big(L^{n,i}_t = \ell\big) = 
    \begin{cases}
        \displaystyle \Phi\!\left(\frac{1 - \mu^{n,i}}{0.2\,\mu^{n,i}}\right), & \ell = 1, \\[10pt]
        \displaystyle \Phi\!\left(\frac{\ell - \mu^{n,i}}{0.2\,\mu^{n,i}}\right) - \Phi\!\left(\frac{\ell - 1 - \mu^{n,i}}{0.2\,\mu^{n,i}}\right), & \ell \geq 2.
    \end{cases}
    \label{eq:lead-time-pmf}
\end{equation}
Collect all draws into the lead-time vector
\begin{equation}
    L_t = \big(L^{n,i}_t\big)_{(n,i) \in \mathcal{N}_{\mathrm{src}} \times \mathcal{I}} \in \mathbb{Z}_{\geq 1}^{|\mathcal{N}_{\mathrm{src}}| \cdot C}.
    \label{eq:lead-time-vector}
\end{equation}
The transition map $\Psi$ in \S\ref{app:transition-map} reads the components indexed by $\mathcal{T}(\xi_t)$ and no others; the remaining components are extra randomness that has no effect on $\xi_{t+1}$.

\textbf{Independence and initialization.} Conditional on $\xi_t$, the lead-time vector $L_t$ has independent components and is independent of $y_t$. The collection $\big((y_t, L_t)\big)_{t \geq 0}$ is jointly independent across $t$. The initial state $\xi_0 \in \mathcal{X}$ is fixed in advance, specified by the simulator configuration. Throughout this paper, the simulation seed is treated as a fixed parameter; randomness in the chain is carried entirely by $\big((y_t, L_t)\big)_{t \geq 0}$.


\paragraph{One-step transition map.}
\label{app:transition-map}

The transition from $\xi_t$ to $\xi_{t+1}$ has two layers. The first is a deterministic map $\Psi$ that, given the current state and a realized draw of the random external inputs, computes the next state by mechanical bookkeeping. The second is the transition probability $P_t$ obtained by averaging $\Psi$ over the distribution of those inputs. This subsection specifies $\Psi$; \S\ref{app:kernel} specifies $P_t$.

Define $\Psi$ as the composition of the seven sub-step operations of Algorithm~\ref{alg:day} (\S\ref{app:day}), executed in fixed order, viewed as a map
\begin{equation}
    \Psi : \mathcal{X} \times \mathbb{Z}_{\geq 0}^C \times \mathbb{Z}_{\geq 1}^{|\mathcal{N}_{\mathrm{src}}| \cdot C} \longrightarrow \mathcal{X}.
    \label{eq:psi-domain}
\end{equation}
Given $\xi_t$ and realized inputs $(y_t, L_t)$, the next state is
\begin{equation}
    \xi_{t+1} = \Psi(\xi_t,\ y_t,\ L_t)
    \label{eq:one-step-update}
\end{equation}
deterministically: there is no further randomness in the transition once $(y_t, L_t)$ are fixed. The seven sub-steps execute in the order specified by Algorithm~\ref{alg:day}, and the resulting state $\xi_{t+1}$ depends on this ordering. Each sub-step is built from arithmetic operations, indicator functions, and finite enumeration over the finite sets $\mathcal{N}$, $\mathcal{I}$, $\mathcal{E}$. Writing $\Psi$ in closed form would require inlining all seven sub-steps with their (s,S) trigger logic and packing-residual dependencies into one expression; the algorithmic specification of Algorithm~\ref{alg:day} is more compact and is the definition of $\Psi$ throughout this paper.

\paragraph{Transition probability.}
\label{app:kernel}

The transition probability defined as $$P_t(\xi, A) = \mathbb{P}(\xi_{t+1} \in A \mid \xi_t = \xi)$$ is constructed from $\Psi$ in three steps: fix the realized inputs and apply $\Psi$ deterministically, ask whether the result lands in $A$, then average over the distribution of inputs.

\textbf{Step 1 --- pointwise transition.} For $\xi_t = \xi$ and realizations $(y, \ell)$, the next state is $\xi_{t+1} = \Psi(\xi, y, \ell)$ (\S\ref{app:transition-map}).

\textbf{Step 2 --- conditional law of inputs.} By \S\ref{app:exogenous}, the joint conditional distribution of $(y_t, L_t)$ given $\xi_t = \xi$ factorizes as $\pi^{\mathrm{dem}}_t(y) \, \pi^{\mathrm{lt}}(\ell)$, with neither factor depending on $\xi$.

\textbf{Step 3 --- average over inputs.} For each $t \geq 0$, $\xi \in \mathcal{X}$, and event $A$,
\begin{equation}
    P_t(\xi, A) = \sum_{y \in \mathbb{Z}_{\geq 0}^C} \sum_{\ell \in \mathbb{Z}_{\geq 1}^{|\mathcal{N}_{\mathrm{src}}| \cdot C}} \mathbf{1}\!\left\{\Psi(\xi, y, \ell) \in A\right\} \, \pi^{\mathrm{dem}}_t(y) \, \pi^{\mathrm{lt}}(\ell),
    \label{eq:kernel}
\end{equation}
where
\begin{equation}
    \pi^{\mathrm{dem}}_t(y) = \prod_{i=1}^C \frac{\lambda_{i,t}^{\,y_i}\, e^{-\lambda_{i,t}}}{y_i!}, \qquad \pi^{\mathrm{lt}}(\ell) = \prod_{(n,i) \in \mathcal{N}_{\mathrm{src}} \times \mathcal{I}} \mathbb{P}\big(L^{n,i} = \ell^{n,i}\big),
    \label{eq:kernel-factors}
\end{equation}
and the lead-time mass function is given in \eqref{eq:lead-time-pmf}.

Summing over the entire state space, $P_t(\xi, \mathcal{X}) = \sum_y \pi^{\mathrm{dem}}_t(y) \sum_\ell \pi^{\mathrm{lt}}(\ell) = 1$, since both factors are probability mass functions and $\Psi$ takes values in $\mathcal{X}$. For each fixed $\xi$, $P_t(\xi, \cdot)$ is a probability distribution over the next state.

The function $P_t : \mathcal{X} \times \mathcal{B}(\mathcal{X}) \to [0,1]$ is also called a \emph{Markov transition kernel} in the probability literature. We use ``transition probability'' throughout to mean the same object in less technical language.

\paragraph{Markov property}
\label{app:markov-prop}

The Markov property compares the conditional distribution of $\xi_{t+1}$ given the entire past against the conditional distribution given the current state alone. To formalize ``the entire past,'' let $\mathcal{F}_t$ be the collection of events whose outcomes are fully determined by the initial state $\xi_0$ and the random external inputs through time unit $t-1$:
\begin{equation*}
    \mathcal{F}_t = \sigma\!\big(\xi_0,\ y_0,\ \ell_0,\ \ldots,\ y_{t-1},\ \ell_{t-1}\big).
\end{equation*}
Concretely, $\mathcal{F}_t$ is the set of all yes-or-no questions whose answer can be read off from the values of these inputs alone --- ``did inventory at warehouse $w$ exceed 100 units yesterday?'', ``has any backlog accumulated at $d^\star$ in the first week?'', and so on. Closing this set under the natural operations on events (union, intersection, complement) gives the formal object $\mathcal{F}_t$ (a $\sigma$-algebra). Since $\xi_t$ is computed from these inputs by repeated application of $\Psi$, any event involving $\xi_t$ belongs to $\mathcal{F}_t$.

By construction, $(y_t, L_t)$ is independent of $\mathcal{F}_t$, so back in \eqref{eq:kernel}
\begin{equation}
    \mathbb{P}\big(\xi_{t+1} \in A \;\big|\; \mathcal{F}_t\big) = P_t(\xi_t, A) \qquad \text{a.s.}
    \label{eq:markov}
\end{equation}
The conditional distribution of $\xi_{t+1}$ given the past coincides with the distribution given $\xi_t$. The chain $(\xi_t)_{t \geq 0}$ is therefore a time-inhomogeneous Markov chain on $\mathcal{X}$ with transition probabilities $(P_t)_{t \geq 0}$, where time-inhomogeneity comes from the time-varying demand intensity $\lambda_{i,t}$ that enters $\pi^{\mathrm{dem}}_t$. The augmented process $\big((t, \xi_t)\big)_{t \geq 0}$ on $\mathbb{Z}_{\geq 0} \times \mathcal{X}$ is time-homogeneous, with transition probability
\begin{equation}
    \bar{P}\big((t, \xi),\ \{t+1\} \times A\big) = P_t(\xi, A).
    \label{eq:augmented-kernel}
\end{equation}

The state space $\mathcal{X} = \mathcal{Q} \times \mathbb{R}_{>0}^C$ is hybrid: transitions on the continuous component $\tilde{\lambda}$ are deterministic given $y_t$, while transitions on the discrete component are random.

\paragraph{DBN representation.}
The chain admits an equivalent representation as a dynamic Bayesian network (DBN). In this representation, the state $\xi_t$ is decomposed into its per-node, per-item components $$\big(\mathrm{OH}^{n,i}_t, B^{n,i}_t, \mathrm{Out}^{n,i}_t, \mathrm{IT}_t, \tilde{\lambda}^i_t\big)_{n \in \mathcal{N},\, i \in \mathcal{I}}\, ,$$ and the random external inputs $(y_t, L_t)$ enter as separate input nodes, see also Figure~\ref{fig:pgm}. The network topology $\mathcal{G}$ is then embedded in the cross-node dependencies within each time slice: $\mathrm{OH}^{n,i}_{t+1}$ depends on $\mathrm{OH}^{u,i}_t$ for upstream neighbours $u$ via the inter-warehouse flows of \S\ref{app:flows}. We do not use the DBN representation in the present work --- the Markov chain abstraction over $\xi_t$ is sufficient for the formal claims of \S\ref{app:properties} and the algorithmic specification of \S\ref{app:params}. The DBN view could be  useful for follow-up work on inference, structure learning, or interventional reasoning.

\subsubsection{Demand intensity}
\label{app:demand}

This section gives the equational form of the demand intensity sequence $(\lambda_{i,t})$ used in \S\ref{app:exogenous} as the rate parameter of the Poisson demand draws.

For each item $i$ and time unit $t$, demand intensity is
\begin{equation}
    \lambda_{i,t} = \bar{\lambda}_i \cdot \big[1 + S_i(t) + W_i(t) + A_i(t) + P_i(t) + g_i\, G(t)\big]_{\geq \varepsilon},
    \label{eq:intensity}
\end{equation}
with $[\cdot]_{\geq \varepsilon} := \max(\cdot, \varepsilon)$ and $\varepsilon = 0.08$ (a floor that prevents the intensity from collapsing to zero). The five components are as follows.

\textbf{Yearly seasonal} (with second harmonic, so the shape is not a perfect sine):
\begin{equation}
    S_i(t) = a^{(y)}_{1,i}\, \sin(\omega_y t + \phi_i) + a^{(y)}_{2,i}\, \sin(2\omega_y t + 0.7\,\phi_i), \qquad \omega_y = 2\pi/365.
    \label{eq:yearly}
\end{equation}

\textbf{Weekly seasonal:}
\begin{equation}
    W_i(t) = a^{(w)}_i\, \sin(\omega_w t + \psi_i), \qquad \omega_w = 2\pi/7.
    \label{eq:weekly}
\end{equation}

\textbf{Long-horizon drift, autoregressive of order 1 (AR(1)) and clipped:}
\begin{equation}
    A_i(t) = \mathrm{clip}\!\big(\phi^{\mathrm{AR}}_i\, A_i(t-1) + \epsilon_{i,t},\ \pm 0.6\big), \qquad \epsilon_{i,t} \sim \mathcal{N}(0, \sigma^2_{\mathrm{AR}, i}).
    \label{eq:drift}
\end{equation}
The AR(1) recursion produces a slowly varying random trajectory whose persistence is set by $\phi^{\mathrm{AR}}_i$ and whose noise level is set by $\sigma_{\mathrm{AR},i}$; the clip at $\pm 0.6$ keeps the drift contribution bounded.

\textbf{Per-item bursts:}
\begin{equation}
    P_i(t) = \sum_k h^P_{i,k}\, \pi\!\left(\frac{t - t^P_{i,k}}{\Delta^P_{i,k}}\right),
    \label{eq:bursts}
\end{equation}
where $\pi(\cdot)$ is a trapezoidal pulse: it ramps linearly up over the first 15\% of its duration, plateaus for 60\%, and ramps down over the last 25\%. Bursts fire as a sparse random process per item.

\textbf{Shared macro-shock:}
\begin{equation}
    G(t) = \sum_k h^G_k\, \pi\!\left(\frac{t - t^G_k}{\Delta^G_k}\right),
    \label{eq:macro-shock}
\end{equation}
drawn once per run and shared across all items, with per-item sensitivity $g_i$.
The distributions of all coefficients are given in Table~\ref{tab:demand-params} (\S\ref{app:params:demand}).

\subsubsection{Replenishment policy and destination service}
\label{app:replenish}

This section specifies the per-node replenishment policy and the destination service rule. The parameters $(s^{n,i}, S^{n,i}, \mu^{n,i})$ are used in \S\ref{app:exogenous} (the trigger set $\mathcal{T}$ and lead-time means) and in \S\ref{app:day} (sub-step (vii)). Standing assumptions held fixed across all releases are recorded at the end.

\textbf{Replenishment.} Replenishment follows an $(s, S)$ policy at every non-destination node: when on-hand drops below the reorder threshold $s$, place an order that brings on-hand back up to the order-up-to level $S$. The two node types realize this policy differently. Source nodes $n \in \mathcal{N}_{\mathrm{src}}$ act as factories with unlimited upstream supply: when $\mathrm{OH}^{n,i} < s^{n,i}$ and no order is outstanding, an order of size $S^{n,i} - \mathrm{OH}^{n,i}$ is placed and arrives after a random lead time $L^{n,i}$ drawn as in \eqref{eq:lead-time}. Intermediate nodes $n \in \mathcal{N}_{\mathrm{int}}$ use the same trigger but pull inventory from upstream neighbours along real edges; the requested quantity is physically routed and packed (Algorithm~\ref{alg:pack}, \S\ref{app:pack}), and the arrival time equals the realized path-transit time. A closed-form lead time therefore exists at the source boundary and nowhere else.

\textbf{Destination service.} 
Demand $y_{i,t}$ is served at $d^\star$ in two ordered events within a single transition $\xi_t \to \xi_{t+1}$. First, scheduled arrivals at $d^\star$ are received: writing
\begin{equation}
A_{i,t} \;:=\; \mathrm{IT}_t[t{+}1,\, i]
\label{eq:dest-arrivals}
\end{equation}
for the units of item $i$ scheduled to arrive at $d^\star$ between $t$ and $t{+}1$, the arrivals are applied first to outstanding backlog, and the residual is added to on-hand:
\begin{equation}
B^{d^\star,i}_{(2)} \;=\; \bigl(B^{d^\star,i}_t - A_{i,t}\bigr)_+,
\qquad
\mathrm{OH}^{d^\star,i}_{(2)} \;=\; \mathrm{OH}^{d^\star,i}_t + \bigl(A_{i,t} - B^{d^\star,i}_t\bigr)_+ .
\label{eq:dest-arrival-update}
\end{equation}
Second, the day's demand is served from the post-arrival on-hand and any unmet portion accumulates as new backlog:
\begin{equation}
\mathrm{OH}^{d^\star,i}_{t+1} \;=\; \bigl(\mathrm{OH}^{d^\star,i}_{(2)} - y_{i,t}\bigr)_+,
\qquad
B^{d^\star,i}_{t+1} \;=\; B^{d^\star,i}_{(2)} + \bigl(y_{i,t} - \mathrm{OH}^{d^\star,i}_{(2)}\bigr)_+ .
\label{eq:backlog-update}
\end{equation}
The fill-rate of item $i$ over a horizon is the ratio of served units to demanded units on that horizon.

\textbf{Standing assumptions.} Three properties are held fixed within a run:
\begin{itemize}
    \item \emph{Stationary topology}: the graph $\mathcal{G}$ and the per-edge triple $(\tau_e, V_e, K_e)$ are set at construction and not modified.
    \item \emph{Fixed control policy}: the $(s, S)$ levels, the routing rule, and the packing primitive are held constant; no online control is applied.
    \item \emph{Deterministic edge transit}: edges in $\mathcal{E}$ contribute no lead-time noise; randomness in lead times is confined to the source-level draw of \eqref{eq:lead-time}.
\end{itemize}

\subsubsection{Inter-warehouse flows}
\label{app:flows}

Cities communicate by shipping physical units across edges. This section gives the mathematical form of the three flow types that move material between cities. Each is a deterministic function of the state once the random external inputs $(y_t, L_t)$ are fixed; the algorithmic resolution is in \S\ref{app:day}.

\textbf{Source arrivals (boundary inflow).} For $n \in \mathcal{N}_{\mathrm{src}}$, source orders are placed under the (s,S) trigger of \eqref{eq:trigger-set} and delivered after the random lead time of \eqref{eq:lead-time}. An order placed at time unit $t' < t$ with realized lead time $L^{n,i}_{t'}$ delivers as a scheduled arrival at time unit $t' + L^{n,i}_{t'}$; if this equals $t$ the order arrives at $t$ with size $S^{n,i} - \mathrm{OH}^{n,i}_{t'}$ (the order-up-to amount at the time of placement). The cumulative source-arrival mass between $t$ and $t+1$, denoted $A^{\mathrm{src}}_{i,t}$, is defined formally in \eqref{eq:source-arrivals}. Source replenishment is the mechanism by which units enter the network from outside; sources act as boundary inflow.

\textbf{Outbound dispatch (warehouses to destination).} For each item $i$ and time unit $t$, the per-step target shipment volume to the destination is
\begin{equation}
    \mathrm{target}_{i,t} = \max\!\left\{0,\ B^{d^\star, i}_t + m\, \tilde{\lambda}^i_t - \mathrm{IT}^i_t - \mathrm{OH}^{d^\star, i}_t\right\},
    \label{eq:dispatch-target}
\end{equation}
where $\mathrm{IT}^i_t = \sum_{\tau > t} \mathrm{IT}_t[\tau, i]$ is total in-transit volume bound for $d^\star$ and $m$ is the pipeline multiplier (a configuration parameter; $m = 7$ in this paper). The interpretation: ship enough to cover the destination's backlog plus $m$ time units of expected demand, minus what is already on hand or in transit. Warehouses are visited in increasing Dijkstra distance to $d^\star$, where the edge weight on $(u,v)$ is $\tau_{(u,v)} / (K_{(u,v)} V_{(u,v)})$ (favoring short paths with high per-step capacity); each warehouse's contribution is bounded by its on-hand stock and by the remaining target after earlier warehouses have shipped. The flow quantity from a warehouse is the volume placed by $\textsc{GreedyPack}$ along the resolved path; arrivals are scheduled at $d^\star$ at time unit $t + \lceil \sum_e \tau_e \rceil$.

\textbf{Inter-warehouse pull.} Intermediate nodes $n \in \mathcal{N}_{\mathrm{int}}$ trigger replenishment when $\mathrm{OH}^{n,i}_t < s^{n,i}$ and $\mathrm{Out}^{n,i}_t = \emptyset$. The requested quantity is $S^{n,i} - \mathrm{OH}^{n,i}_t$. The flow is sourced from the first upstream supplier $u \in \mathcal{N}_{\mathrm{int}} \cup \mathcal{N}_{\mathrm{src}}$ with available stock $\mathrm{OH}^{u,i}_t > 0$, walked in increasing total path-transit-time order. The flow quantity is the volume placed by $\textsc{GreedyPack}$ along the path from $u$ to $n$, and arrival is scheduled at $n$ at time unit $t + \lceil \sum_e \tau_e \rceil$.

\textbf{Sequential coupling within a step.} In all three flow types, the actual quantity placed depends on the residual capacity of edge containers along the path, which is itself a function of earlier flows in the same step. This makes the inter-warehouse flows sequentially coupled across items and across warehouses within the per-step update; the ordering of sub-steps in Algorithm~\ref{alg:day} fixes this sequence. In Algorithm~\ref{alg:day}, source arrivals are realized in sub-step (vii) (order placement) and sub-step (i) (receipt), outbound dispatch in sub-step (v), and inter-warehouse pull in sub-step (vi).


\subsection{Algorithms}
\label{app:params}
 
This section gives the explicit algorithmic construction of the transition map $\Psi$ defined abstractly in \S\ref{app:transition-map}. \S\ref{app:day} gives the seven-sub-step per-step procedure that defines $\Psi$. \S\ref{app:pack} defines the packing routine called from \S\ref{app:day}. \S\ref{app:demand-alg} gives the intensity-construction procedure that produces the demand rates of \S\ref{app:demand}. The parameter values used to produce the released runs are in Appendix~\ref{app:param-values}.
 
\subsubsection{Per-step algorithm}
\label{app:day}
 
Algorithm~\ref{alg:day} realizes the transition map $\Psi$ of \S\ref{app:transition-map} as a sequence of seven sub-steps, executed in the fixed order given below. Each sub-step reads state written by the previous one, and the conservation identities of Proposition~\ref{prop:conservation} depend on this ordering.
 
\par\noindent\hrule height 0.6pt\vspace{3pt}
\noindent\textbf{Algorithm~\refstepcounter{algorithm}\thealgorithm\label{alg:day}}\quad
One step of the \textsc{ISOMORPH} simulation loop.
\vspace{2pt}\hrule height 0.4pt\vspace{4pt}
\begin{algorithmic}[1]
\State \textbf{Input:} time-unit index $t$; graph $\mathcal{G}$; per-node $(s, S, \mu_L)$; current state $\{\mathrm{OH}, B, \mathrm{Out}, \mathrm{IT}\}$; demand function $\mathrm{demand}(t)$; smoothed demand $\tilde{\lambda}$ with $\alpha = 0.05$; pipeline multiplier $m$.
\Statex
\Statex \textbf{(1) Receive scheduled $(s, S)$ replenishment.}
\For{$n \in \mathcal{N} \setminus \{d^\star\}$ and $i \in \mathcal{I}$}
    \If{$\mathrm{Out}^{n,i}$ is $(\tau, q)$ with $\tau \leq t$}
        \State $\mathrm{OH}^{n,i} \gets \mathrm{OH}^{n,i} + q$; clear $\mathrm{Out}^{n,i}$.
    \EndIf
\EndFor
\Statex
\Statex \textbf{(2) Receive arrivals at destination; backlog cleared first.}
\For{$i$ such that $\mathrm{IT}_t[t, i] > 0$}
    \State $q \gets \mathrm{IT}_t[t, i]$; $b \gets B^{d^\star, i}$.
    \State $B^{d^\star, i} \gets (b - q)_+$;\quad $\mathrm{OH}^{d^\star, i} \gets \mathrm{OH}^{d^\star, i} + (q - b)_+$.
\EndFor
\Statex
\Statex \textbf{(3) Reset edge containers each step.}
\For{$e \in \mathcal{E}$}
    \State $\mathrm{containers}_e \gets [V_e] \times K_e$.
\EndFor
\Statex
\Statex \textbf{(4) Demand, service, and smoothed-demand update at $d^\star$.}
\State $y \gets \mathrm{demand}(t)$.
\For{$i \in \mathcal{I}$}
    \State $\tilde{\lambda}_i \gets \alpha\, y_i + (1-\alpha)\, \tilde{\lambda}_i$.
    \State $\mathrm{served} \gets \min(\mathrm{OH}^{d^\star, i},\ y_i)$;\quad $\mathrm{unfilled} \gets y_i - \mathrm{served}$.
    \State $\mathrm{OH}^{d^\star, i} \gets \mathrm{OH}^{d^\star, i} - \mathrm{served}$;\quad $B^{d^\star, i} \gets B^{d^\star, i} + \mathrm{unfilled}$.
\EndFor
\Statex
\Statex \textbf{(5) Dispatch warehouses $\to d^\star$ (round-robin over items).}
\For{$i$ in round-robin order}
    \State $\mathrm{target}_i \gets \max\{0,\ B^{d^\star, i} + m\, \tilde{\lambda}_i - \mathrm{in\text{-}transit}_i - \mathrm{OH}^{d^\star, i}\}$.
    \State $\mathrm{remaining} \gets \mathrm{target}_i$.
    \For{$w$ in warehouses sorted by Dijkstra distance to $d^\star$ (ascending)}
        \If{$\mathrm{remaining} \leq 0$} \textbf{break} \EndIf
        \State $q_{\mathrm{try}} \gets \min(\mathrm{OH}^{w,i},\ \mathrm{remaining})$.
        \State $q_{\mathrm{placed}} \gets \textsc{GreedyPack}(i,\ q_{\mathrm{try}},\ \mathrm{Path}(w, d^\star))$.\quad $\triangleright$ Alg.~\ref{alg:pack}
        \State $\mathrm{OH}^{w,i} \mathrel{-}= q_{\mathrm{placed}}$;\quad $\mathrm{remaining} \mathrel{-}= q_{\mathrm{placed}}$.
        \State Schedule arrival at $d^\star$ at time unit $t + \lceil \sum_{e \in \mathrm{Path}(w, d^\star)} \tau_e \rceil$ of size $q_{\mathrm{placed}}$.
    \EndFor
\EndFor
\Statex
\Statex \textbf{(6) Inter-warehouse replenishment (pull along real edges).}
\For{$n \in \mathcal{N}_{\mathrm{int}}$ in upstream-first order and each $i \in \mathcal{I}$}
    \If{$\mathrm{OH}^{n,i} < s^{n,i}$ and $\mathrm{Out}^{n,i}$ is empty}
        \State $\mathrm{qty} \gets S^{n,i} - \mathrm{OH}^{n,i}$.
        \For{$u$ upstream supplier of $n$, sorted by total transit time}
            \State $q_{\mathrm{placed}} \gets \textsc{GreedyPack}(i,\ \min(\mathrm{OH}^{u,i},\ \mathrm{qty}),\ \mathrm{Path}(u, n))$.
            \If{$q_{\mathrm{placed}} > 0$}
                \State $\mathrm{OH}^{u,i} \mathrel{-}= q_{\mathrm{placed}}$;\quad $\mathrm{Out}^{n,i} \gets (t + \lceil \sum_e \tau_e \rceil,\ q_{\mathrm{placed}})$;\quad \textbf{break}.
            \EndIf
        \EndFor
    \EndIf
\EndFor
\Statex
\Statex \textbf{(7) Place $(s, S)$ orders at source nodes (random lead time).}
\For{$n \in \mathcal{N}_{\mathrm{src}}$ and $i \in \mathcal{I}$}
    \If{$\mathrm{OH}^{n,i} < s^{n,i}$ and $\mathrm{Out}^{n,i}$ is empty}
        \State Draw $X \sim \mathcal{N}(\mu^{n,i},\ (0.2\,\mu^{n,i})^2)$; $L \gets \max(1, \lceil X \rceil)$.
        \State $\mathrm{Out}^{n,i} \gets (t + L,\ S^{n,i} - \mathrm{OH}^{n,i})$.
    \EndIf
\EndFor
\end{algorithmic}
\vspace{2pt}\hrule height 0.6pt\vspace{8pt}
 
\subsubsection{Greedy first-fit packing}
\label{app:pack}
 
The same routine handles outbound shipping (sub-step 5 of Algorithm~\ref{alg:day}) and inter-warehouse pull (sub-step 6). A call attempts to place $Q$ units of a single item along a fixed path $P$, one unit at a time. For each unit, the routine walks $P$ edge by edge and selects the first container on that edge with enough residual volume. This keeps the order of units along the path consistent with the order in which they were requested.
 
\begin{algorithm}
\caption{\textsc{GreedyPack}: first-fit allocation of up to $Q$ units of item $i$ along a path $P$.}
\label{alg:pack}
\begin{algorithmic}[1]
\State \textbf{Input:} item $i$ with volume $v_i$; max units $Q$; path $P = (e_1, \ldots, e_k)$.
\State $\mathrm{placed} \gets 0$.
\For{$q = 1, \ldots, Q$}
    \State $\mathrm{slots} \gets [\,]$;\quad $\mathrm{ok} \gets \textsc{true}$.
    \For{$e \in P$}
        \State $j \gets$ smallest index with $\mathrm{containers}_e[j] \geq v_i$, or $\bot$ if none.
        \If{$j = \bot$}
            \State $\mathrm{ok} \gets \textsc{false}$;\quad \textbf{break}.
        \EndIf
        \State Append $(e, j)$ to $\mathrm{slots}$.
    \EndFor
    \If{not $\mathrm{ok}$}
        \State \textbf{break} \quad $\triangleright$ stop filling; do not retry a later unit
    \EndIf
    \For{$(e, j) \in \mathrm{slots}$}
        \State $\mathrm{containers}_e[j] \mathrel{-}= v_i$.
    \EndFor
    \State $\mathrm{placed} \gets \mathrm{placed} + 1$.
\EndFor
\State \Return $\mathrm{placed}$.
\end{algorithmic}
\end{algorithm}
 
\subsubsection{Intensity construction algorithm}
\label{app:demand-alg}
 
Algorithm~\ref{alg:demand} constructs the intensity tensor $\lambda$ used in \eqref{eq:intensity}. Of the five components it samples, four are drawn independently per item: yearly seasonal, weekly seasonal, AR(1) drift, and per-item burst train. The fifth, $G(t)$, is drawn once and reused across every item, with a per-item weight $g_i$. As noted in \S\ref{app:demand}, $G(t)$ is the component that lifts every item's intensity simultaneously, and is the mechanism by which demand forecasting and bottleneck forecasting become coupled.
 
\begin{algorithm}
\caption{\textsc{BuildIntensity}: sample the five-component intensity $\lambda$.}
\label{alg:demand}
\begin{algorithmic}[1]
\State \textbf{Input:} item list $\mathcal{I}$; horizon $T$; seed.
\Statex
\State \textbf{Global macro-shock process $G(t)$.}
\State $G \gets 0_T$;\quad draw $N$ uniformly from $\{5, \ldots, 11\}$.
\For{$k = 1, \ldots, N$}
    \State Draw $t^G_k$ uniform on $\{0, \ldots, T-1\}$;\quad $\Delta^G_k$ uniform on $\{180, \ldots, 1099\}$;\quad $h^G_k \sim \mathcal{U}[0.20, 0.60]$.
    \State Add trapezoidal pulse $\pi(\cdot)$ with parameters $(t^G_k, \Delta^G_k, h^G_k)$ to $G$.
\EndFor
\Statex
\State \textbf{Per-item components.}
\For{$i \in \mathcal{I}$}
    \State $\bar{\lambda}_i \sim \mathcal{U}[80, 250]$.
    \State \textbf{Yearly:} $S_i(t) \gets a^{(y)}_{1,i} \sin(2\pi t / 365 + \phi_i) + a^{(y)}_{2,i} \sin(4\pi t / 365 + 0.7\, \phi_i)$
    \Statex \quad with $a^{(y)}_{1,i} \sim \mathcal{U}[0.12, 0.28]$, $a^{(y)}_{2,i} \sim \mathcal{U}[0.04, 0.10]$, $\phi_i \sim \mathcal{U}[0, 2\pi]$.
    \State \textbf{Weekly:} $W_i(t) \gets a^{(w)}_i \sin(2\pi (t \bmod 7) / 7 + \psi_i)$ with $a^{(w)}_i \sim \mathcal{U}[0.04, 0.10]$, $\psi_i \sim \mathcal{U}[0, 2\pi]$.
    \State \textbf{AR(1) drift:} $\phi^{\mathrm{AR}}_i \sim \mathcal{U}[0.9990, 0.9996]$, $\sigma_{\mathrm{AR}, i} \sim \mathcal{U}[0.008, 0.018]$;
    \Statex \quad $A_i(0) \sim \mathcal{N}(0, 0.10^2)$;\quad $A_i(t) \gets \mathrm{clip}\!\big[\phi^{\mathrm{AR}}_i A_i(t-1) + \mathcal{N}(0, \sigma^2_{\mathrm{AR}, i}),\ \pm 0.6\big]$.
    \State \textbf{Per-item spikes $P_i(t)$:} burst rate $r_i \sim \mathcal{U}[2 \cdot 10^{-4}, 10^{-3}]$; at each $t$ with $\mathrm{Bern}(r_i) = 1$, add a trapezoidal pulse of duration $\mathcal{U}\{30, \ldots, 179\}$ and height $\mathcal{U}[0.20, 0.70]$.
    \State \textbf{Global sensitivity:} $g_i \sim \mathcal{U}[0.4, 1.2]$.
    \State \textbf{Assemble:} $\lambda_{i,t} \gets \bar{\lambda}_i \cdot \max\!\big(0.08,\ 1 + S_i(t) + W_i(t) + A_i(t) + P_i(t) + g_i\, G(t)\big)$.
\EndFor
\State \Return $\lambda$.
\end{algorithmic}
\end{algorithm}

Algorithm~\ref{alg:day} specifies the ordered sequence of seven sub-steps
that defines the transition map. App.~\ref{app:extensions} considers
alternative transition rules and summarizes the extensions implemented in
the released simulator, as well as additional mechanisms that are not
currently implemented. For each extension, we indicate whether the stated
Markov property and the conservation identities of
Proposition~\ref{prop:conservation} continue to hold.

\section{The right state space and its consequences}
\label{sec:state-space-consequences}

Unlike existing resources limited to per-node
demand or inventory
\citep{wasi2024supplygraph, naik2025bullode}, the state
vector $\xi_t$ in \sysname carries outstanding orders,
scheduled arrivals, and smoothed demand estimates
alongside per-node inventory. These auxiliary components
lift the system into a space where three properties
emerge: the dynamics are Markovian
(\S\ref{sec:markov-consequence}), three discrete
conservation laws hold exactly on every sample path
(\S\ref{app:properties}), and the discrete laws align
with classical continuum transport and queueing
equations (\S\ref{sec:conservation-continuous}).

\subsection{Markov property}
\label{sec:markov-consequence}

Without the auxiliary components $\mathrm{Out}_t$,
$\mathrm{IT}_t$, and $\tilde\lambda_t$ in the state,
the next-step inventory would depend on a window of
past dispatches and lead-time draws, not on $\xi_t$
alone: the dynamics would require history and the Markov
property would fail. By including them, the system
closes into a Markov chain
$\xi_{t+1} = \Psi(\xi_t, y_t, L_t)$ whose transition
probability acts linearly on the empirical distribution
of the state (\S\ref{app:kernel})--the right space for
scalable simulation, forward UQ, and downstream
stochastic control.

The resulting state representation contains
$C(3|\mathcal{N}|+1)$ fixed scalar dimensions (for $|\mathcal{N}|$ nodes and
catalogue size $C$), together with a variable-length
in-transit list $\mathrm{IT}_t$. In this high-dimensional space of $$\mathrm{dimension} \geq C(3|\mathcal{N}|+1),$$ the dynamics become
Markovian: these components contain all information
required to determine the next step's state without
reliance on historical trajectories, enabling efficient
simulation of the digital twin over arbitrary horizons.
The full ISOMORPH release provides datasets at catalogue
sizes $C=50$ and $C=200$, corresponding to more than
$2{,}000$ and $8{,}000$ fixed scalar state dimensions,
respectively.

\subsection{Discrete conservation laws}
\label{app:properties}

The same state completeness that enables the Markov
property also reveals conservation structure structurally encoded
in the transition function $\Psi$. The network-internal
stock $I_{i,t}$ in~\eqref{eq:stock} sums over on-hand,
in-transit, and scheduled-arrival components---all
present in $\xi_t$ by design---so the conservation laws
below are expressible as closed-form functions of the
current state.

\begin{proposition}[Conservation of units and demand
events]
\label{prop:conservation}
For each item $i \in \mathcal{I}$ and time unit
$t \in \{0, \ldots, T-1\}$, let $u^e_{i,t}$ denote the
units of item $i$ placed onto edge $e$ at sub-steps
(5)--(6) of Algorithm~\ref{alg:day}. Define per-node
receipts and dispatches in $[t, t+1]$ as
\begin{align}
R^{n,i}_t &:=
\begin{cases}
q, & n \neq d^\star \text{ and } \mathrm{Out}^{n,i}_t = (t{+}1, q), \\
(A_{i,t} - B^{d^\star,i}_t)_+, & n = d^\star, \\
0, & \text{otherwise};
\end{cases} \label{eq:per-node-R} \\
D^{n,i}_t &:=
\begin{cases}
\sum_{(n, v) \in \mathcal{E}} u^{(n,v)}_{i,t}, & n \neq d^\star, \\
\min\bigl(\mathrm{OH}^{d^\star,i}_{(2)},\, y_{i,t}\bigr), & n = d^\star.
\end{cases} \label{eq:per-node-D}
\end{align}
Define the network-internal stock
\begin{equation}
    I_{i,t} := \sum_{n \in \mathcal{N}} \mathrm{OH}^{n,i}_t + \sum_{(u,v) \in \mathcal{E} , \\ v \neq d^\star} q^{u \to v}_{i,t} + \sum_{\tau > t} \mathrm{IT}_t[\tau, i],
    \label{eq:stock}
\end{equation}
where $q^{u \to v}_{i,t}$ is the per-edge in-transit
count defined in \S\ref{app:formal} and
$\mathrm{IT}_t[\tau, i]$ is the multiplicity of arrivals
scheduled at $d^\star$ at time unit $\tau$ for item $i$.
Define the source arrivals between $t$ and $t+1$ as
\begin{equation}
    A^{\mathrm{src}}_{i,t} := \sum_{n \in \mathcal{N}_{\mathrm{src}}} \mathbf{1}\!\left\{\mathrm{Out}^{n,i}_t = (\tau, q),\ \tau = t+1\right\} \cdot q,
    \label{eq:source-arrivals}
\end{equation}
the mass arriving at non-destination nodes from sources
between time units $t$ and $t+1$.
With $A_{i,t} := \mathrm{IT}_t[t{+}1, i]$ and
$\mathrm{OH}^{d^\star,i}_{(2)} = \mathrm{OH}^{d^\star,i}_t
+ (A_{i,t} - B^{d^\star,i}_t)_+$ as defined
in~\eqref{eq:dest-arrivals}--\eqref{eq:dest-arrival-update},
define the total destination outflow as the sum of the
arrival-clears-backlog flow and the immediate-service
flow:
\begin{equation}
S_{i,t} \;:=\; \min\!\bigl(A_{i,t},\, B^{d^\star,i}_t\bigr) \;+\; \min\!\bigl(\mathrm{OH}^{d^\star,i}_{(2)},\, y_{i,t}\bigr).
\label{eq:served}
\end{equation}
Then the chain $(\xi_t)_{t \geq 0}$ satisfies, almost
surely for all $t \in \{0, \ldots, T-2\}$ and
$i \in \mathcal{I}$:
\begin{align}
    \text{(a) Per-node mass conservation:} \qquad & \mathrm{OH}^{n,i}_{t+1} = \mathrm{OH}^{n,i}_t + R^{n,i}_t - D^{n,i}_t, \quad \forall n \in \mathcal{N},
    \label{eq:per-node-mass} \\
    \text{(b) Global mass conservation:} \qquad & I_{i,t+1} = I_{i,t} + A^{\mathrm{src}}_{i,t} - S_{i,t},
    \label{eq:conservation} \\
    \text{(c) Backlog conservation:} \qquad & B_{t+1}^{d^\star, i} = B_t^{d^\star, i} + \big(y_{i,t} - \mathrm{OH}_t^{d^\star, i}\big)_+.
    \label{eq:backlog-thm}
\end{align}
\end{proposition}

\begin{proof}
\textbf{Part (a).} Fix $n \in \mathcal{N}$ and
$i \in \mathcal{I}$. We trace changes to
$\mathrm{OH}^{n,i}$ across the seven sub-steps of
Algorithm~\ref{alg:day} during the transition from $t$
to $t+1$.

\emph{Case $n \neq d^\star$.}
Sub-step (1) increments $\mathrm{OH}^{n,i}$ by $q$ when
$\mathrm{Out}^{n,i}_t = (t{+}1, q)$, and by zero
otherwise; this matches $R^{n,i}_t$ as defined
in~\eqref{eq:per-node-R}.
Sub-steps (2) and (4) act at $d^\star$ alone and leave
$\mathrm{OH}^{n,i}$ unchanged.
Sub-step (3) resets edge containers without modifying
any node's on-hand.
Sub-steps (5) and (6) decrement $\mathrm{OH}^{n,i}$ by
exactly the units placed onto $n$'s outgoing edges via
\textsc{GreedyPack}, i.e.,
$\sum_{(n,v) \in \mathcal{E}} u^{(n,v)}_{i,t}
= D^{n,i}_t$.
Sub-step (7) modifies $\mathrm{Out}^{n,i}$ for
$n \in \mathcal{N}_{\mathrm{src}}$ but does not modify
$\mathrm{OH}^{n,i}$.
Collecting:
$\mathrm{OH}^{n,i}_{t+1} = \mathrm{OH}^{n,i}_t
+ R^{n,i}_t - D^{n,i}_t$.

\emph{Case $n = d^\star$.}
Sub-step (1) acts at non-destination nodes.
Sub-step (2) increments $\mathrm{OH}^{d^\star,i}$ by
$(A_{i,t} - B^{d^\star,i}_t)_+ = R^{d^\star,i}_t$
(the post-backlog residual of arriving units,
per~\eqref{eq:dest-arrival-update}).
Sub-step (3) resets edge containers.
Sub-step (4) decrements $\mathrm{OH}^{d^\star,i}$ by
$\min(\mathrm{OH}^{d^\star,i}_{(2)}, y_{i,t})
= D^{d^\star,i}_t$ (the served quantity, second term of
$S_{i,t}$ in~\eqref{eq:served}).
Sub-steps (5)--(7) do not modify
$\mathrm{OH}^{d^\star,i}$: the destination is excluded
from the warehouse loop in (5), is not an upstream
supplier in (6), and is not a source in (7).
Hence $\mathrm{OH}^{d^\star,i}_{t+1}
= \mathrm{OH}^{d^\star,i}_t + R^{d^\star,i}_t
- D^{d^\star,i}_t$.

\textbf{Part (b).} Summing Part~(a) over all nodes and
adding the in-transit and scheduled-arrival components:
routing along edges transfers units between
$\sum_n \mathrm{OH}^{n,i}$ and
$\sum_e q^{u \to v}_i$ without changing the total
$I_{i,t}$ defined in~\eqref{eq:stock}. Two
mechanisms change $I_{i,t}$: source arrivals
(sub-steps i and vii), which add
$A^{\mathrm{src}}_{i,t}$, and destination service
(sub-steps ii and iv), which removes $S_{i,t}$. No other sub-step changes it.
Therefore
$I_{i,t+1} - I_{i,t} = A^{\mathrm{src}}_{i,t} - S_{i,t}$
almost surely.

\textbf{Part~(c).} Immediate
from~\eqref{eq:dest-arrival-update}--\eqref{eq:backlog-update}
of \S\ref{app:replenish}: at sub-step~(ii), backlog is
updated to
$B^{d^\star,i}_{(2)} = (B^{d^\star,i}_t - A_{i,t})_+$;
at sub-step~(iv), the unfilled remainder
$(y_{i,t} - \mathrm{OH}^{d^\star,i}_{(2)})_+$ is added
to $B^{d^\star,i}_{(2)}$ to give $B^{d^\star,i}_{t+1}$.
No other sub-step modifies $B^{d^\star,i}$.
\end{proof}

All three identities are preserved exactly---not
approximately---because the simulator's transitions are
bookkeeping operations on integer counts, with no
rounding or stochastic loss inside the network. Any
deviation at any $(t, i)$ in a released run would
indicate a bug in the simulator, corruption in the
released files, or a mismatch between the model and its
algorithmic realization.

\paragraph{Conservation laws as structural invariants
of the DBN.}
All three laws hold pathwise---for every realization of
the random external inputs $(y_t, L_t)$, and not merely in
expectation. The two boundary terms of the global mass
conservation align with the two random input nodes in
the Markov chain (Figure~\ref{fig:pgm}):
$A^{\mathrm{src}}_{i,t}$ is determined by the lead-time
vector $L_t$ (which controls when source orders arrive),
and the service component of $S_{i,t}$ is determined by
the demand vector $y_t$ together with the on-hand state
at $d^\star$. The places where randomness enters the
chain are the same places where mass crosses the
boundary of the conserved domain. Internal flows
(routing, packing, inter-warehouse transfers) are
deterministic given $(y_t, L_t)$ and conservative by
construction. The conservation laws are therefore
structural invariants of the Markov chain: they
constrain every sample path the simulator can produce.

\paragraph{Conservation laws as a fidelity check.}
\label{sec:datasets-conservation}%
Because every unit and every demand event is accounted for at every step,
downstream metrics---fill rate, service level, backlog counts---are well
defined and recoverable from the released data. When users add their own
control rules to the simulator, these conservation laws, or adapted versions
of them, can verify the fidelity of the modified implementation.

\subsection{Continuous-time fluid limit: coupled transport and queueing}
\label{sec:conservation-continuous}

This section highlights the mathematical structure of \sysname{} and connects it to the literature on queueing systems and their fluid limits. 
First, the three discrete conservation laws of
Proposition~\ref{prop:conservation} are not isolated
bookkeeping identities: they are the pre-limit forms of
classical continuum equations that govern the fluid limit
of the Markov chain. Consider a sequence of systems
indexed by a system-size parameter $N$, in which the
catalogue size, demand intensities, edge capacities, and
control thresholds $(s, S)$ are all scaled by $N$
simultaneously. Under this Kurtz-type
law-of-large-numbers
scaling~\citep{kurtz1970solutions}, each individual unit
becomes a vanishing fraction of the total, and the
rescaled density
$\rho^{N,n,i}(t) = N^{-1}\mathrm{OH}^{n,i}_{\lfloor tN\rfloor}$
converges in probability to the deterministic solution
of a coupled system: a transport equation on the graph
$\mathcal{G}$ governing inventory, and a
Skorokhod-reflected ODE at $d^\star$ governing backlog.
The discrete conservation laws hold at every finite $N$
and pass to the limit, making the continuum system their
deterministic counterpart.

At the discrete level, the simulator is a jump process
on a network of control volumes: each node
$n \in \mathcal{N}$ is a fixed control volume that
accumulates on-hand inventory $\mathrm{OH}^{n,i}_t$,
and the edges $\mathcal{E}$ are conduits that carry
discrete unit-valued jumps between control volumes.
The fluid limit smooths these integer-valued jumps into
continuous densities and fluxes.

\paragraph{Transport (inventory side).}
In the fluid limit, the per-node, per-item inventory
satisfies a delay-differential equation on
$\mathcal{G}$: for each node $n \in \mathcal{N}$ and
item $i \in \mathcal{I}$,
\begin{equation}
\partial_t \rho^{n,i}(t) = \sum_{u:(u,n)\in\mathcal{E}}
  \phi^{u\to n}_i(t - \tau_{u,n}) -
  \sum_{v:(n,v)\in\mathcal{E}} \phi^{n\to v}_i(t)
  + r^n_i(t)\,\mathbf{1}\{n \in \mathcal{N}_{\mathrm{src}}\}
  - s^n_i(t)\,\mathbf{1}\{n = d^\star\},
\label{eq:fluid-limit}
\end{equation}
subject to edge capacity constraints
$\sum_i v_i \phi^{u\to v}_i(t) \leq \kappa_{u,v}$.
This is the continuous form of the local continuity
equation $\partial_t \rho + \nabla \cdot j = \sigma$
on the graph, where the delayed inflows from upstream
edges play the role of flux divergence, and source
replenishment $r^n_i$ and destination service $s^n_i$
act as boundary source/sink terms.

The per-node mass conservation~\eqref{eq:per-node-mass}
is the time-discretized form of~\eqref{eq:fluid-limit}:
on-hand inventory $\mathrm{OH}^{n,i}_t$ is the local
density in the control volume at node $n$, receipts
$R^{n,i}_t$ and dispatches $D^{n,i}_t$ are the discrete
inflow and outflow across the control-volume boundary,
and their difference is the discrete divergence.
Integrating over the entire graph--summing over all
control volumes--yields the global mass
conservation~\eqref{eq:conservation}: internal fluxes
cancel in pairs, leaving the boundary terms, giving
the discrete analog of
$\frac{d}{dt}\int_\Omega \rho
= -\int_{\partial\Omega} j \cdot n
+ \int_\Omega \sigma$.

\paragraph{Queueing (demand side).}
In classical queueing theory, the queue length $Q$
evolves according to the fluid model
$dQ/dt = \lambda(t) - \mu(t)$, where $\lambda$ is the
arrival rate and $\mu$ is the service rate, subject to
the constraint $Q \geq 0$ (Skorokhod reflection). In
\sysname, the backlog $B^{d^\star,i}_t$ at the
destination $d^\star$ for each item $i \in \mathcal{I}$
plays the role of $Q$, demand $y_{i,t}$ plays the role
of $\lambda$, and fulfilled orders
$\min(\mathrm{OH}^{d^\star,i}_t, y_{i,t})$ play the
role of $\mu$. The discrete backlog
conservation~\eqref{eq:backlog-thm} is the
time-discretized form of this fluid model, with the
$(\cdot)_+$ operator enforcing non-negativity.

In the fluid limit, the rescaled backlog at the
destination $d^\star$ for each item
$i \in \mathcal{I}$ converges to the
Skorokhod-reflected ODE
\begin{equation}
\partial_t b^{d^\star,i}
= (\lambda_i - \rho^{d^\star,i})_+
- (A_i - b^{d^\star,i})_+,
\label{eq:fluid-backlog}
\end{equation}
where the first term adds unmet demand (arrivals
exceeding on-hand density) and the second term drains
backlog when shipments arrive at $d^\star$.

\paragraph{Coupling at the destination.}
The transport and queueing subsystems couple at
$d^\star$: the service term $S_{i,t}$ is simultaneously
the boundary outflow of the transport
equation~\eqref{eq:fluid-limit} and the service rate of
the queue~\eqref{eq:fluid-backlog}. The transport
network's ability to deliver units to $d^\star$
determines the queue's service rate, and the queue's
backlog state feeds back into dispatch targets through
the proactive shipping formula of sub-step~(5). This
coupling persists in both the discrete and continuum
settings.
Table~\ref{tab:conservation-correspondence} summarizes
the full correspondence.

\begin{table}[h]
\centering\small
\caption{Correspondence between the fluid-limit
equations and the discrete conservation laws of
Proposition~\ref{prop:conservation}. Each node
$n \in \mathcal{N}$ serves as a control volume; edges
carry discrete unit-valued jumps that become continuous
fluxes in the fluid limit.}
\label{tab:conservation-correspondence}
\begin{tabular}{@{}l l@{}}
\toprule
Fluid limit & Discrete (\sysname) \\
\midrule
\multicolumn{2}{@{}l}{\emph{(a) Per-node mass
  conservation~\eqref{eq:per-node-mass}: local
  continuity
  equation~\eqref{eq:fluid-limit}}} \\
\quad density $\rho^{n,i}(t)$ &
  on-hand inventory $\mathrm{OH}^{n,i}_t$ \\
\quad flux $\phi^{u \to v}_i$ &
  shipment flow $u^{e}_{i,t}$ along edge $e$ \\
\quad divergence &
  net outflow from control volume at node $n$:
  $D^{n,i}_t - R^{n,i}_t$
  (Eqs.~\ref{eq:per-node-D},~\ref{eq:per-node-R}) \\
\quad source/sink $r^n_i,\, s^n_i$ &
  $A^{\mathrm{src}}_{i,t}$~\eqref{eq:source-arrivals}
  (source inflow),
  $S_{i,t}$~\eqref{eq:served} (destination outflow) \\
\midrule
\multicolumn{2}{@{}l}{\emph{(b) Global mass
  conservation~\eqref{eq:conservation}: integrated
  form}} \\
\quad integrated density
  $\int_\Omega \rho$ &
  network-internal stock
  $I_{i,t}$~\eqref{eq:stock} \\
\quad boundary flux
  $\int_{\partial\Omega} j \cdot n$ &
  $A^{\mathrm{src}}_{i,t}$~\eqref{eq:source-arrivals}
  (in), $S_{i,t}$~\eqref{eq:served} (out) \\
\midrule
\multicolumn{2}{@{}l}{\emph{(c) Backlog
  conservation~\eqref{eq:backlog-thm}: reflected
  ODE~\eqref{eq:fluid-backlog}}} \\
\quad queue length $b^{d^\star,i}$ &
  backlog $B^{d^\star,i}_t$ \\
\quad arrival rate $\lambda_i$ &
  demand $y_{i,t}$ \\
\quad service rate &
  $\min(\mathrm{OH}^{d^\star,i}_t,\, y_{i,t})$ \\
\midrule
\multicolumn{2}{@{}l}{\emph{Coupling at $d^\star$:
  $S_{i,t}$ = transport outflow = queue service
  rate}} \\
\bottomrule
\end{tabular}
\end{table}

Each sub-step of Algorithm~\ref{alg:day} contributes a
jump of size $O(1/N)$ at rate $O(N)$ to the rescaled
generator, placing the chain in the regime where Kurtz's
theorem~\citep{kurtz1970solutions} applies. The $(s, S)$
replenishment policy introduces switching boundaries at
the reorder threshold, producing either a hybrid system
with Filippov solutions or, under a smoothed base-stock
control, a classical ODE; both are standard in the
fluid-network
literature~\citep{chen2001fundamentals,
mandelbaum1998state}. The edge capacity constraints
produce a complementarity structure--flow at capacity
implies zero slack at the controller--shared with fluid
models of multi-class queueing
networks~\citep{whitt2002stochastic}. A rigorous
formulation of the limit theorem and the associated
diffusion correction (CLT-type Brownian fluctuations
around $\rho$) is of independent mathematical, computational  and statistical interest.

\section{Released datasets}
\label{sec:datasets}

The simulator provides a configurable generative process, while the released
datasets instantiate this process as concrete forecasting and uncertainty-
quantification benchmarks in \S~\ref{sec:foundation}. We first describe the structural configurations,
scenario knobs, and output sequences used to construct the release. We then
illustrate observable supply-chain phenomena that emerge from these
configurations, including demand amplification, inventory and transportation
dynamics, service degradation, and disruption propagation. The resulting
dataset family combines fixed structural settings with controlled variation
in the operating regime, providing both long-horizon forecasting series and
network-state trajectories for downstream tasks.

\subsection{What the user configures and what the simulator records
}
\label{sec:datasets-io}

The released datasets are configured through two kinds of parameters.
\emph{Structural parameters} are fixed within a release and define what dataset
is being generated: the graph $\mathcal{G}$, the catalogue size $C$, and the
time horizon $T$. \emph{Scenario knobs} govern the operating regime of the
network and fall into three categories: demand structure, inventory, and edge
transport. The complete parameter tables and their mechanisms within the
simulator are in App.~\ref{app:param-values}.

\paragraph{Structural parameters.}
All releases in this paper use the following configuration:
\begin{itemize}[leftmargin=1.4em,itemsep=2pt,topsep=2pt]
  \item \emph{Graph}: the 13-node U.S.\ network of
    Figure~\ref{fig:pgm}(a)---three sources
    $\{$San Francisco, St.\,Louis, Orlando$\}$, nine
    intermediate warehouses organized in five tiers
    (hub, three mid-level tiers, and last-mile), and
    destination NewYork. All three sources converge at
    Nashville and fan out through Atlanta; every shipment
    to NewYork passes through one of two last-mile edges
    (Philadelphia$\to$NewYork or Baltimore$\to$NewYork).
  \item \emph{Cardinality}: two regimes, $C{=}50$ and
    $C{=}200$. The two regimes share the same demand
    parameters, reorder thresholds $(s, S)$, routing, and
    packing rules; the edge capacities differ.
    Upstream edge container volumes scale linearly
    with~$C$, so they are four times larger at $C{=}200$
    when the per-step container count~$K_e$ is held fixed.
    The two last-mile edges are sized by back-solving the
    realized mean intensity
    (\S\ref{app:params:backsolve}).
    The first 50 items at $C{=}200$ coincide with the
    entire $C{=}50$ release under the same seed, so
    $C{=}50$ is suitable for faster model development
    with fewer item channels while $C{=}200$ tests
    scalability to a larger catalogue.
  \item \emph{Horizon}: $T{=}7{,}300$ time units (20 years at daily resolution), comparable to daily-resolution TSF benchmarks such as Exchange
($T{=}7{,}588$). 
\end{itemize}

\paragraph{Scenario knobs.}
We vary six scenario knobs to generate different demand regimes and network
responses. Three knobs perturb the demand process: \emph{drift} controls
the persistence of long-lived demand-level changes, \emph{burst} controls
transient item-specific demand surges, and \emph{shock} controls
network-wide demand events shared across items. Three additional knobs
perturb the network response through transportation capacity, inventory
buffering, and source lead times. Each knob is varied independently around
a common baseline, with additional compound scenarios combining demand-side
perturbations. The resulting scenarios provide controlled variation in the
temporal and cross-item structure of demand and in the network conditions
under which it propagates. Complete sweep definitions, parameter settings,
and scenario names are given in App.~\ref{app:scenarios}.

\paragraph{Output sequences.}
Each rollout records the full network state trajectory:
per-item demand, service, and backlog at the destination;
per-node inventory levels; in-transit volumes; and
per-edge shipments with resolved paths. Dataset schemas are given in
Appendix~\ref{app:schema}. For the TSF evaluation of
\S\ref{sec:foundation}, we use four target families:
\emph{demand}, \emph{backlog}, and \emph{fill rate} at the item level,
and \emph{edge utilization} at the edge level. Demand is taken from
\texttt{daily\_records.csv}; backlog from
\texttt{backlog\_history.csv} at $n=d^\star$; and fill rate is computed as
\texttt{served\_from\_stock}/\texttt{demand} from
\texttt{daily\_records.csv}. Edge utilization is taken from
\texttt{edge\_utilisation.npy}, with columns ordered according to
\texttt{edge\_list.csv}. Thus, the first three targets are indexed by item,
whereas edge utilization is indexed by edge. Diverse dynamics emerge from
different scenario configurations (scenario definitions in
Table~\ref{tab:scenario-sweeps}, App.~\ref{app:scenarios};
instantiation in \S\ref{sec:uq}); Figure~\ref{fig:family-compact}
illustrates this for a single item across four scenarios and two output
variables.

\begin{figure}[t]
\centering
\includegraphics[width=\linewidth]{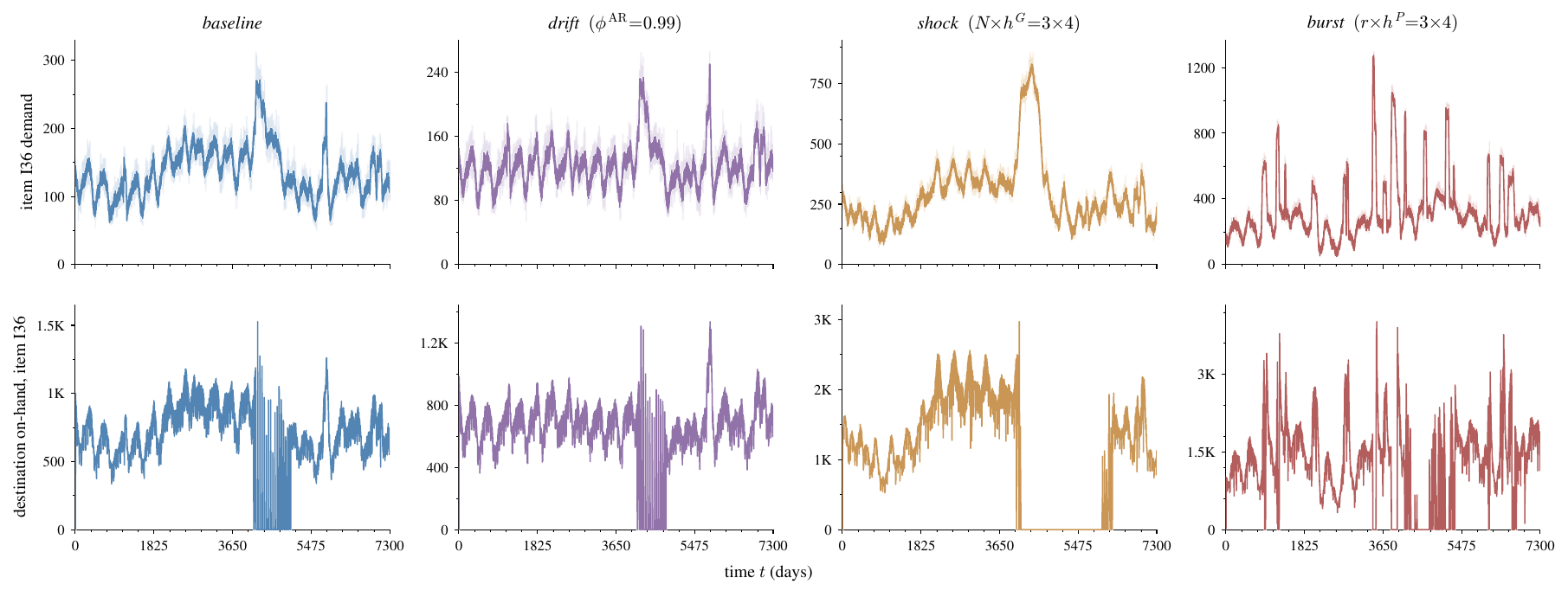}
\caption{Diversity of dynamics generated by the simulator
for a single item (I36) across four demand-side scenarios, over the full
released horizon ($T{=}7{,}300$ daily steps; $x$-axis in days).
\textbf{Top row:} realized per-item demand $y_{i,t}$ (light) with the
deterministic intensity $\lambda_{i,t}$ of Eq.~\eqref{eq:intensity} overlaid
(solid). \textbf{Bottom row:} destination on-hand stock for the same item,
i.e.\ the network's response to that demand. Columns are baseline, drift
($\phi^{\mathrm{AR}}{=}0.99$), shock ($N{\times}h^{G}{=}3{\times}4$), and
burst ($r{\times}h^{P}{=}3{\times}4$). The four columns differ in both the
demand structure and the resulting stockout signature: shock produces one
sustained network-wide surge (peak intensity roughly $2.5\times$ baseline)
followed by an extended depletion window, burst produces repeated sharp
per-item spikes and correspondingly frequent stock crashes, and drift
compresses the demand range relative to baseline. Note the differing $y$-axis
scales across columns.}
\label{fig:family-compact}
\end{figure}

\subsection{Reproducing known supply-chain behaviors with \sysname{}
}
\label{sec:datasets-validation}
\label{sec:datasets-validation:phenomena}

To demonstrate that \sysname{} can generate known supply-chain behaviors, we
examine a set of commonly observed phenomena that provide interpretable
examples of the simulator's dynamics. These include demand amplification,
inventory and transportation dynamics, service degradation, and disruption
propagation. Where a corresponding public dataset is available, we compare
the simulated and observed behaviors in terms of their qualitative patterns
and effect sizes; otherwise, we illustrate the phenomenon using \sysname{}
alone. Since the network structure, capacities, and control rules behind those
datasets are not available, these comparisons assess behavioral
consistency rather than trajectory-level agreement with a particular
operational system. Because some of these phenomena require configurations beyond the
baseline scenarios, we therefore include dedicated rollouts in the released
dataset to demonstrate them directly. The corresponding scenario definitions
and directory organization are given in App.~\ref{app:scenarios} and
App.~\ref{app:schema-folders}.

\paragraph{Bullwhip effect.}
\label{sec:datasets-validation:bullwhip}%
The bullwhip effect describes the amplification of demand variability as
material moves upstream from the destination toward the source. We compute
the corresponding amplification ratio at each non-source node and summarize
the results by tier in Table~\ref{tab:bullwhip-validation}, averaged over the $29$-rollout cohort of \S\ref{sec:foundation-protocol}. The five upstream tiers ($1.07$--$1.45$) fall within the $[0.97,1.90]$ range spanning the 50th--90th percentiles of \citet{cachon2007search}, and the destination ($1.93 \pm 0.41$) lies at its upper edge. 
Per-node values are reported in Table~~\ref{tab:bullwhip-pernode}
(App.~\ref{app:additional:bullwhip}).

\begin{table}[h]
\centering
\small
\caption{Bullwhip amplification by tier on the $C{=}50$ release.
$\bar B = \mathrm{Var}(\mathrm{inflow})\,/\,\mathrm{Var}(\mathrm{outflow})$,
computed per (node, item) and averaged
over items, then over nodes in each tier, then over the $29$-rollout cohort of
\S\ref{sec:foundation-protocol}; the monthly column reports mean $\pm$ 95\% Student-$t$ CI over the $29$ rollouts. The three source nodes are excluded because they have no inflow on the released edges. The ``Daily'' column is inflated near the destination because shipments arrive in bursts within a month; the ``Monthly'' column smooths this out and is the granularity comparable to \citet{cachon2007search}'s industry-level evidence.}
\label{tab:bullwhip-validation}
\begin{tabular}{@{}l p{0.34\linewidth} c c@{}}
\toprule
Tier         & Nodes                              & Daily $\bar B$ & Monthly $\bar B$ \\
\midrule
Destination  & NewYork                            & $22.64$ & $1.93 \pm 0.41$ \\
Tier-5 (LM)  & Baltimore, Philadelphia            & $14.73$ & $1.45 \pm 0.10$ \\
Tier-4       & Columbus, Richmond                 & $1.36$  & $1.24 \pm 0.06$ \\
Tier-3       & Charlotte, Chicago, Memphis        & $1.17$  & $1.22 \pm 0.03$ \\
Tier-2       & Atlanta                            & $1.30$  & $1.14 \pm 0.01$ \\
Hub          & Nashville                          & $0.97$  & $1.07 \pm 0.01$ \\
\bottomrule
\end{tabular}
\end{table}

\paragraph{Inventory oscillation.}
Figure~\ref{fig:inventory-oscillation} illustrates the inventory oscillations generated by the $(s,S)$
replenishment cycle. Individual items exhibit recurring sawtooth patterns,
with cycle timing and amplitude determined by the replenishment and demand
dynamics. Across the scenario library, the detrended $\mathrm{AC}(7)$ of
on-hand inventory ranges from $-0.00$ to $0.98$, spanning the
$[0.31,0.72]$ range observed across SupplyGraph SKUs 
(App.~\ref{app:additional:phenomena}). At a
single node, different items exhibit similar oscillation periods with
staggered phases, so aggregation across items can substantially smooth the
pattern.

\paragraph{High edge-utilization periods.}
Figure~\ref{fig:edge-utilization}
illustrates how transportation capacity shapes high-utilization
periods across the network. At baseline capacity, no released edge spends
more than $0.33\%$ of days at or above $U_{e,t}=0.8$; under a uniform
$0.3\times$ capacity scaling, the most highly utilized edge reaches $7.2\%$
of days above the same threshold. Capacity tightening therefore creates
localized periods of elevated edge utilization. The full edge-level results
are shown in Figure~\ref{fig:edge-utilization}b.

\paragraph{Disruption and delivery lateness.}
Figure~\ref{fig:disruption-lateness} illustrates disruption-related service
degradation in both \sysname{} and Global Supply Chain Disruption
\citep{bertnardo_mario_uskono_2026}, a second synthetic dataset whose
generating process is not documented by its publisher. In that dataset,
disruption events are associated with elevated late-delivery rates on several
shipping lanes. In
\sysname{}, removing a corridor edge without an alternate route produces a
substantial increase in destination unfilled-on-time demand, whereas removing
an edge with an alternate route produces a smaller response. A matched
no-cut control remains unchanged. Both series are synthetic and measure
different service-level quantities, so the comparison illustrates the
corresponding disruption response rather than matching the outcome variables
directly, and it carries no claim about either dataset's realism.

\paragraph{Ripple and cascading disruption.}
\label{sec:datasets-validation:disruption}%
The dynamic edge-cut configuration produces a disruption that propagates
through successive echelons of the network (Figure~8). After the
Nashville$\rightarrow$Atlanta edge is removed, downstream edge flows and
node inventories decline progressively with distance from the cut. The
propagation is not instantaneous or uniform: transit times and intermediate
inventory buffers shape both the timing and magnitude of the response.
When the edge is restored, recovery begins but can extend beyond the
disruption window. This scenario provides an example of multi-echelon
disruption propagation that can be inspected directly from the released
network-state trajectories.

\paragraph{Bottleneck migration.}
Figure~\ref{fig:bottleneck-migration} illustrates how a temporary edge disruption can redistribute flow
across alternative routes. Removing the Atlanta$\rightarrow$Memphis edge
causes Baltimore's inbound volume to shift toward the Columbus and Richmond
routes before returning toward its original composition after restoration.
The response is delayed because the affected route can continue supplying
from buffer stock during the disruption. Thus, the network can exhibit
route-level supplier migration rather than simple loss of downstream flow.

\begin{figure}[t]
\centering
\begin{subfigure}[t]{0.49\linewidth}
  \centering
  \includegraphics[width=\linewidth]{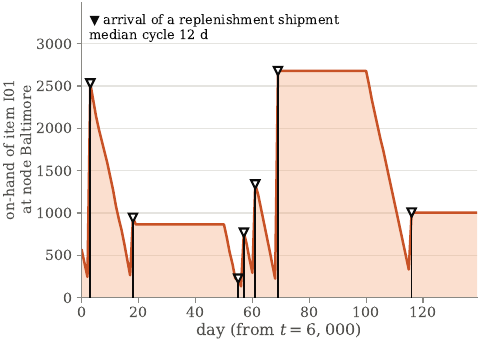}
  \subcaption{On-hand inventory of item I01 over 140 days.}
  \label{fig:inventory-oscillation:cycle}
\end{subfigure}\hfill
\begin{subfigure}[t]{0.49\linewidth}
  \centering
  \includegraphics[width=\linewidth]{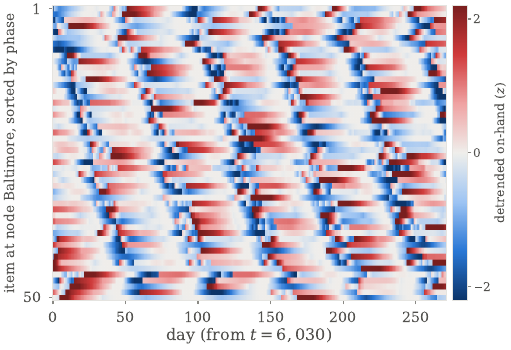}
  \subcaption{Detrended and standardized on-hand inventory for all 50 items over 300 days, sorted by the timing of the first trough.}
  \label{fig:inventory-oscillation:phase}
\end{subfigure}
\caption{\textbf{Inventory oscillation in \sysname{}.} Both panels show the baseline
$C=50$ rollout at Baltimore. (a) The on-hand inventory of item I01 over
140 days, showing the sawtooth pattern produced by the $(s,S)$
replenishment cycle. (b) Detrended and standardized on-hand inventory for
all 50 items over 300 days, sorted by the timing of the first trough. The
staggered phases across items smooth the oscillation when aggregated.}
\label{fig:inventory-oscillation}
\end{figure}

\begin{figure}[t]
\centering
\begin{subfigure}[t]{0.505\linewidth}
  \centering
  \includegraphics[width=\linewidth]{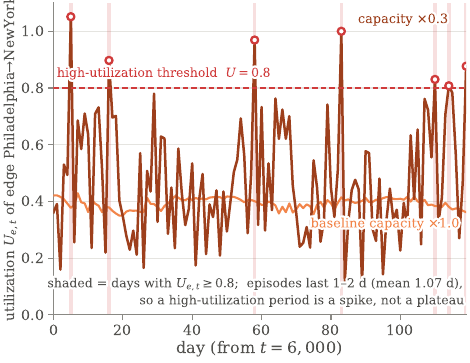}
  \subcaption{}
  \label{fig:edge-utilization:hot}
\end{subfigure}\hfill
\begin{subfigure}[t]{0.475\linewidth}
  \centering
  \includegraphics[width=\linewidth]{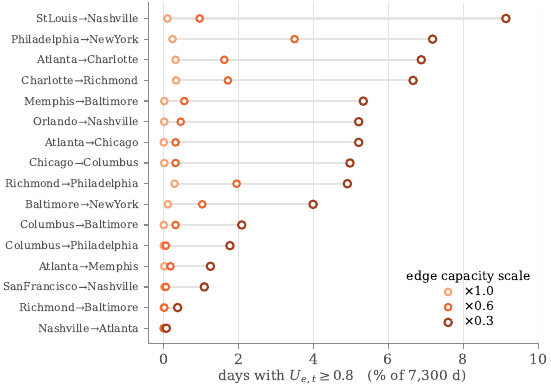}
  \subcaption{}
  \label{fig:edge-utilization:rank}
\end{subfigure}
\caption{\textbf{High edge-utilization periods in \sysname{}.} (a) Utilization of the Philadelphia$\rightarrow$NewYork edge
over 120 days at baseline and under $0.3\times$ capacity. (b) Fraction of
the 7,300 days with $U_{e,t}\geq0.8$ for all 16 edges under three uniform
capacity scalings. Capacity tightening produces localized periods of high
edge utilization, with different edges affected according to the network
topology and routing.}
\label{fig:edge-utilization}
\end{figure}

\begin{figure}[t]
\centering
\begin{subfigure}[t]{0.49\linewidth}
  \centering
  \includegraphics[width=\linewidth]{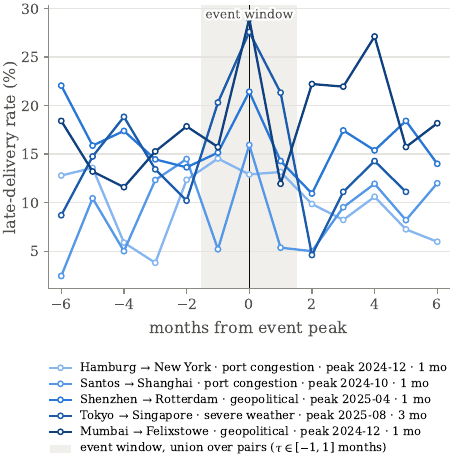}
  \subcaption{Global Supply Chain Disruption, a synthetic dataset
\citep{bertnardo_mario_uskono_2026} (10,000 orders, 2024-01 to
2026-01): monthly late-delivery rates for five lane--event-type pairs,
aligned by months from each pair's peak disruption month.}
  \label{fig:disruption-lateness:real}
\end{subfigure}\hfill
\begin{subfigure}[t]{0.49\linewidth}
  \centering
  \includegraphics[width=\linewidth]{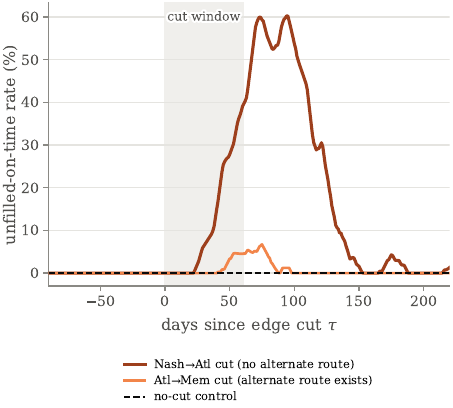}
  \subcaption{\sysname{}:
destination unfilled-on-time rate after 60-day cuts to
Nashville$\rightarrow$Atlanta (no alternate route) and
Atlanta$\rightarrow$Memphis (alternate route), compared with a matched
no-cut control.}
  \label{fig:disruption-lateness:iso}
\end{subfigure}
\caption{\textbf{Disruption-related service degradation in two synthetic
supply-chain datasets.}
(a) Global Supply Chain Disruption, a synthetic dataset: event-aligned
late-delivery rates for five lane--event-type pairs. (b) \sysname{}: destination unfilled-on-time
rate following two 60-day edge cuts, compared with a matched no-cut control.
The two panels use different service-level measures; the
comparison illustrates the response to disruption rather than a
one-to-one correspondence between the outcome variables.}
\label{fig:disruption-lateness}
\end{figure}

\begin{figure}[t]
\centering
\begin{subfigure}[t]{0.49\linewidth}
  \centering
  \includegraphics[width=\linewidth]{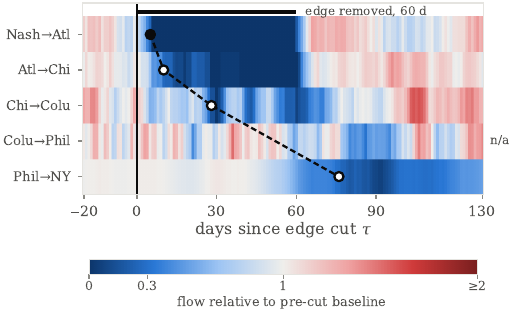}
  \subcaption{Flow on each downstream corridor edge, measured as the 7-day mean
flow relative to its own pre-cut 14-day baseline, after removing the
Nashville$\rightarrow$Atlanta edge for 60 days. The filled circle marks the
removed edge, and open circles indicate the first time each downstream edge
falls below 30\% of its baseline. Flow declines progressively farther
downstream, while Columbus$\rightarrow$Philadelphia remains largely
unaffected because Philadelphia has an alternate Richmond branch.}
  \label{fig:ripple-cascade:flow}
\end{subfigure}\hfill
\begin{subfigure}[t]{0.49\linewidth}
  \centering
  \includegraphics[width=\linewidth]{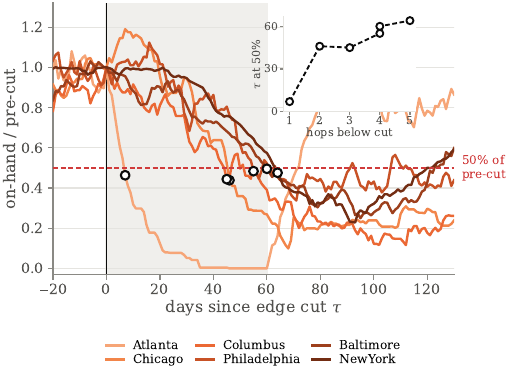}
  \subcaption{On-hand inventory at downstream nodes, shown relative to each node's
pre-cut level, following the same edge cut. Open circles mark the first time
inventory falls below 50\% of its pre-cut level. Inventory depletion
propagates downstream with increasing graph distance from the cut, and
recovery begins after restoration but does not return the destination to its
pre-cut level within the 130-day window.}
  \label{fig:ripple-cascade:stock}
\end{subfigure}
\caption{\textbf{Ripple and cascading disruption across multiple echelons in
\sysname{}.} The Nashville$\rightarrow$Atlanta edge is removed for 60 days
in the \texttt{ripple\_scenario} rollout. (a) Flow along downstream
corridor edges, normalized by their pre-cut baselines; the progressive
decline illustrates the propagation of the disruption through the network.
(b) On-hand inventory relative to the pre-cut level at successive downstream
nodes. Recovery begins after the edge is restored, with the timing shaped by
network distance and intermediate inventory.}
\label{fig:ripple-cascade}
\end{figure}

\begin{figure}[t]
\centering
\begin{subfigure}[t]{0.49\linewidth}
  \centering
  \includegraphics[width=\linewidth]{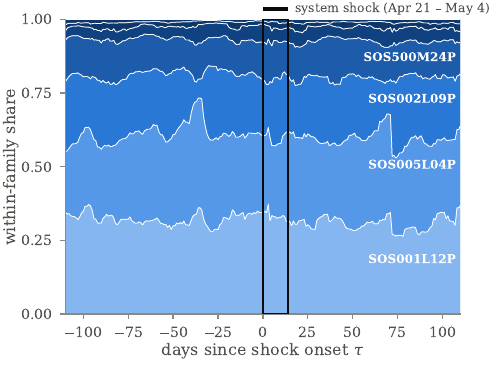}
  \subcaption{SupplyGraph: seven close-substitute SKUs from the SOS family, showing
within-family Sales-Order composition (7-day rolling) around the Eid-ul-Fitr
system shutdown (2023-04-21 to 2023-05-04, boxed). The composition is measured as each SKU's share of total
sales orders within the family, with time aligned to the onset of the
shutdown. The figure shows how order share is redistributed among related
SKUs following the disruption.}
  \label{fig:bottleneck-migration:real}
\end{subfigure}\hfill
\begin{subfigure}[t]{0.49\linewidth}
  \centering
  \includegraphics[width=\linewidth]{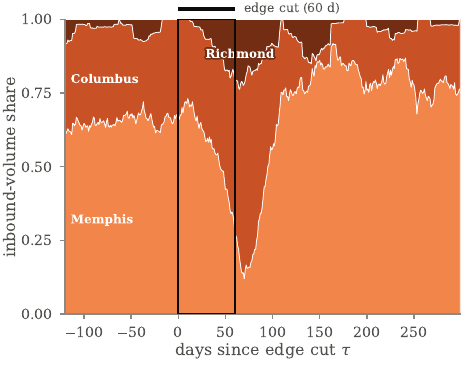}
  \subcaption{\sysname{}: inbound-volume composition at Baltimore under a
60-day removal of the Atlanta$\rightarrow$Memphis edge. Each area shows
the share of Baltimore's inbound volume supplied through Memphis, Columbus,
or Richmond (30-day rolling), aligned to the edge-cut onset. Removing the
Memphis route causes its share to fall while Columbus and Richmond increase
their shares to absorb the displaced flow, with the response lagging the
edge cut and recovery after restoration.}
  \label{fig:bottleneck-migration:iso}
\end{subfigure}

\vspace{0.7em}
\begin{subfigure}[t]{\linewidth}
  \centering
  \includegraphics[width=\linewidth]{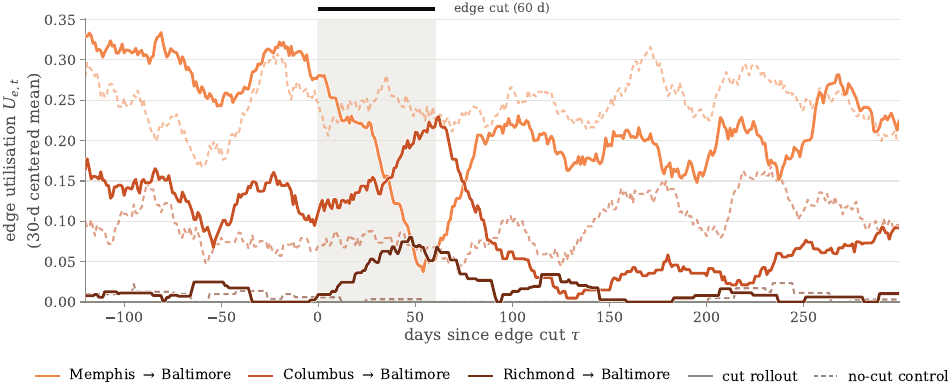}
  \subcaption{The same three Baltimore-inbound edges as in (b), showing edge
utilization $U_{e,t}$ (30-day centered mean) under the edge cut (solid)
and a matched no-cut control (dashed). Following the
Atlanta$\rightarrow$Memphis cut, utilization of Memphis$\rightarrow$Baltimore
falls while the Columbus$\rightarrow$Baltimore and
Richmond$\rightarrow$Baltimore alternatives increase, whereas the control
remains comparatively stable.}
  \label{fig:bottleneck-migration:util}
\end{subfigure}
\caption{\textbf{Flow redistribution following a disruption in real and simulated
supply-chain data.}
(a) SupplyGraph: within-family sales-order composition
for seven related SKUs around a system disruption. (b) \sysname{}:
inbound-volume composition at Baltimore following a 60-day
Atlanta$\rightarrow$Memphis edge cut, showing redistribution toward
alternative routes. (c) Utilization of the three Baltimore-inbound edges
under the cut and a matched no-cut control. The two systems exhibit
temporary redistribution among alternatives, with SupplyGraph capturing
SKU-level substitution and \sysname{} capturing route-level migration.}
\label{fig:bottleneck-migration}
\end{figure}

\FloatBarrier

\FloatBarrier

\section{Combining the digital twin \sysname{} with TSF foundation models}
\label{sec:foundation}

Combining a digital twin with TSF foundation models enables two complementary
capabilities: using \sysname{} as a regenerable benchmark for time-series
forecasting and as a testbed for evaluating foundation models as fast
surrogates for digital-twin forward simulation and uncertainty
quantification. We first assess forecasting performance across the released
target variables and scenario configurations
(\S\ref{sec:foundation-protocol}), and then evaluate whether foundation
models can reproduce the predictive variability induced by changes in the
underlying system configuration (\S\ref{sec:uq}). Experiments probing these
capabilities are in \S\ref{sec:results}.

\subsection{Experimental setup: TSF benchmarking 
}
\label{sec:foundation-protocol}

\sysname introduces the logistics domain to TSF
benchmarking. Different scenario configurations produce
diverse sequential patterns---oscillatory reorder cycles,
persistent drift, global shocks, and localized bursts
(Figure~\ref{fig:family-compact})---that are largely
absent from existing benchmark collections. The setup below is designed to examine two aspects separately: how current
forecasters perform across these generated patterns, and how forecasting
difficulty varies between the demand process and network-dependent targets
arising from the multi-echelon structure.

\paragraph{Forecast targets.}
We evaluate four targets, each a different functional of the network state and
each recoverable exactly from every released rollout by the conservation laws
of Proposition~\ref{prop:conservation}:
\begin{itemize}[leftmargin=1.4em,itemsep=2pt,topsep=2pt]
  \item \emph{Demand} $y_{i,t}$, the per-item demand at the destination. It is
    generated by Eq.~\eqref{eq:intensity} \emph{before} the network acts on it,
    so it carries no information about the network state and a model that
    forecasts it well has learned the intensity, not the logistics. We
    include it as the point of comparison with existing benchmarks.
  \item \emph{Backlog} $b_{i,t}$, accumulated unmet demand at the destination
    (Eq.~\eqref{eq:backlog-update}): a cumulative, integrating quantity.
  \item \emph{Fill rate}, served over demanded per item: a bounded ratio whose
    denominator can be small.
  \item \emph{Edge utilization} $U_{e,t}$, per-step shipped volume over
    capacity $K_e V_e$: bounded in $[0,1]$ and capacity-clipped.
\end{itemize}
The last three are network-dependent targets generated through the multi-echelon structure, making that structure an integral part of the evaluation.

\paragraph{Models.}
Three pretrained foundation models are run \emph{zero-shot}, chosen to cover
the main architectural families:
\textbf{Chronos-T5} \citep{ansari2024chronos} (tokenized values,
encoder-decoder T5 backbone, real and KernelSynth pretraining);
\textbf{Moirai-1.1-R} \citep{woo2024moirai} (masked encoder, pretrained on
LOTSA); and
\textbf{TimesFM-2.0} \citep{das2024timesfm} (decoder-only, trained on Google
Trends and Wikipedia pageviews).
Three further forecasters are fit \emph{in domain}, per rollout:
\textbf{ARIMA}, \textbf{ETS}, and \textbf{PatchTST}
\citep{nie2023patchtst}. Including in-domain baselines is what makes the
difficulty claim testable: without them, a high zero-shot error is equally
consistent with a hard domain and with an unfamiliar one.

\paragraph{Rolling-origin protocol.}
All models see $L{=}512$ past steps and predict at horizons
$h \in \{1, 7, 14, 30\}$, which fall within the short-to-medium range of
GIFT-Eval's 97 task configurations ($h \in [6, 900]$)
\citep{aksu2024gifteval}. Each rollout is split chronologically $70/15/15$;
the forecast origin slides across the test split in non-overlapping steps of
$30$, giving $|W|{=}36$ windows per rollout at $T{=}7{,}300$. Context windows
are allowed to span the train/test boundary, since no model is refit online.

\paragraph{Error metric and seasonal normalizers.}
We report Mean Absolute Scaled Error (MASE) following the GIFT-Eval protocol
\citep{aksu2024gifteval}. At each window, the model's MAE over the first $h$
steps is normalized by the seasonal-naive error on that window's own $L$-step
context, $\frac{1}{L-m}\sum_{t}|y_{t+m}-y_{t}|$
(App.~\ref{app:mase-formal}). Because the released data is daily and the
choice of $m$ changes what counts as ``trivially predictable,'' we report
\emph{three} normalizers rather than one: $m{=}1$ (the GIFT-Eval default),
$m{=}7$ (weekly, matching the daily resolution and the weekly harmonic of
Eq.~\eqref{eq:weekly}), and $m{=}30$ (monthly). 

\paragraph{Evaluation cohort and uncertainty.}
Every reported interval is a $95\%$ Student-$t$ confidence interval over the
$n{=}29$-rollout cohort of \S\ref{sec:datasets-io}: one baseline, $24$ static
one-knob variants, two compound-stress scenarios, and the two dynamic edge-cut
scenarios of \S\ref{sec:datasets-validation:disruption}. These intervals
therefore quantify variability across operating regimes. For backlog and fill
rate, the effective sample is $n{=}24$, because five under-loaded rollouts
have identically zero backlog and unit fill rate, making MASE undefined.
These rollouts are excluded for this reason.

\subsection{Experimental setup: zero-shot forward UQ}
\label{sec:uq}

Standard fixed datasets limit UQ to perturbations of the observed data,
such as additive noise, masking, or jitter. Because \sysname{} exposes the
full data-generating process, we can additionally quantify the effect of
uncertainty in the underlying system configuration on foundation-model
predictions.
\emph{Parameter uncertainty} varies the numerical values
of scenario knobs (App.~\ref{app:param-values}) while keeping the model structure
fixed; \emph{model uncertainty} substitutes structural
choices---such as the Poisson demand family
(Eq.~\ref{eq:intensity}, \S\ref{app:demand}),
rectangular pulse shapes for shocks and bursts
(Eqs.~\ref{eq:macro-shock}--\ref{eq:bursts}), or the
$(s, S)$ replenishment policy
(\S\ref{app:replenish})---while keeping knob values
fixed. In both cases, the simulator generates a new
rollout for each configuration, and the pretrained
foundation model is applied zero-shot, allowing us to
attribute prediction variability directly to the scenario
configurations and produce confidence bands from multiple
configurations. The experiments below instantiate parameter uncertainty; model uncertainty
is left to future work (\S\ref{sec:conclusion}).

\begin{tcolorbox}[colback=gray!5,colframe=gray!50,boxrule=0.4pt,
  left=4pt,right=4pt,top=4pt,bottom=4pt,
  title={\textbf{Forward-UQ protocol (parameter uncertainty)}}]
\begin{enumerate}[leftmargin=1.5em,itemsep=1pt,topsep=2pt]
  \item Take the scenario design of \S\ref{sec:datasets-io} as the
    perturbation structure: six one-at-a-time sweeps each move a single knob
    away from a shared baseline through four settings
    (App.~\ref{app:scenarios}).
  \item Use $K{=}20$ of the $29$ cohort rollouts as the forward-UQ
    ensemble, keeping every knob family of the design represented. The ensemble
    reads the same directories as the forecasting evaluation and adds none of
    its own (App.~\ref{app:schema-folders}).
  \item For each ensemble member $k$, read all
    four forecast targets $c$ of \S\ref{sec:foundation-protocol} off the
    corresponding rollout, giving trajectories $y^{(k,c)}_{t}$. One rollout
    therefore serves every target; no re-simulation is needed per target.
  \item For each foundation model $m \in \{\text{Chronos}, \text{Moirai},
    \text{TimesFM}\}$, apply it zero-shot with context length $L{=}512$ to each
    $y^{(k,c)}$, producing forecasts $\hat y^{(k,c),m}_{t}$.
  \item Score each $(k, w, c)$ cell against the model's own predictive
    distribution at that cell, and aggregate over the ensemble into the two
    metrics below.
\end{enumerate}
\end{tcolorbox}

The ensemble is a designed set of configurations and not a sample from a
probability measure on knob space. The coverage reported below is therefore
coverage across the scenario families the design spans, and it does not carry
the frequentist reading that sampling from an assumed prior would give it.

\paragraph{Calibration metrics.}
We report two scalar metrics, both computed per $(k, w, c)$ cell
(configuration, window, channel).
\emph{Empirical coverage} is the fraction of cells whose realized value
falls inside the model's own $[P_{10}, P_{90}]$ predictive interval; a
calibrated model reports $80\%$, below indicates overconfidence and above
indicates overcoverage. \emph{Weighted CRPS} is
$\mathrm{wCRPS} = \sum_i \mathrm{CRPS}_i / \sum_i |y_i|$, with
$\mathrm{CRPS}_i$ the sample-based continuous ranked probability score of
the predictive ensemble at cell $i$; lower is better. The two answer
different questions---coverage is a property of one nominal level, wCRPS
integrates over all of them---so a model can be good on one and poor on
the other, and \S\ref{sec:results:uq} reports a case where exactly that
happens. Trajectory-level envelopes are straightforward to produce from the
same ensemble; we report the metrics rather than the pictures because the
metrics are what establish whether an interval is correctly sized.

\paragraph{Foundation models as fast surrogates for forward UQ.}
This evaluation requires an accessible data-generating process. Because
\sysname{} exposes the generating process, uncertainty over scenario
configurations can be propagated through both the simulator and a
forecaster. The resulting predictive spread can then be compared with the
variation exhibited by the system across the same set of configurations.
The protocol is applicable to any forecaster and evaluates whether its
zero-shot predictive distribution reflects the variability induced by the
underlying system configurations.

\subsection{Results}
\label{sec:results}

\subsubsection{TSF benchmarking}
\label{sec:results:targets}

\paragraph{Seasonal structure of the four targets.}
The seasonal-naive denominator $\mathrm{SN}(m)$, reported in the last row of
each block of Tables~\ref{tab:combined-ci-demand},
\ref{tab:combined-ci-backlog}, \ref{tab:combined-ci-fill_rate}, and
\ref{tab:combined-ci-edge_util}, measures the typical error of repeating the
target at lag $m$. Its dependence on $m$ differs substantially across the
four targets. For \emph{demand}, $\mathrm{SN}(m)$ changes little from
$m=1$ to $m=7$ (16.08 to 16.87) but increases substantially at $m=30$
(27.23). \emph{Edge utilization} shows little dependence on the lag, with
$\mathrm{SN}(m)$ varying by less than $5\%$ across all three lags.
\emph{Backlog} exhibits the strongest lag dependence, with
$\mathrm{SN}(m)$ increasing by factors of $4.5$ and $9.7$ from $m=1$ to
$m=7$ and $m=30$, respectively, while \emph{fill rate} shows an
intermediate increase. These differences indicate distinct temporal
structures across the four targets and motivate reporting MASE at all three
seasonal lags.

\begin{table}[t]
\centering\footnotesize
\caption{Univariate MASE on \textbf{demand}, mean $\pm$ 95\% Student-$t$ CI over the $n{=}29$-rollout cohort. \emph{In-domain trained} baselines (ARIMA, ETS, PatchTST) are fit per-rollout; foundation models (Chronos, Moirai, TimesFM) run \emph{zero-shot}. Agg. is the geometric mean across horizons.}
\label{tab:combined-ci-demand}
\setlength{\tabcolsep}{4pt}\renewcommand{\arraystretch}{1.1}
\begin{tabular}{@{}l ccccc@{}}
\toprule
Model & $h{=}1$ & $h{=}7$ & $h{=}14$ & $h{=}30$ & Agg. \\
\midrule
\multicolumn{6}{@{}l}{\textbf{MASE} ($m{=}1$)} \\
ARIMA & 0.756\,$\pm$\,0.015 & 0.831\,$\pm$\,0.048 & 0.936\,$\pm$\,0.091 & 1.186\,$\pm$\,0.185 & 0.914\,$\pm$\,0.075 \\
ETS & 0.746\,$\pm$\,0.013 & 0.828\,$\pm$\,0.050 & 0.952\,$\pm$\,0.111 & 1.280\,$\pm$\,0.271 & 0.932\,$\pm$\,0.092 \\
PatchTST & 0.756\,$\pm$\,0.018 & 0.809\,$\pm$\,0.041 & 0.884\,$\pm$\,0.077 & 1.048\,$\pm$\,0.164 & 0.867\,$\pm$\,0.068 \\
Chronos & 0.774\,$\pm$\,0.009 & 0.843\,$\pm$\,0.042 & 0.939\,$\pm$\,0.090 & 1.156\,$\pm$\,0.199 & 0.917\,$\pm$\,0.074 \\
Moirai & 0.792\,$\pm$\,0.016 & 0.867\,$\pm$\,0.044 & 0.962\,$\pm$\,0.086 & 1.166\,$\pm$\,0.176 & 0.937\,$\pm$\,0.072 \\
TimesFM & 0.746\,$\pm$\,0.014 & 0.817\,$\pm$\,0.048 & 0.915\,$\pm$\,0.097 & 1.130\,$\pm$\,0.205 & 0.891\,$\pm$\,0.080 \\
\emph{SN($m{=}1$) MAE} & 16.02 & 16.06 & 16.10 & 16.12 & 16.08 \\
\midrule
\multicolumn{6}{@{}l}{\textbf{MASE} ($m{=}7$)} \\
ARIMA & 0.733\,$\pm$\,0.014 & 0.798\,$\pm$\,0.007 & 0.888\,$\pm$\,0.031 & 1.108\,$\pm$\,0.084 & 0.871\,$\pm$\,0.030 \\
ETS & 0.725\,$\pm$\,0.015 & 0.794\,$\pm$\,0.008 & 0.897\,$\pm$\,0.042 & 1.169\,$\pm$\,0.131 & 0.881\,$\pm$\,0.041 \\
PatchTST & 0.732\,$\pm$\,0.014 & 0.778\,$\pm$\,0.006 & 0.841\,$\pm$\,0.025 & 0.980\,$\pm$\,0.077 & 0.828\,$\pm$\,0.028 \\
Chronos & 0.754\,$\pm$\,0.019 & 0.811\,$\pm$\,0.006 & 0.890\,$\pm$\,0.030 & 1.070\,$\pm$\,0.093 & 0.873\,$\pm$\,0.033 \\
Moirai & 0.769\,$\pm$\,0.015 & 0.835\,$\pm$\,0.006 & 0.916\,$\pm$\,0.028 & 1.091\,$\pm$\,0.080 & 0.895\,$\pm$\,0.029 \\
TimesFM & 0.725\,$\pm$\,0.015 & 0.785\,$\pm$\,0.008 & 0.866\,$\pm$\,0.036 & 1.046\,$\pm$\,0.098 & 0.847\,$\pm$\,0.035 \\
\emph{SN($m{=}7$) MAE} & 16.71 & 16.83 & 16.92 & 17.02 & 16.87 \\
\midrule
\multicolumn{6}{@{}l}{\textbf{MASE} ($m{=}30$)} \\
ARIMA & 0.474\,$\pm$\,0.021 & 0.511\,$\pm$\,0.013 & 0.563\,$\pm$\,0.004 & 0.691\,$\pm$\,0.018 & 0.554\,$\pm$\,0.014 \\
ETS & 0.469\,$\pm$\,0.021 & 0.508\,$\pm$\,0.012 & 0.566\,$\pm$\,0.004 & 0.720\,$\pm$\,0.038 & 0.558\,$\pm$\,0.018 \\
PatchTST & 0.473\,$\pm$\,0.020 & 0.499\,$\pm$\,0.014 & 0.535\,$\pm$\,0.005 & 0.614\,$\pm$\,0.019 & 0.528\,$\pm$\,0.015 \\
Chronos & 0.488\,$\pm$\,0.023 & 0.520\,$\pm$\,0.015 & 0.564\,$\pm$\,0.005 & 0.665\,$\pm$\,0.023 & 0.556\,$\pm$\,0.017 \\
Moirai & 0.497\,$\pm$\,0.021 & 0.534\,$\pm$\,0.015 & 0.582\,$\pm$\,0.005 & 0.683\,$\pm$\,0.017 & 0.570\,$\pm$\,0.015 \\
TimesFM & 0.469\,$\pm$\,0.021 & 0.503\,$\pm$\,0.012 & 0.549\,$\pm$\,0.002 & 0.650\,$\pm$\,0.026 & 0.539\,$\pm$\,0.015 \\
\emph{SN($m{=}30$) MAE} & 26.68 & 27.07 & 27.39 & 27.78 & 27.23 \\
\bottomrule
\end{tabular}
\end{table}

\paragraph{Dependence of forecast difficulty on the target variable.}
The four targets exhibit substantially different forecasting difficulty, as
measured by aggregate MASE across models and horizons
(Table~\ref{tab:networkstate-summary}). Edge utilization has the lowest
aggregate MASE ($0.72$--$0.89$; Table~\ref{tab:combined-ci-edge_util}),
followed by demand ($0.83$--$0.90$; Table~\ref{tab:combined-ci-demand}),
while backlog and fill rate have substantially higher errors
($1.5$--$7.7$; Tables~\ref{tab:combined-ci-backlog} and
\ref{tab:combined-ci-fill_rate}). This difference is consistent with the
temporal structure of the targets: backlog accumulates mismatches between
arrivals and demand, leading to persistent forecast errors, while fill rate
is a ratio whose variability can increase when demand is small. The harder
targets also exhibit greater variation across scenario rollouts. For
example, the aggregate confidence intervals are wide for both backlog and
fill rate, with Moirai achieving $2.228 \pm 1.653$ on backlog and PatchTST
$7.714 \pm 3.027$ on fill rate (Tables~\ref{tab:combined-ci-backlog} and
\ref{tab:combined-ci-fill_rate}).

\begin{table}[t]
\centering\small
\caption{Aggregate MASE (geometric mean over $h \in \{1,7,14,30\}$) at
$m{=}7$ on the four \sysname{} forecast targets, mean $\pm$ 95\% Student-$t$
CI over the evaluation cohort ($n{=}29$; $n{=}24$ for backlog and fill rate,
see \S\ref{sec:foundation-protocol}). Bold marks the column minimum. ARIMA
diverges on backlog and fill rate. Full per-horizon, per-normalizer tables are
Tables~\ref{tab:combined-ci-demand} and
\ref{tab:combined-ci-backlog}--\ref{tab:combined-ci-edge_util}.}
\label{tab:networkstate-summary}
\setlength{\tabcolsep}{6pt}\renewcommand{\arraystretch}{1.1}
\begin{tabular}{@{}l cccc@{}}
\toprule
Model & Demand & Backlog & Fill rate & Edge utilization \\
\midrule
ARIMA     & $0.871 \pm 0.030$ & $7.7\mathrm{e}{+}06$ & $4.8\mathrm{e}{+}06$ & $0.798 \pm 0.033$ \\
ETS       & $0.881 \pm 0.041$ & $4.063 \pm 2.156$ & $6.931 \pm 3.811$ & $0.798 \pm 0.021$ \\
PatchTST  & $\mathbf{0.828 \pm 0.028}$ & $4.839 \pm 1.862$ & $7.714 \pm 3.027$ & $0.886 \pm 0.038$ \\
Chronos   & $0.873 \pm 0.033$ & $2.291 \pm 1.639$ & $2.643 \pm 0.962$ & $0.735 \pm 0.016$ \\
Moirai    & $0.895 \pm 0.029$ & $\mathbf{2.228 \pm 1.653}$ & $\mathbf{1.473 \pm 0.721}$ & $\mathbf{0.721 \pm 0.014}$ \\
TimesFM   & $0.847 \pm 0.035$ & $5.545 \pm 2.951$ & $5.021 \pm 3.281$ & $0.817 \pm 0.109$ \\
\bottomrule
\end{tabular}
\end{table}

\paragraph{Zero-shot foundation models against in-domain baselines.}
Zero-shot foundation models and in-domain-trained baselines show different
performance across the four targets
(Table~\ref{tab:networkstate-summary}). Among the three network-dependent
targets, the zero-shot Moirai model has the lowest aggregate MASE:
$0.721$ for edge utilization, $2.228$ for backlog, and $1.473$ for fill rate.
The corresponding best aggregate MASE among the in-domain-trained baselines
is $0.798$, $4.063$, and $6.931$, respectively. For demand, which is
independent of the network, PatchTST has the lowest aggregate MASE ($0.828$),
closely followed by zero-shot TimesFM ($0.847$). Thus, the relative
performance of zero-shot foundation models and in-domain-trained baselines
depends on the target.

\paragraph{Difficulty relative to standard TSF benchmarks.}
Table~\ref{tab:crossbench-ci} compares the four \sysname{} targets with
ETTh1, Electricity, and Weather under a common MASE evaluation, using the
$m=1$ seasonal-naive baseline for \sysname{} and the corresponding seasonal
periods for each benchmark. The \sysname{} targets exhibit substantially
higher MASE values than the standard benchmarks, particularly for backlog
and fill rate, indicating that the foundation models capture substantially
less of the short-lag predictive structure in \sysname{} than in the
existing benchmarks. Demand also shows consistently higher MASE than the
three benchmarks, while edge utilization is closer to the difficulty of
ETTh1. Thus, the released \sysname{} trajectories provide forecasting
regimes in which the predictive structure captured by current foundation
models is substantially weaker than on commonly used TSF benchmarks.

This comparison highlights the role of \sysname{} as a complementary
benchmark to existing TSF benchmarks such as GIFT-Eval. Standard benchmarks
provide fixed real-world datasets across diverse domains, whereas
\sysname{} provides a regenerable logistics-domain dataset with forecasting
regimes that are substantially less well covered by current foundation
models. This makes \sysname{} a useful additional benchmark for evaluating
forecasting models beyond the regimes represented in existing datasets.

\begin{table}[h]
\centering\footnotesize
\setlength{\tabcolsep}{4pt}\renewcommand{\arraystretch}{1.1}
\caption{Zero-shot foundation-model MASE across \textsc{Isomorph} targets and standard TSF benchmarks. \textsc{Isomorph} rows: mean $\pm$ 95\% Student-$t$ CI over $n{=}29$ rollouts. Benchmark rows: single-run point estimates. Agg. is the geometric mean across horizons. Full per-horizon,
per-normalizer tables for each target are Table~\ref{tab:combined-ci-demand}
(demand), Table~\ref{tab:combined-ci-backlog} (backlog),
Table~\ref{tab:combined-ci-fill_rate} (fill rate), and
Table~\ref{tab:combined-ci-edge_util} (edge utilization).}
\label{tab:crossbench-ci}
\begin{tabular}{@{}l ccccc@{}}
\toprule
Model & $h{=}1$ & $h{=}7$ & $h{=}14$ & $h{=}30$ & Agg. \\
\midrule
\multicolumn{6}{@{}l}{\emph{\textsc{Isomorph} demand (50 items, daily) \hfill\textit{ours}}} \\
\quad Chronos & 0.774\,$\pm$\,0.009 & 0.843\,$\pm$\,0.042 & 0.939\,$\pm$\,0.090 & 1.156\,$\pm$\,0.199 & 0.910\,$\pm$\,0.068 \\
\quad Moirai & 0.792\,$\pm$\,0.016 & 0.867\,$\pm$\,0.044 & 0.962\,$\pm$\,0.086 & 1.166\,$\pm$\,0.176 & 0.932\,$\pm$\,0.067 \\
\quad TimesFM & 0.746\,$\pm$\,0.014 & 0.817\,$\pm$\,0.048 & 0.915\,$\pm$\,0.097 & 1.130\,$\pm$\,0.205 & 0.885\,$\pm$\,0.073 \\
\midrule
\multicolumn{6}{@{}l}{\emph{\textsc{Isomorph} edge utilisation (16 edges, daily) \hfill\textit{ours}}} \\
\quad Chronos & 0.736\,$\pm$\,0.026 & 0.757\,$\pm$\,0.016 & 0.758\,$\pm$\,0.018 & 0.773\,$\pm$\,0.019 & 0.755\,$\pm$\,0.018 \\
\quad Moirai & 0.721\,$\pm$\,0.024 & 0.747\,$\pm$\,0.015 & 0.747\,$\pm$\,0.016 & 0.758\,$\pm$\,0.018 & 0.743\,$\pm$\,0.016 \\
\quad TimesFM & 0.812\,$\pm$\,0.072 & 0.831\,$\pm$\,0.108 & 0.857\,$\pm$\,0.151 & 0.844\,$\pm$\,0.114 & 0.831\,$\pm$\,0.097 \\
\midrule
\multicolumn{6}{@{}l}{\emph{\textsc{Isomorph} backlog (50 items, daily) \hfill\textit{ours}}} \\
\quad Chronos & 2.109\,$\pm$\,1.164 & 3.876\,$\pm$\,2.343 & 4.860\,$\pm$\,2.616 & 6.361\,$\pm$\,3.946 & 3.690\,$\pm$\,2.042 \\
\quad Moirai & 1.570\,$\pm$\,0.910 & 3.431\,$\pm$\,2.174 & 4.566\,$\pm$\,2.467 & 6.427\,$\pm$\,3.988 & 3.242\,$\pm$\,1.873 \\
\quad TimesFM & 5.157\,$\pm$\,2.579 & 6.236\,$\pm$\,3.229 & 8.038\,$\pm$\,4.165 & 11.914\,$\pm$\,6.788 & 7.152\,$\pm$\,3.740 \\
\midrule
\multicolumn{6}{@{}l}{\emph{\textsc{Isomorph} fill rate (50 items, daily) \hfill\textit{ours}}} \\
\quad Chronos & 2.902\,$\pm$\,0.951 & 3.837\,$\pm$\,1.204 & 4.592\,$\pm$\,1.457 & 5.578\,$\pm$\,2.166 & 3.979\,$\pm$\,1.266 \\
\quad Moirai & 2.308\,$\pm$\,0.993 & 2.939\,$\pm$\,1.249 & 3.391\,$\pm$\,1.284 & 4.185\,$\pm$\,1.853 & 2.976\,$\pm$\,1.251 \\
\quad TimesFM & 4.902\,$\pm$\,2.266 & 5.979\,$\pm$\,2.877 & 7.192\,$\pm$\,3.878 & 9.066\,$\pm$\,5.558 & 6.457\,$\pm$\,3.399 \\
\midrule
\multicolumn{6}{@{}l}{\emph{ETTh1 (7 channels, hourly)}} \\
\quad Chronos & 0.519 & 0.694 & 0.814 & 0.890 & 0.715 \\
\quad Moirai & 0.582 & 0.792 & 0.894 & 0.949 & 0.791 \\
\quad TimesFM & 0.572 & 0.752 & 0.848 & 0.916 & 0.761 \\
\midrule
\multicolumn{6}{@{}l}{\emph{Electricity (321 channels, hourly)}} \\
\quad Chronos & 0.508 & 0.643 & 0.704 & 0.767 & 0.648 \\
\quad Moirai & 0.754 & 0.862 & 0.931 & 0.990 & 0.880 \\
\quad TimesFM & 0.548 & 0.677 & 0.722 & 0.762 & 0.672 \\
\midrule
\multicolumn{6}{@{}l}{\emph{Weather (21 channels, 10-min)}} \\
\quad Chronos & 0.490 & 0.559 & 0.613 & 0.799 & 0.605 \\
\quad Moirai & 0.448 & 0.529 & 0.593 & 0.812 & 0.581 \\
\quad TimesFM & 0.498 & 0.469 & 0.446 & 0.527 & 0.484 \\
\bottomrule
\end{tabular}
\end{table}

\subsubsection{Forward UQ}
\label{sec:results:uq}

We evaluate the zero-shot predictive intervals of the three foundation
models across the four targets, comparing their empirical coverage with the
nominal $80\%$ level and their weighted CRPS across horizons
(Tables~\ref{tab:uq-coverage} and \ref{tab:uq-wcrps}). The results differ
substantially across targets:

\begin{itemize}
    \item \textbf{Demand.}
    Moirai and TimesFM remain close to the nominal $80\%$ coverage across
    horizons, with coverage ranging from $80.1$--$84.2\%$ for Moirai and
    $75.6$--$80.1\%$ for TimesFM (Table~\ref{tab:uq-coverage}).
    Their weighted CRPS values are also similar, ranging from
    $0.048$--$0.079$ for Moirai and $0.046$--$0.077$ for TimesFM
    (Table~\ref{tab:uq-wcrps}). Chronos is consistently undercovered,
    with coverage decreasing from $69.5\%$ to $63.0\%$ as the horizon
    increases.

    \item \textbf{Backlog.}
    Coverage varies substantially across the three models and horizons
    (Table~\ref{tab:uq-coverage}). Moirai remains overcovered at
    $93.0$--$95.6\%$, while Chronos is severely undercovered at
    $7.4$--$12.0\%$. TimesFM shows a strongly non-monotone pattern,
    with coverage ranging from $30.0\%$ to $91.1\%$. The corresponding
    weighted CRPS values remain small in absolute terms
    ($0.0020$--$0.0070$), but do not by themselves reflect the large
    differences in coverage (Table~\ref{tab:uq-wcrps}).

    \item \textbf{Fill rate.}
    The coverage patterns are similar to those for backlog
    (Table~\ref{tab:uq-coverage}). Moirai is consistently overcovered,
    with $93.0$--$96.8\%$ coverage, whereas Chronos is severely
    undercovered at only $0.5$--$3.0\%$. TimesFM is highly
    non-monotone, with coverage dropping to $13.4$--$16.5\%$ at
    $h=7$--$14$ before returning to $96.8\%$ at $h=30$. Weighted CRPS
    values range from $0.0079$ to $0.025$ across the three models
    (Table~\ref{tab:uq-wcrps}).

    \item \textbf{Edge utilization.}
    Moirai and TimesFM remain close to the nominal coverage, ranging
    from $76.7$--$79.2\%$ and $78.7$--$80.4\%$, respectively
    (Table~\ref{tab:uq-coverage}). Chronos is consistently
    undercovered at $34.3$--$36.4\%$. The weighted CRPS values are
    comparable between Moirai and TimesFM ($0.266$--$0.296$), and
    higher for Chronos ($0.295$--$0.312$) at most horizons
    (Table~\ref{tab:uq-wcrps}).
\end{itemize}

Across the four targets, zero-shot calibration is strongest on demand and
edge utilization, whereas backlog and fill rate exhibit substantially larger
and more horizon-dependent deviations from the nominal coverage
(Table~\ref{tab:uq-coverage}). These results also motivate fine-tuning
foundation models on the released trajectories, particularly for backlog and
fill rate, to investigate whether domain adaptation can improve uncertainty
quantification under the same forward-UQ protocol.

\begin{table}[h]
\centering\footnotesize
\caption{Empirical coverage of the $[P_{10},P_{90}]$ predictive interval (nominal
$80\%$), averaged over $K{=}20$ scenario rollouts, mean $\pm$ 95\% Student-$t$ CI across scenarios. A calibrated model reports $80\%$; below indicates overconfidence (intervals too narrow), above indicates overcoverage. The forward-UQ ensemble covers these three models.}
\label{tab:uq-coverage}
\setlength{\tabcolsep}{4pt}\renewcommand{\arraystretch}{1.1}
\begin{tabular}{@{}l cccc@{}}
\toprule
Model & $h{=}1$ & $h{=}7$ & $h{=}14$ & $h{=}30$ \\
\midrule
\multicolumn{5}{@{}l}{\textbf{demand}} \\
Chronos & 69.5\%\,$\pm$\,0.6\% & 68.5\%\,$\pm$\,0.8\% & 66.4\%\,$\pm$\,0.8\% & 63.0\%\,$\pm$\,1.1\% \\
Moirai & 80.8\%\,$\pm$\,0.6\% & 84.2\%\,$\pm$\,0.6\% & 82.6\%\,$\pm$\,0.4\% & 80.8\%\,$\pm$\,0.2\% \\
TimesFM & 78.4\%\,$\pm$\,0.5\% & 80.1\%\,$\pm$\,0.3\% & 78.3\%\,$\pm$\,0.2\% & 75.6\%\,$\pm$\,0.6\% \\
\midrule
\multicolumn{5}{@{}l}{\textbf{backlog}} \\
Chronos & 7.4\%\,$\pm$\,5.9\% & 10.0\%\,$\pm$\,8.0\% & 10.5\%\,$\pm$\,8.5\% & 12.0\%\,$\pm$\,9.7\% \\
Moirai & 95.0\%\,$\pm$\,3.5\% & 93.6\%\,$\pm$\,4.1\% & 93.2\%\,$\pm$\,4.6\% & 93.6\%\,$\pm$\,4.5\% \\
TimesFM & 91.1\%\,$\pm$\,4.2\% & 30.0\%\,$\pm$\,14.6\% & 31.0\%\,$\pm$\,14.4\% & 90.5\%\,$\pm$\,4.7\% \\
\midrule
\multicolumn{5}{@{}l}{\textbf{fill\_rate}} \\
Chronos & 0.5\%\,$\pm$\,0.5\% & 1.5\%\,$\pm$\,0.9\% & 2.1\%\,$\pm$\,1.4\% & 3.0\%\,$\pm$\,1.9\% \\
Moirai & 95.5\%\,$\pm$\,2.4\% & 93.0\%\,$\pm$\,3.9\% & 95.6\%\,$\pm$\,2.1\% & 96.3\%\,$\pm$\,2.0\% \\
TimesFM & 95.8\%\,$\pm$\,2.2\% & 13.4\%\,$\pm$\,6.9\% & 16.5\%\,$\pm$\,8.3\% & 96.8\%\,$\pm$\,1.9\% \\
\midrule
\multicolumn{5}{@{}l}{\textbf{edge\_util}} \\
Chronos & 36.2\%\,$\pm$\,1.5\% & 36.4\%\,$\pm$\,1.4\% & 36.0\%\,$\pm$\,1.4\% & 34.3\%\,$\pm$\,1.5\% \\
Moirai & 76.7\%\,$\pm$\,0.9\% & 78.8\%\,$\pm$\,1.0\% & 77.6\%\,$\pm$\,1.1\% & 79.2\%\,$\pm$\,1.0\% \\
TimesFM & 78.8\%\,$\pm$\,1.0\% & 80.4\%\,$\pm$\,1.1\% & 79.4\%\,$\pm$\,1.2\% & 78.7\%\,$\pm$\,0.8\% \\
\bottomrule
\end{tabular}
\end{table}

\begin{table}[h]
\centering\footnotesize
\caption{Weighted CRPS $\mathrm{wCRPS}=\sum_i \mathrm{CRPS}_i/\sum_i|y_i|$ of the forward-UQ ensemble ($K{=}20$ scenario rollouts), per target, model, and horizon; lower is better. Cells are aggregated over all (scenario, window, channel) triples at each horizon. The forward-UQ ensemble covers these three models.}
\label{tab:uq-wcrps}
\setlength{\tabcolsep}{4pt}\renewcommand{\arraystretch}{1.1}
\begin{tabular}{@{}l cccc@{}}
\toprule
Model & $h{=}1$ & $h{=}7$ & $h{=}14$ & $h{=}30$ \\
\midrule
\multicolumn{5}{@{}l}{\textbf{demand}} \\
Chronos & 0.047 & 0.060 & 0.079 & 0.118 \\
Moirai & 0.048 & 0.061 & 0.079 & 0.110 \\
TimesFM & 0.046 & 0.059 & 0.077 & 0.113 \\
\midrule
\multicolumn{5}{@{}l}{\textbf{backlog}} \\
Chronos & 0.0020 & 0.0032 & 0.0044 & 0.0066 \\
Moirai & 0.0008 & 0.0026 & 0.0047 & 0.0070 \\
TimesFM & 0.0009 & 0.0021 & 0.0032 & 0.0055 \\
\midrule
\multicolumn{5}{@{}l}{\textbf{fill\_rate}} \\
Chronos & 0.0093 & 0.019 & 0.022 & 0.025 \\
Moirai & 0.0083 & 0.017 & 0.020 & 0.023 \\
TimesFM & 0.0079 & 0.016 & 0.020 & 0.021 \\
\midrule
\multicolumn{5}{@{}l}{\textbf{edge\_util}} \\
Chronos & 0.299 & 0.295 & 0.298 & 0.312 \\
Moirai & 0.271 & 0.274 & 0.393 & 0.296 \\
TimesFM & 0.270 & 0.266 & 0.266 & 0.273 \\
\bottomrule
\end{tabular}
\end{table}


\section{Release and maintenance}
\label{sec:release}

The simulator code, evaluation pipeline, and data
loaders are released under MIT at
\url{https://github.com/tuhinsahai/ISOMORPH};
the datasets used in this work are available from the authors upon request. An interactive stress-test demo
on the 13-node U.S. network with three disruption
scenarios is available at
\url{https://huggingface.co/spaces/HyeminGu/ISOMORPH-demo}.
Bug
reports and feature requests use the repository's issue tracker;
the corresponding author will respond to issues for at least two
years after publication. 

\section{Conclusion and outlook}
\label{sec:conclusion}

We presented \sysname{}, an open-source digital twin for generating
multi-echelon, multi-item logistics trajectories on a user-configurable
directed network. Supply-chain dynamics are inherently nonlinear and
history-dependent, combining persistent demand variation, discrete
replenishment and shipping decisions, inventory accumulation, and
network-level flow constraints. We close these history-dependent dynamics
as a Markov chain by augmenting the state with per-node on-hand inventory,
outstanding orders, in-transit shipments, and a smoothed demand estimate.
This yields a high-dimensional but tractable state space with a
deterministic transition function driven by two random inputs. The
resulting transition kernel acts linearly on the state distribution and
supports simulation over user-configurable directed networks, while
pathwise conservation laws provide trajectory-level verification.

For the released benchmark, we instantiate a 13-node U.S. network with
16 edges, daily resolution over $T=7{,}300$ steps, and two catalogue
regimes with $C=50$ and $C=200$ items. Six scenario knobs vary demand and
network conditions around this common configuration, generating controlled
variation across the resulting trajectories. The release comprises a
33-rollout scenario library, including 29 forecasting rollouts and four
additional physical-realism demonstrations; 20 of the forecasting
rollouts form the forward-UQ ensemble.

We use these trajectories for two main evaluations. First, we benchmark
time-series forecasting with three zero-shot foundation models and three
in-domain baselines across item-level and network-level targets. The
results show that \sysname{} provides a regenerable logistics-domain
benchmark with target-dependent forecasting difficulty and controlled
variation in the underlying data-generating system. Second, we propagate
variation across scenario configurations through the forecasting pipeline
to evaluate forward uncertainty quantification for the foundation models.
This provides a direct test of whether their predictive distributions
reflect variability induced by changes in the underlying system
configuration.

Together, these results establish \sysname{} as a regenerable
logistics-domain benchmark that complements fixed real-world TSF
benchmarks, while providing a generative testbed for evaluating and
developing foundation models through system-level uncertainty
quantification.

\paragraph{Outlook 1--- Scope of the current release.}
The released datasets instantiate a single 13-node U.S. network, two catalogue
sizes, and one time horizon, with variation provided by the scenario library
described in App.~\ref{app:scenarios}. The simulator itself accepts
user-supplied directed graphs and modular transition rules, so these
configurations define the scope of the present benchmark rather than the
generative framework as a whole. The current release therefore provides a
concrete starting point for exploring larger networks, alternative
topologies, and additional operational mechanisms.

\paragraph{Outlook 2 --- Simulator extensions.}
The modular transition map provides several direct extensions, catalogued in
App.~\ref{app:extensions}. Online closed-loop control and multi-destination
allocation preserve the current state-space closure and would enable
forecasting tasks with adaptive control and destination-specific allocation.
Other extensions include richer shock processes, time-dependent congestion,
vehicle scheduling, and cross-item substitution, with the latter requiring
corresponding changes to the conservation structure. These extensions provide
systematic directions for increasing the operational complexity and diversity
of the generated trajectories.

\paragraph{Outlook 3 --- Forecasting and uncertainty quantification.}
The present evaluation focuses on zero-shot forecasting and parameter
uncertainty. Extending forward UQ to structural alternatives---including the
demand family, shock and burst shapes, and replenishment policy---would
provide a complementary form of model uncertainty. The forecasting results
also motivate recalibration methods for integrating targets such as backlog
and fill rate, where predictive intervals are poorly calibrated, and
multivariate forecasting across the coupled network-state targets.

\paragraph{Outlook 4 --- Foundation-model adaptation and benchmark integration.}
The regenerable trajectories also provide a basis for training
logistics-domain forecasting models. Fine-tuning foundation models on
\sysname{}-generated trajectories can test whether exposure to the simulated
domain improves forecasting and uncertainty quantification on held-out
scenarios. Incorporating \sysname{} into multi-domain suites such as
GIFT-Eval would likewise extend their coverage with a regenerable
logistics-domain benchmark.

\bibliographystyle{plainnat}
\bibliography{references}

@misc{naik2025bullode,
      title={BULL-ODE: Bullwhip Learning with Neural ODEs and Universal Differential Equations under Stochastic Demand}, 
      author={Nachiket N. Naik and Prathamesh Dinesh Joshi and Raj Abhijit Dandekar and Rajat Dandekar and Sreedath Panat},
      year={2025},
      eprint={2509.18105},
      archivePrefix={arXiv},
      primaryClass={cs.LG},
      url={https://arxiv.org/abs/2509.18105}, 
}

@inproceedings{emami2023buildingsbench,
title={BuildingsBench: A Large-Scale Dataset of 900K Buildings and Benchmark for Short-Term Load Forecasting},
author={Patrick Emami and Abhijeet Sahu and Peter Graf},
booktitle={Thirty-seventh Conference on Neural Information Processing Systems Datasets and Benchmarks Track},
year={2023},
url={https://openreview.net/forum?id=c5rqd6PZn6}
}

@misc{wasi2024supplygraph,
      title={SupplyGraph: A Benchmark Dataset for Supply Chain Planning using Graph Neural Networks}, 
      author={Azmine Toushik Wasi and MD Shafikul Islam and Adipto Raihan Akib},
      year={2024},
      eprint={2401.15299},
      archivePrefix={arXiv},
      primaryClass={cs.LG},
      url={https://arxiv.org/abs/2401.15299},
}

@misc{hubbs2020orgym,
      title={OR-Gym: A Reinforcement Learning Library for Operations Research Problems}, 
      author={Christian D. Hubbs and Hector D. Perez and Owais Sarwar and Nikolaos V. Sahinidis and Ignacio E. Grossmann and John M. Wassick},
      year={2020},
      eprint={2008.06319},
      archivePrefix={arXiv},
      primaryClass={cs.AI},
      url={https://arxiv.org/abs/2008.06319}, 
}

@article{ansari2024chronos,
title={Chronos: Learning the Language of Time Series},
author={Abdul Fatir Ansari and Lorenzo Stella and Ali Caner Turkmen and Xiyuan Zhang and Pedro Mercado and Huibin Shen and Oleksandr Shchur and Syama Sundar Rangapuram and Sebastian Pineda Arango and Shubham Kapoor and Jasper Zschiegner and Danielle C. Maddix and Hao Wang and Michael W. Mahoney and Kari Torkkola and Andrew Gordon Wilson and Michael Bohlke-Schneider and Bernie Wang},
journal={Transactions on Machine Learning Research},
year={2024},
url={https://openreview.net/forum?id=gerNCVqqtR},
}

@inproceedings{woo2024moirai,
  title={Unified Training of Universal Time Series Forecasting Transformers},
  author={Woo, Gerald and Liu, Chenghao and Kumar, Akshat and Xiong, Caiming and Savarese, Silvio and Sahoo, Doyen},
  booktitle={Forty-first International Conference on Machine Learning},
  year={2024}
}

@misc{das2024timesfm,
      title={A decoder-only foundation model for time-series forecasting}, 
      author={Abhimanyu Das and Weihao Kong and Rajat Sen and Yichen Zhou},
      year={2024},
      eprint={2310.10688},
      archivePrefix={arXiv},
      primaryClass={cs.CL},
      url={https://arxiv.org/abs/2310.10688}, 
}

@inproceedings{zhou2021informer,
  author    = {Haoyi Zhou and
               Shanghang Zhang and
               Jieqi Peng and
               Shuai Zhang and
               Jianxin Li and
               Hui Xiong and
               Wancai Zhang},
  title     = {Informer: Beyond Efficient Transformer for Long Sequence Time-Series Forecasting},
  booktitle = {The Thirty-Fifth {AAAI} Conference on Artificial Intelligence, {AAAI} 2021, Virtual Conference},
  volume    = {35},
  number    = {12},
  pages     = {11106--11115},
  publisher = {{AAAI} Press},
  year      = {2021},
}

@inproceedings{godahewa2021monash,
title={Monash Time Series Forecasting Archive},
author={Rakshitha Wathsadini Godahewa and Christoph Bergmeir and Geoffrey I. Webb and Rob Hyndman and Pablo Montero-Manso},
booktitle={Thirty-fifth Conference on Neural Information Processing Systems Datasets and Benchmarks Track (Round 2)},
year={2021},
url={https://openreview.net/forum?id=wEc1mgAjU-}
}

@article{m4comp,
title = {The {M4} Competition: 100,000 time series and 61 forecasting methods},
journal = {International Journal of Forecasting},
volume = {36},
number = {1},
pages = {54--74},
year = {2020},
url = {https://www.sciencedirect.com/science/article/pii/S0169207019301128},
author = {Spyros Makridakis and Evangelos Spiliotis and Vassilios Assimakopoulos},
}

@article{m5comp,
title = {M5 accuracy competition: Results, findings, and conclusions},
journal = {International Journal of Forecasting},
volume = {38},
number = {4},
pages = {1346--1364},
year = {2022},
url = {https://www.sciencedirect.com/science/article/pii/S0169207021001874},
author = {Spyros Makridakis and Evangelos Spiliotis and Vassilios Assimakopoulos},
}

@misc{aksu2024gifteval,
  title        = {{GIFT}-Eval: A Benchmark for General Time Series Forecasting Model Evaluation},
  author       = {Taha Aksu and Gerald Woo and Juncheng Liu and Xu Liu and Chenghao Liu and Silvio Savarese and Caiming Xiong and Doyen Sahoo},
  year         = {2024},
  eprint       = {2410.10393},
  archivePrefix= {arXiv},
  primaryClass = {cs.LG},
  url={https://arxiv.org/abs/2410.10393}, 
}

@inproceedings{nie2023patchtst,
  author    = {Nie, Yuqi and Nguyen, Nam H. and Sinthong, Phanwadee
               and Kalagnanam, Jayant},
  title     = {A Time Series is Worth 64 Words: Long-term Forecasting
               with Transformers},
  booktitle = {International Conference on Learning Representations (ICLR)},
  year      = {2023},
  eprint    = {2211.14730},
}

@inproceedings{dooley2023forecastpfn,
author = {Dooley, Samuel and Khurana, Gurnoor Singh and Mohapatra, Chirag and Naidu, Siddartha and White, Colin},
title = {ForecastPFN: synthetically-trained zero-shot forecasting},
year = {2023},
booktitle={Advances in Neural Information Processing Systems},
}

@misc{arief2026deepbullwhip,
      title={Deepbullwhip: An Open-Source Simulation and Benchmarking for Multi-Echelon Bullwhip Analyses}, 
      author={Mansur M. Arief},
      year={2026},
      eprint={2604.13478},
      archivePrefix={arXiv},
      primaryClass={math.OC},
      url={https://arxiv.org/abs/2604.13478}, 
}

@article{watsonparris2022climatebench,
author = {Watson-Parris, D. and Rao, Y. and Olivié, D. and Seland, Ø. and Nowack, P. and Camps-Valls, G. and Stier, P. and Bouabid, S. and Dewey, M. and Fons, E. and Gonzalez, J. and Harder, P. and Jeggle, K. and Lenhardt, J. and Manshausen, P. and Novitasari, M. and Ricard, L. and Roesch, C.},
title = {ClimateBench v1.0: A Benchmark for Data-Driven Climate Projections},
journal = {Journal of Advances in Modeling Earth Systems},
volume = {14},
number = {10},
pages = {e2021MS002954},
url = {https://agupubs.onlinelibrary.wiley.com/doi/abs/10.1029/2021MS002954},
year = {2022},
}

@article{cachon2007search,
author = {Cachon, Gérard and Randall, Taylor and Schmidt, Glen},
year = {2007},
pages = {457--479},
title = {In Search of the Bullwhip Effect},
volume = {9},
journal = {Manufacturing \& Service Operations Management},
}

@article{kurtz1970solutions,
  author  = {Kurtz, Thomas G.},
  title   = {Solutions of ordinary differential equations
             as limits of pure jump {M}arkov processes},
  journal = {Journal of Applied Probability},
  volume  = {7},
  number  = {1},
  pages   = {49--58},
  year    = {1970},
}

@book{chen2001fundamentals,
  author    = {Chen, Hong and Yao, David D.},
  title     = {Fundamentals of Queueing Networks:
               Performance, Asymptotics, and Optimization},
  series    = {Stochastic Modelling and Applied Probability},
  publisher = {Springer},
  year      = {2001},
}

@article{mandelbaum1998state,
  author  = {Mandelbaum, Avi and Pats, Gennady},
  title   = {State-dependent stochastic networks.
             {P}art {I}: Approximations and applications
             with continuous diffusion limits},
  journal = {Annals of Applied Probability},
  volume  = {8},
  number  = {2},
  pages   = {569--646},
  year    = {1998},
}

@book{whitt2002stochastic,
  author    = {Whitt, Ward},
  title     = {Stochastic-Process Limits: An Introduction
               to Stochastic-Process Limits and Their
               Application to Queues},
  series    = {Springer Series in Operations Research},
  publisher = {Springer},
  year      = {2002},
}

@incollection{scarf1960,
  author    = {Scarf, Herbert},
  title     = {The Optimality of $(S,s)$ Policies in the Dynamic Inventory Problem},
  booktitle = {Mathematical Methods in the Social Sciences},
  editor    = {Arrow, Kenneth J. and Karlin, Samuel and Suppes, Patrick},
  publisher = {Stanford University Press},
  pages     = {196--202},
  year      = {1960}
}

@article{dijkstra1959,
  author  = {Dijkstra, Edsger W.},
  title   = {A Note on Two Problems in Connexion with Graphs},
  journal = {Numerische Mathematik},
  volume  = {1},
  number  = {1},
  pages   = {269--271},
  year    = {1959}
}

@misc{bertnardo_mario_uskono_2026,
  title={Global Supply Chain Disruption \& Resilience},
  author={Bertnardo Mario Uskono},
  year={2026},
  publisher={Kaggle},
  url={https://www.kaggle.com/datasets/bertnardomariouskono/global-supply-chain-disruption-and-resilience},
  doi={10.34740/KAGGLE/DSV/14387044}
}

\newpage
\appendix

\section{Parameter values}
\label{app:param-values}
 
This appendix lists the parameter values used to produce the released runs. The algorithms that use these parameters are in \S\ref{app:params}.
 
\subsection{Demand-generator coefficients}
\label{app:params:demand}
 
Each item $i$ draws its demand-process coefficients independently from the distributions of Table~\ref{tab:demand-params}. The same distributions are used at both cardinalities $C = 50$ and $C = 200$; with shared seed, the coefficients of the first 50 items at $C = 200$ coincide with those of the entire $C = 50$ release. The per-item unit volume $v_i$ used by the packing routine of \S\ref{app:pack} is drawn from $\mathcal{U}[1.0, 4.0]$ using the same shared seed.
 
\begin{table}[h]
\centering
\caption{Distributions of the demand-generator coefficients of Eq.~\eqref{eq:intensity}. All coefficients are drawn once per run from the distributions below. The notation $\mathcal{U}\{a, \ldots, b\}$ denotes the discrete uniform on $\{a, a+1, \ldots, b\}$.}
\label{tab:demand-params}
\begin{tabular}{lll}
\toprule
Symbol & Role & Distribution \\
\midrule
$\bar{\lambda}_i$ & baseline per-item rate & $\mathcal{U}[80, 250]$ \\
$a^{(y)}_{1,i}$ & yearly 1st-harmonic amplitude & $\mathcal{U}[0.12, 0.28]$ \\
$a^{(y)}_{2,i}$ & yearly 2nd-harmonic amplitude & $\mathcal{U}[0.04, 0.10]$ \\
$a^{(w)}_i$ & weekly amplitude & $\mathcal{U}[0.04, 0.10]$ \\
$\phi_i, \psi_i$ & seasonal phases & $\mathcal{U}[0, 2\pi]$ \\
$\phi^{\mathrm{AR}}_i$ & AR(1) coefficient & $\mathcal{U}[0.9990, 0.9996]$ \\
$\sigma_{\mathrm{AR},i}$ & AR(1) innovation scale & $\mathcal{U}[0.008, 0.018]$ \\
$A_i(0)$ & AR(1) initial value & $\mathcal{N}(0, 0.10^2)$ \\
$r_i$ & per-item burst rate & $\mathcal{U}[2 \cdot 10^{-4}, 10^{-3}]$ \\
$\Delta^P_{i,k}$ & per-item burst duration & $\mathcal{U}\{30, \ldots, 179\}$ \\
$h^P_{i,k}$ & per-item burst height & $\mathcal{U}[0.20, 0.70]$ \\
$N$ & macro-shock event count & $\mathcal{U}\{5, \ldots, 11\}$ \\
$t^G_k$ & macro-shock event start & $\mathcal{U}\{0, \ldots, T-1\}$ \\
$\Delta^G_k$ & macro-shock event duration & $\mathcal{U}\{180, \ldots, 1099\}$ \\
$h^G_k$ & macro-shock event height & $\mathcal{U}[0.20, 0.60]$ \\
$g_i$ & per-item macro-shock sensitivity & $\mathcal{U}[0.4, 1.2]$ \\
\bottomrule
\end{tabular}
\end{table}
 
\subsection{Inventory policy and source lead times}
\label{app:params:inv}
 
The $(s, S)$ levels and source lead-time means (Table~\ref{tab:inventory-params}) are configuration constants of the simulator, set per node and identical across the two releases. For each (node, item) pair the realized value is drawn once at construction as $x = \text{base} + \mathcal{U}[-\text{var}, +\text{var}]$, rounded to the nearest integer; clipping ensures $0 \leq s^{n,i} < S^{n,i}$ and initial inventory lies in $[s^{n,i}, S^{n,i}]$. The lead-time mean $\mu^n$ enters the source $(s, S)$ rule at $n \in \mathcal{N}_{\mathrm{src}}$; at intermediate nodes the realized arrival time is determined by the routing and packing of sub-steps 5--6 in Algorithm~\ref{alg:day}.
 
\begin{table}[h]
\centering
\caption{Per-node inventory parameters (base $\pm$ var) and source lead-time means $\mu^n$. Drawn once per (node, item) at simulator construction. ``LM'' marks last-mile DCs.}
\label{tab:inventory-params}
\begin{tabular}{lllllc}
\toprule
Node & Tier & Initial inv. & $s$ & $S$ & $\mu^n$ (days) \\
\midrule
SanFrancisco & Source & $4000 \pm 400$ & $400 \pm 60$ & $4000 \pm 400$ & $3 \pm 1$ \\
St.\,Louis & Source & $4000 \pm 400$ & $400 \pm 60$ & $4000 \pm 400$ & $3 \pm 1$ \\
Orlando & Source & $4000 \pm 400$ & $400 \pm 60$ & $4000 \pm 400$ & $3 \pm 1$ \\
Nashville & Hub & $8000 \pm 800$ & $1000 \pm 150$ & $8000 \pm 800$ & $3 \pm 1$ \\
Atlanta & Tier-2 & $6000 \pm 600$ & $500 \pm 80$ & $6000 \pm 600$ & $1$ \\
Chicago & Tier-3 & $5000 \pm 500$ & $1000 \pm 150$ & $5000 \pm 500$ & $8 \pm 1$ \\
Charlotte & Tier-3 & $5000 \pm 500$ & $1000 \pm 150$ & $5000 \pm 500$ & $7 \pm 1$ \\
Memphis & Tier-3 & $3000 \pm 300$ & $500 \pm 80$ & $3000 \pm 300$ & $7 \pm 1$ \\
Columbus & Tier-4 & $4000 \pm 400$ & $500 \pm 80$ & $4000 \pm 400$ & $2$ \\
Richmond & Tier-4 & $4000 \pm 400$ & $500 \pm 80$ & $4000 \pm 400$ & $2$ \\
Philadelphia & Tier-5 (LM) & $3000 \pm 300$ & $500 \pm 80$ & $3000 \pm 300$ & $1$ \\
Baltimore & Tier-5 (LM) & $3000 \pm 300$ & $500 \pm 80$ & $3000 \pm 300$ & $2$ \\
\bottomrule
\end{tabular}
\end{table}
 
\subsection{Edge transport parameters}
\label{app:params:edges}
 
Each edge $e = (u, v)$ has a transit time $\tau_e$ in time units, a per-container volume $V_e$ in the same units, and a per-step container count $K_e$, giving a per-step volume capacity $C_e = K_e V_e$. The two releases share $\tau_e$ and $K_e = 3$ on every edge; $V_e$ scales by a factor of four on the fourteen upstream edges (Table~\ref{tab:edge-params}). The two last-mile edges Philadelphia$\to$NewYork and Baltimore$\to$NewYork are back-solved at runtime; the formula is given in Appendix~\ref{app:params:backsolve}.
 
\begin{table}[h]
\centering
\caption{Per-edge transport parameters at the two cardinalities $C\in\{50, 200\}$. $\tau_e$ in time units, $V_e$ in unit-equivalent volume, $K_e$ per-step container count. Last-mile edges are back-solved (\S\ref{app:params:backsolve}).}
\label{tab:edge-params}
\begin{tabular}{lcccc}
\toprule
Edge & $\tau_e$ & $V_e$ ($C{=}50$) & $V_e$ ($C{=}200$) & $K_e$ \\
\midrule
SanFrancisco $\to$ Nashville & 4 & 5{,}000 & 20{,}000 & 3 \\
St.\,Louis $\to$ Nashville & 2 & 5{,}000 & 20{,}000 & 3 \\
Orlando $\to$ Nashville & 2 & 5{,}000 & 20{,}000 & 3 \\
Nashville $\to$ Atlanta & 1 & 15{,}000 & 60{,}000 & 3 \\
Atlanta $\to$ Chicago & 8 & 4{,}000 & 16{,}000 & 3 \\
Atlanta $\to$ Charlotte & 7 & 4{,}000 & 16{,}000 & 3 \\
Atlanta $\to$ Memphis & 7 & 4{,}000 & 16{,}000 & 3 \\
Chicago $\to$ Columbus & 2 & 4{,}000 & 16{,}000 & 3 \\
Charlotte $\to$ Richmond & 2 & 4{,}000 & 16{,}000 & 3 \\
Columbus $\to$ Philadelphia & 2 & 4{,}000 & 16{,}000 & 3 \\
Richmond $\to$ Philadelphia & 1 & 4{,}000 & 16{,}000 & 3 \\
Richmond $\to$ Baltimore & 3 & 3{,}000 & 12{,}000 & 3 \\
Columbus $\to$ Baltimore & 3 & 3{,}000 & 12{,}000 & 3 \\
Memphis $\to$ Baltimore & 2 & 3{,}000 & 12{,}000 & 3 \\
Philadelphia $\to$ NewYork & 1 & back-solved & back-solved & 3 \\
Baltimore $\to$ NewYork & 2 & back-solved & back-solved & 3 \\
\bottomrule
\end{tabular}
\end{table}
 
\subsection{Last-mile back-solve}
\label{app:params:backsolve}
 
Let $\bar{\lambda}$ denote the realized mean intensity of the run, $\bar{\lambda} := (CT)^{-1} \sum_{i,t} \lambda_{i,t}$, computed after the demand tensor is sampled. The target combined per-step volume for the two last-mile edges, with a buffer and an allowance for wasted packing space, is
\begin{equation}
    V_{\mathrm{raw}} = \frac{C \,\bar{\lambda}\, \bar{v}\, \rho}{\eta}, \qquad \bar{v} = 2.5,\ \rho = 1.20,\ \eta = 0.93,
    \label{eq:back-solve}
\end{equation}
where $\bar{v}$ is the midpoint of the per-item volume distribution $\mathcal{U}[1.0, 4.0]$ (used in place of the realized sample mean to keep the back-solve deterministic in the demand tensor), $\rho$ is a target load factor giving 20\% headroom above the mean, and $\eta$ accounts for greedy first-fit waste. The split between the two edges is fixed at $0.55 : 0.45$, after which each share is divided by $K_e = 3$ and rounded to the nearest 100 to give the integer per-container volumes. The values of $\rho$ and $\eta$ are engineering choices, held fixed across the two releases; varying them is not part of this work.
 
\subsection{Scenario sweeps}
\label{app:scenarios}

The baseline release uses the default values of every
scenario knob in Tables~\ref{tab:demand-params}--\ref{tab:edge-params}. To study
how the network's behavior and the resulting time series
change under different operating conditions, we construct six one-at-a-time sweeps: each varies a single knob across five settings---a shared baseline (all knobs at default) plus four perturbations---while holding all other knobs at baseline,
producing $6 \times 5 = 30$ rollouts on the $C{=}50$
regime.

Table~\ref{tab:scenario-sweeps} lists all six sweeps;
the baseline value is bold in each row. Three sweeps
perturb the demand process and three perturb the
network's response:
\begin{itemize}[leftmargin=1.4em,itemsep=2pt,topsep=2pt]
  \item \emph{Drift}: sets the AR(1) persistence to a
    single shared value $\phi^{\mathrm{AR}}$ across all
    items (Eq.~\ref{eq:drift}); baseline $0.9993$.
  \item \emph{Shock}: multiplies the macro-shock event
    count $N$ and height $h^G$
    (Eq.~\ref{eq:macro-shock}); baseline $(1,1)$.
  \item \emph{Burst}: multiplies the per-item burst rate
    $r_i$ and height $h^P$ (Eq.~\ref{eq:bursts});
    baseline $(1,1)$.
  \item \emph{Edge cap}: multiplies the per-edge container
    count $K_e$; baseline $1$.
    \item \emph{Buffer}: multiplies the per-node reorder
    thresholds $(s, S)$ and initial inventory;
    baseline $1$.
  \item \emph{Lead time}: multiplies the source-node mean
    $\mu^n$; edge transit times $\tau_e$ unchanged;
    baseline $1$.
\end{itemize}

\begin{table}[h]
\centering\small
\caption{The six scenario sweeps. Each sweep varies a
single knob across five settings while all others stay at
baseline (bold). The Type column indicates whether the
sweep changes the demand process or the network's response
to it. Details of each sweep are in the text. }

\label{tab:scenario-sweeps}
\begin{tabular}{@{}l c l l@{}}
\toprule
Sweep      & Notation                            & Type    & Settings {\bf (baseline)} \\
\midrule
Drift      & $\phi^{\mathrm{AR}}$              & demand  & $0.71,\; 0.86,\; 0.96,\; 0.99,\; \mathbf{0.9993}$ \\
\midrule
Shock      & $(N \times h^G)$                  & demand  & $(0,1),\;(0.5,0.7),\;\mathbf{(1,1)},\;(2,2),\;(3,4)$ \\
Burst      & $(r \times h^P)$                  & demand  & $\mathbf{(1,1)},\;(1.5,2),\;(2,3),\;(3,4),\;(5,8)$ \\
\midrule
Edge cap   & $\rho_K$                          & supply  & $0.3,\; 0.6,\; \mathbf{1.0},\; 1.5,\; 2.5$ \\
Buffer     & $\rho_{sS}$                       & supply  & $0.1,\; 0.2,\; 0.5,\; 0.75,\; \mathbf{1.0}$ \\
Lead time  & $\rho_{lt}$                       & supply  & $\mathbf{1.0},\; 2.0,\; 5.0,\; 10.0,\; 20.0$ \\
\bottomrule
\end{tabular}
\end{table}

We select five demand-side scenarios from the six sweeps
of Table~\ref{tab:scenario-sweeps}; three are single-knob
picks and two are compound scenarios that combine
perturbations along multiple axes:
\label{fivescenarios}
\begin{itemize}[leftmargin=1.4em,itemsep=1pt,topsep=2pt]
\item \texttt{shock\_xhi}: $(3, 4)$ from the shock sweep.
\item \texttt{drift\_mid}:
  $\phi^{\mathrm{AR}}{\in}[0.95, 0.97]$ from the drift
  sweep.
\item \texttt{burst\_xhi}: $(3, 4)$ from the burst sweep.
\item \texttt{chaos\_compound}: \texttt{shock\_xhi} with
  $\phi^{\mathrm{AR}}{\in}[0.96, 0.98]$.
\item \texttt{chaos\_burst}: \texttt{burst\_xhi} with
  $\phi^{\mathrm{AR}}{\in}[0.96, 0.98]$.
\end{itemize}
Supply-side sweeps leave the demand series unchanged and
therefore cannot propagate to forecasts of $y_{i,t}$; we
restrict the evaluation to demand-side knobs.
The evaluation protocol matches
\S\ref{sec:foundation-protocol} throughout.

\subsection{Runtime parameters}
\label{app:params:invocation}
 
The two releases were produced by running the simulator script with the parameters of Table~\ref{tab:invocation}. All values are the script's built-in defaults, except for the pipeline multiplier $m$. The default $m = 0$ selects a reactive shipping rule that ships against backlog plus a three-time-unit buffer of smoothed demand; $m = 7$ selects the proactive rule used in this work, which keeps seven time units of smoothed demand in the pipeline at all times.
 
\begin{table}[h]
\centering
\caption{Runtime parameters used to produce the two releases.
The pipeline multiplier $m = 7$ is the single value that
differs from the script's defaults.}
\label{tab:invocation}
\begin{tabular}{ll}
\toprule
Parameter & Value \\
\midrule
Horizon $T$ & $7{,}300$ (daily steps) \\
Seed & 2025 \\
Pipeline multiplier $m$ & 7.0 (default 0.0) \\
Packing & greedy first-fit \\
Streaming output & enabled \\
\bottomrule
\end{tabular}
\end{table}

\section{Extending the base transition map}
\label{app:extensions}

Algorithm~\ref{alg:day} fixes seven sub-steps, but not their contents. Because
each sub-step reads and writes the state vector $\xi_t$ of
Eq.~\eqref{eq:state-vector}, a modified sub-step yields a valid simulator as
long as it leaves that vector's semantics intact. Two properties are worth
tracking separately when extending: the \emph{Markov closure} of
\S\ref{sec:markov-consequence}---whether $\xi_{t+1}$ still depends on the past
through $\xi_t$ and the external inputs alone---and the \emph{conservation
laws} of Proposition~\ref{prop:conservation}, which additionally require that
no unit and no demand event is created or destroyed. Closure can survive a
change that conservation does not, and vice versa; the two are not
interchangeable, which is why Tables~\ref{tab:extension-shipped} and
\ref{tab:extension-future} report them as one column with per-row
qualifications rather than as a single verdict.

\paragraph{Dynamics available in the release.}
Beyond the base configuration, five further dynamics are implemented and
exposed as command-line flags, each with an archived example rollout under
\texttt{data/output\_mixture/} (Table~\ref{tab:extension-shipped}).
Four of them add a stochastic mechanism to the transition function; the
fifth is a deterministic scenario control applied to an otherwise unchanged
mechanism, which is why it is described separately below.
\begin{itemize}[leftmargin=1.4em,itemsep=3pt,topsep=2pt]
  \item \emph{Non-infinite source supply} (\texttt{reservoir\_tight}).
    Flags \texttt{--prod\_rate\_per\_item} and
    \texttt{--reservoir\_cap\_per\_item} replace unbounded source inventory
    with a bounded reservoir replenished at a fixed rate. The reservoir level
    joins the state as an additional per-(source, item) coordinate, so closure
    is preserved; conservation is preserved because units leaving the
    reservoir are exactly the units entering the transport system. Emits
    \texttt{reservoir\_history} alongside the standard files.
  \item \emph{Finite warehouse capacity} (\texttt{warehouse\_cap2x}).
    Flag \texttt{--warehouse\_cap\_scale} caps per-node on-hand inventory.
    Arrivals exceeding the cap are held on the inbound edge rather than
    discarded; this choice is what preserves
    Proposition~\ref{prop:conservation}. A spillover policy that discarded
    excess would keep the Markov property but break conservation.
  \item \emph{Non-deterministic edge transit} (\texttt{tt\_noise10}).
    Flag \texttt{--edge\_tt\_std\_frac} makes the realized transit time on
    edge $e$ a random draw around $\tau_e$. Closure is preserved by treating
    the draw as an external random input alongside $L_t$
    (\S\ref{app:exogenous}); the conservation laws are untouched, since they
    do not reference transit duration.
  \item \emph{Supplier disruption} (\texttt{disruption\_mid}).
    Flags \texttt{--disruption\_p\_onset} and
    \texttt{--disruption\_length\_days} switch source availability off for a
    sampled window. Both closure and conservation are unaffected: an
    unavailable source simply places no orders. Emits
    \texttt{source\_availability}.
\end{itemize}
A fifth capability, \emph{dynamic edge cut and restore}, relaxes the fixed
topology assumption and is used to generate the disruption scenarios of
\S\ref{sec:datasets-validation:disruption}. \emph{Item-specific shipping volumes} are not
an extension at all---they are part of the base model, since each item carries
its own unit volume $v_i$ into the packing step of Algorithm~\ref{alg:pack}.

\begin{table}[h]
\centering\footnotesize
\caption{Dynamics available in the released simulator. ``Closure'' asks
whether the Markov property of \S\ref{sec:markov-consequence} and the
conservation identities of Prop.~\ref{prop:conservation} survive;
``$\checkmark$*'' marks a conditional answer explained in the note. Each row
below the first has an archived example rollout under
\texttt{data/output\_mixture/}.}
\label{tab:extension-shipped}
\setlength{\tabcolsep}{4pt}\renewcommand{\arraystretch}{1.15}
\begin{tabularx}{\linewidth}{@{}>{\raggedright\arraybackslash}p{0.24\linewidth} l c >{\raggedright\arraybackslash}X@{}}
\toprule
Dynamic & Example run & Closure & Note \\
\midrule
Item-specific shipping volumes & (base model) & $\checkmark$ &
  Per-item unit volume $v_i$ enters the packing step
  (Algorithm~\ref{alg:pack}); not an extension but a property of the base
  configuration. \\
Non-infinite source supply & \texttt{reservoir\_tight} & $\checkmark$ &
  Bounded source reservoir; relaxes unbounded upstream supply. \\
Finite warehouse capacity & \texttt{warehouse\_cap2x} & $\checkmark$* &
  Conservation holds because excess arrivals are retained in the transport
  system; a policy that discarded them would keep closure but break
  \eqref{eq:conservation}. \\
Non-deterministic edge transit & \texttt{tt\_noise10} & $\checkmark$ &
  Stochastic travel times augment the external random inputs; relaxes
  deterministic edge transit. \\
Supplier disruption & \texttt{disruption\_mid} & $\checkmark$ &
  On/off source availability; state and accounting unchanged. \\
Dynamic edge cut/restore & \texttt{ripple\_scenario}, \texttt{bottleneck\_scenario} & $\checkmark$ &
  Relaxes stationary topology; generates the disruption dynamics of
  \S\ref{sec:datasets-validation:disruption}. \\
\bottomrule
\end{tabularx}
\end{table}

\paragraph{Further dynamics and their effect on the state-space design.}
Table~\ref{tab:extension-future} lists the dynamics that users and reviewers
have asked for and that the released configuration does not implement. For each
it reports the implementation effort we estimate and, in the column that
carries the actual content, whether closure and the conservation identities
would survive. That column is what a user needs before modifying the simulator.
A row marked $\checkmark$ can be added and the result still checked against
Proposition~\ref{prop:conservation}, so the verification tools of
\S\ref{app:properties} continue to apply. The row marked $\times$---substitution
across items---changes the accounting itself: it couples demand events between
items, so \eqref{eq:backlog-thm} must be restated before any check of a
modified implementation means anything. Three rows are conditional, and each
note states the condition rather than leaving it to the effort tier. The whole
classification is a design argument, not a measurement: the effort tiers are
our estimate, and the closure column reasons about the state vector rather than
reporting a verified property of code that does not yet exist.

\begin{table}[h]
\centering\footnotesize
\caption{Extension pathways. For each dynamic not exercised by the released
configuration: implementation effort (a coarse tier, not a time estimate) and
whether the Markov property of \S\ref{sec:markov-consequence} and the
conservation identities of Prop.~\ref{prop:conservation} are preserved.}
\label{tab:extension-future}
\setlength{\tabcolsep}{4pt}\renewcommand{\arraystretch}{1.15}
\begin{tabularx}{\linewidth}{@{}>{\raggedright\arraybackslash}p{0.24\linewidth} l c >{\raggedright\arraybackslash}X@{}}
\toprule
Dynamic & Effort & Closure & Note \\
\midrule
Non-Dijkstra routing & small--large & $\checkmark$ &
  Preserved for any \emph{fixed} routing policy; the routing module is already
  swappable. Effort scales with the alternative, from $k$-shortest paths
  through min-cost flow to learned routing. \\
Online closed-loop control & small & $\checkmark$* &
  Preserved if the controller is a function of the current state; a controller
  with memory requires state augmentation. This is the one standing assumption
  of \S\ref{app:replenish} still in force in every released run. \\
Multi-destination allocation & moderate & $\checkmark$ &
  Backlog becomes per-destination; \eqref{eq:per-node-mass}--\eqref{eq:backlog-thm}
  generalize directly. \\
Time-dependent congestion & moderate & $\checkmark$* &
  Transit time becomes a function of current edge load; preserved for
  dependence on the current state, fails for dependence on history. \\
Vehicle scheduling & large & $\checkmark$* &
  Adds vehicle-position state; conservation holds, but the state dimension
  grows with fleet size. \\
Substitution / cannibalization & large & $\times$ &
  Cross-item demand coupling breaks the per-item demand-event accounting of
  \eqref{eq:backlog-thm}; the backlog identity would need restating. \\
\bottomrule
\end{tabularx}
\end{table}

\section{Dataset schema}
\label{app:schema}

\subsection{Release overview}
\label{app:schema-conventions}

\paragraph{Directory structure.}
A release contains one directory, \texttt{output\_item50/} or \texttt{output\_item200/}, containing six CSV files, one NumPy array, and one short text sidecar. Both releases use the same schema; they differ in $C$ and in the row counts of files that scale with $C$. The scenario library is a sibling tree, \texttt{output\_mixture/<scenario>/}, whose rollouts use this same schema; App.~\ref{app:schema-folders} maps each directory to the experiment that reads it.

\paragraph{Time index.}
The time index $t$ runs over $\{0, \ldots, T-1\}$ with $T = 7{,}300$ in both releases.
Each file stores the index as an integer column named \texttt{day}; the column name reflects the resolution of the released runs (one step is one day). The simulator itself accepts any user-chosen step length, and at other resolutions the column carries the same integer index of $t$.

\paragraph{Item identifiers.}
Items are zero padded strings: \texttt{I01}--\texttt{I50} at $C{=}50$ and \texttt{I001}--\texttt{I200} at $C{=}200$.

\paragraph{Sidecar metadata.}
The release also contains \texttt{demand\_signals\_cols.txt}, a one-line text file listing the item identifiers in the column order used by \texttt{demand\_signals.npy}.

\begin{figure}[t]
\centering
\includegraphics[width=\linewidth]{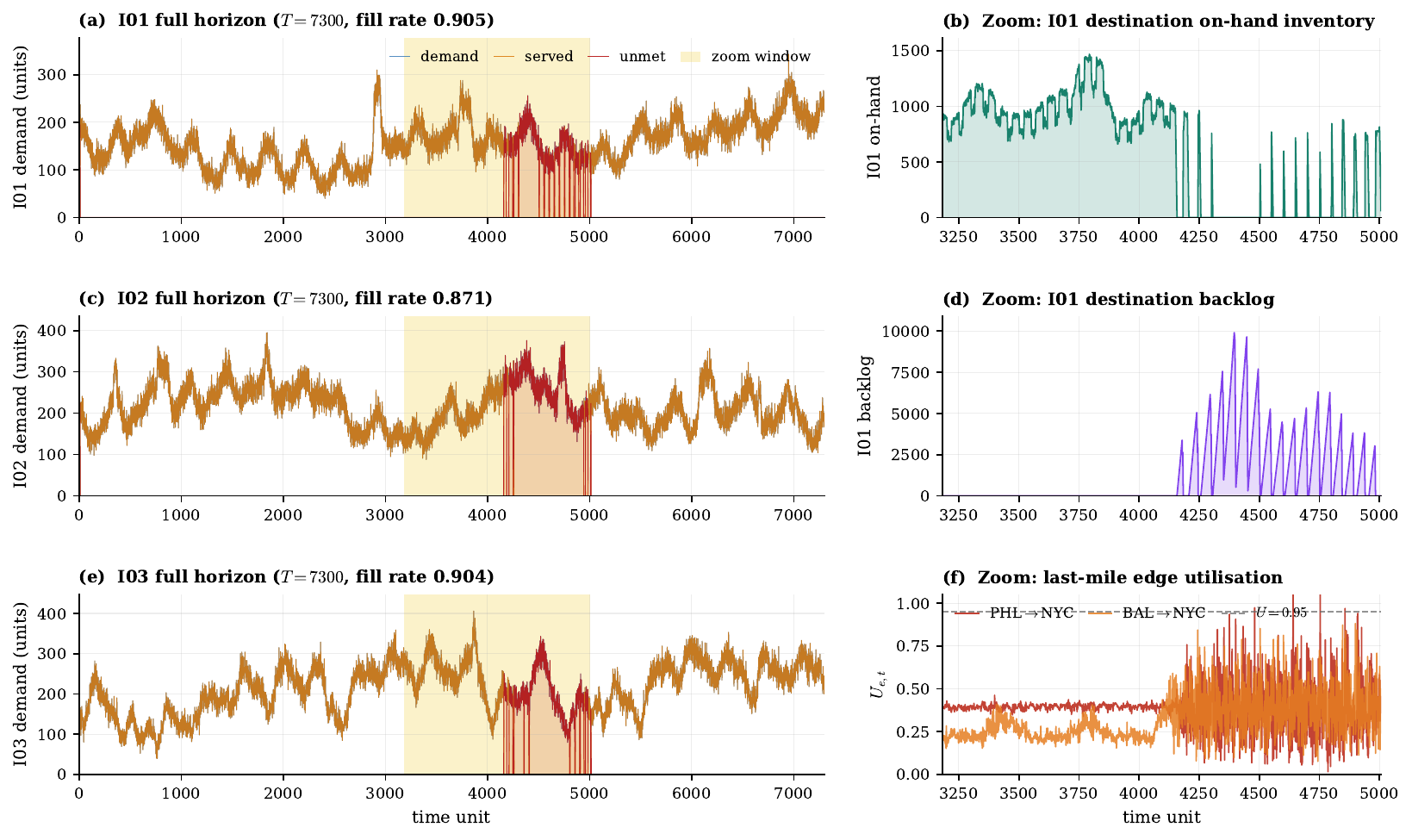}
\caption{Baseline release at $C{=}50$, raw values at each daily
time unit over the full horizon $T{=}7{,}300$.
\textbf{Left column (a, c, e):} for items I01--I03, the daily demand, the part
served immediately from destination inventory, and the part left unmet (red,
shaded). Per-item fill rates over the whole horizon ($0.905$, $0.871$,
$0.904$) appear in each panel title; the yellow band marks the five-year zoom
window (days $3183$--$5008$) selected as the window of largest rolling
aggregate demand. \textbf{Right column (b, d, f):} the network's response for
item I01 inside that window --- \textbf{(b)} destination on-hand inventory,
\textbf{(d)} destination backlog, and \textbf{(f)} utilization of the two
last-mile edges (Philadelphia$\to$NewYork, Baltimore$\to$NewYork), with the
$U{=}0.95$ saturation line marked. The three panels read together: on-hand
collapses to zero from roughly day $4150$, backlog rises in the same window to
a peak near $10^{4}$ units, and last-mile utilization becomes volatile rather
than saturated, since there is no stock upstream to ship.}
\label{fig:baseline-overview}
\end{figure}

\subsection{File schemas}
\label{app:schema-files}
\label{app:schema-inventory}

\paragraph{File inventory.}
Table~\ref{tab:schema-overview} lists every file in a release, the columns that uniquely identify a row, and the number of rows or the array shape. The paragraphs that follow give the columns of each file, in the order of the table.

\begin{table}[h]
\centering
\small
\caption{File inventory of one \sysname release. ``Key'' is the
combination of columns that uniquely identifies a row.
``Rows'' is given as a function of the cardinality $C\in\{50,200\}$,
the horizon $T = 7{,}300$, and the node count $|\mathcal{N}|=13$.}
\label{tab:schema-overview}
\begin{tabular}{@{}l l l@{}}
\toprule
File & Key & Rows / shape \\
\midrule
\texttt{daily\_records.csv}        & (day, item)         & $T \cdot C$ \\
\texttt{shipments.csv}             & none (event log)    & variable ($4.90{\times}10^{5}$ at $C{=}50$; $1.80{\times}10^{6}$ at $C{=}200$) \\
\texttt{inventory\_history.csv}    & (day, node, item)   & $T \cdot |\mathcal{N}| \cdot C$ \\
\texttt{backlog\_history.csv}      & (day, node, item)   & $T \cdot |\mathcal{N}| \cdot C$ \\
\texttt{intransit\_history.csv}    & (day, item)         & $T \cdot C$ \quad (destination only) \\
\texttt{service\_summary.csv}      & item                & $C$ \\
\texttt{demand\_signals.npy}       & ---                 & $(T, C)$ \\
\texttt{demand\_signals\_cols.txt} & ---                 & 1 line ($C$ item identifiers) \\
\bottomrule
\end{tabular}
\end{table}

\paragraph{\texttt{daily\_records.csv}.}
One row per $(t, i)$ pair, recording per-item demand and service at $d^\star$. The two \texttt{*\_end\_before\_ship} columns are the destination's on-hand and backlog snapshots taken between sub-steps (4) and (5) of Algorithm~\ref{alg:day} --- after demand and service but before any dispatch. So that $\texttt{served\_from\_stock} + \texttt{new\_backlog\_today} = \texttt{demand}$.

\begin{table}[h]
\centering\small
\caption{Schema of \texttt{daily\_records.csv}.}
\label{tab:schema-daily}
\begin{tabular}{@{}l l p{0.55\linewidth}@{}}
\toprule
Column & Type & Description \\
\midrule
\texttt{day}                              & int   & time index $t$ \\
\texttt{item}                             & str   & item identifier \\
\texttt{demand}                           & int   & realized demand $y_{i,t}$ (Poisson sample, Eq.~\ref{eq:poisson-demand}) \\
\texttt{served\_from\_stock}              & int   & $\min(\mathrm{OH}^{d^\star,i}_t, y_{i,t})$ \\
\texttt{new\_backlog\_today}              & int   & $(y_{i,t} - \mathrm{OH}^{d^\star,i}_t)_+$ \\
\texttt{dest\_on\_hand\_end\_before\_ship}& int   & destination on-hand right after service, before sub-step (5) \\
\texttt{dest\_backlog\_end\_before\_ship} & int   & destination backlog right after service, before sub-step (5) \\
\bottomrule
\end{tabular}
\end{table}

\paragraph{\texttt{shipments.csv}.}
One row per dispatch committed by \textsc{GreedyPack} in sub-steps (5) and (6) of Algorithm~\ref{alg:day}. A single $(t, \text{from}, \text{to}, i)$ combination may appear in multiple rows, since round-robin dispatch can revisit an item within the same step.

\begin{table}[h]
\centering\small
\caption{Schema of \texttt{shipments.csv}.}
\label{tab:schema-ship}
\begin{tabular}{@{}l l p{0.55\linewidth}@{}}
\toprule
Column & Type & Description \\
\midrule
\texttt{day}          & int  & dispatch step $t$ \\
\texttt{arrival\_day} & int  & arrival step at \texttt{to}, $t + \lceil\sum_{e \in P}\tau_e\rceil$ \\
\texttt{from}         & str  & origin node of the dispatch \\
\texttt{to}           & str  & path endpoint \\
\texttt{item}         & str  & item identifier \\
\texttt{units}        & int  & units dispatched on this row \\
\texttt{path\_nodes}  & str  & repr of \texttt{list[str]}: resolved Dijkstra path $P=(v_0,\dots,v_k)$ \\
\texttt{edge\_times}  & str  & repr of \texttt{list[float]}: per-edge transit times $(\tau_{e_1},\dots,\tau_{e_k})$ in time units \\
\bottomrule
\end{tabular}
\end{table}

\paragraph{\texttt{inventory\_history.csv}, \texttt{backlog\_history.csv}, \texttt{intransit\_history.csv}.}
Three long form files record the network state at the end of each step. All three share the key columns (\texttt{day}, \texttt{node}, \texttt{item}) of types (int, str, str) and differ in the value column they carry (Table~\ref{tab:schema-history}). \texttt{inventory\_history.csv} and \texttt{backlog\_history.csv} cover every node $n \in \mathcal{N}$; \texttt{intransit\_history.csv} restricts to the destination $d^\star$, so its \texttt{node} column is constant. Backlog is zero at every node other than $d^\star$ by construction (\S\ref{app:replenish}) and is recorded as zero for schema regularity. All three values are taken at the end of step $t$, after all seven sub-steps of Algorithm~\ref{alg:day} have run.

\begin{table}[h]
\centering\small
\caption{Value column of each history file. All three share the key
(\texttt{day}, \texttt{node}, \texttt{item}).}
\label{tab:schema-history}
\begin{tabular}{@{}l l l p{0.45\linewidth}@{}}
\toprule
File & Value column & Type & Description \\
\midrule
\texttt{inventory\_history.csv} & \texttt{on\_hand}    & int & on-hand inventory $\mathrm{OH}^{n,i}_t$ \\
\texttt{backlog\_history.csv}   & \texttt{backlog}     & int & backlog $B^{n,i}_t$ (zero unless $n = d^\star$) \\
\texttt{intransit\_history.csv} & \texttt{in\_transit} & int & total units of item $i$ already scheduled to arrive at $d^\star$ on any step after $t$: $\sum_{\tau > t}\mathrm{IT}_t[\tau, i]$ \\
\bottomrule
\end{tabular}
\end{table}

\paragraph{\texttt{service\_summary.csv}.}
One row per item with five columns: \texttt{item} (str), \texttt{total\_demand} (int), \texttt{served\_from\_stock} (int), \texttt{new\_backlog\_added} (int), and \texttt{fill\_rate\_stock\_only} (float). The three integer columns aggregate the corresponding columns of \texttt{daily\_records.csv} over the full horizon; the float column is \texttt{served\_from\_stock} divided by \texttt{total\_demand}, rounded to six decimal places. The file is derived from \texttt{daily\_records.csv} and included for convenience.

\paragraph{\texttt{demand\_signals.npy}.}
It is a \texttt{float64} array of shape $(T, C)$ holding the intensity values $\lambda_{i,t}$ used to sample the \texttt{demand} column of \texttt{daily\_records.csv} (Eq.~\ref{eq:intensity}); pairing samples with rates lets users study sample-versus-rate behavior without rerunning the simulator. Columns are indexed by item identifier in alphabetical order, listed comma-separated on a single line in \texttt{demand\_signals\_cols.txt}.

\subsection{Data directories and their associated experiments}
\label{app:schema-folders}

Table~\ref{tab:data-map} summarizes the released data directories and their
corresponding experiments and analyses. The directory structure contains two
points that are not evident from the directory names alone.

\emph{The forward-UQ ensemble does not have a separate data directory.} It
uses a $K{=}20$ subset of the same $29$ scenario rollouts used for the
forecasting evaluation (\S\ref{sec:uq}). Thus, no additional rollouts are
generated specifically for forward UQ; the two experiments use the same
scenario files but aggregate them differently.

\emph{Four of the $33$ rollouts are demonstrations rather than evaluation
data.} The \texttt{disruption\_mid}, \texttt{reservoir\_tight},
\texttt{warehouse\_cap2x}, and \texttt{tt\_noise10} rollouts each activate
one physical-realism extension (App.~\ref{app:extensions}). These rollouts
are excluded from the forecasting cohort because the edge arrays required
for the edge-utilization target are generated by a separate derivation step,
which is applied to the two edge-cut rollouts but not to these four extension
demonstrations.

\begin{table}[h]
\centering\footnotesize
\caption{Released data directories and their associated experiments and
analyses. The $33$ rollouts in \texttt{output\_mixture/} comprise the
$29$-rollout forecasting cohort and four extension demonstrations. The
forward-UQ ensemble is a subset of the forecasting cohort rather than a
separate set of rollouts.}
\label{tab:data-map}
\setlength{\tabcolsep}{4pt}\renewcommand{\arraystretch}{1.15}
\begin{tabularx}{\linewidth}{@{}>{\raggedright\arraybackslash}p{0.30\linewidth} c >{\raggedright\arraybackslash}X@{}}
\toprule
Directory & Rollouts & Read by \\
\midrule
\texttt{output\_item50/},\newline\texttt{output\_item200/} & 2 &
  Baseline releases at $C{=}50$ and $C{=}200$; the schema of
  App.~\ref{app:schema} and the overview of
  Fig.~\ref{fig:baseline-overview}. \\
\addlinespace[2pt]
\texttt{output\_mixture/}\newline\texttt{<sweep>/seed2025/} & 27 &
  Baseline plus the $24$ one-knob variants of App.~\ref{app:scenarios} and the
  two compound-stress scenarios. \\
\addlinespace[2pt]
\texttt{output\_mixture/}\newline\texttt{ripple\_scenario},\newline\texttt{bottleneck\_scenario} & 2 &
  Dynamic edge-cut rollouts; the disruption dynamics of
  \S\ref{sec:datasets-validation:disruption}. \\
\addlinespace[3pt]
\multicolumn{3}{@{}>{\itshape}p{0.98\linewidth}@{}}{The $29$ rollouts above are
the forecasting evaluation cohort of \S\ref{sec:foundation-protocol}; $20$ of
them are the forward-UQ ensemble of \S\ref{sec:uq}.} \\
\addlinespace[3pt]
\texttt{output\_mixture/}\newline\texttt{disruption\_mid},\newline\texttt{reservoir\_tight},\newline\texttt{warehouse\_cap2x},\newline\texttt{tt\_noise10} & 4 &
  One physical-realism extension active in each
  (Table~\ref{tab:extension-shipped}); outside the cohort. \\
\midrule
\texttt{network\_state\_}\newline\texttt{forecast\_datasets/}\newline\texttt{<rollout>\_<target>/} & --- &
  Target-recast copies written by the extraction step for the three
  network-state targets; read by the forecasting drivers of
  \S\ref{sec:foundation-protocol}. \\
\bottomrule
\end{tabularx}
\end{table}

The scenario directories are not uniformly structured. The $27$ sweep
rollouts store their files under a \texttt{seed2025/} leaf, whereas the two
edge-cut rollouts and four extension rollouts store their files at the top
level. Accordingly, data loaders must be pointed to the directory containing
the corresponding files. The extension rollouts additionally contain
\texttt{reservoir\_history} and \texttt{source\_availability}, which are not
written by the base configuration.

\section{Detailed comparison with related simulators}
\label{app:related-detailed}

Table~\ref{tab:related-detailed} extends the high-level comparison in
Table~\ref{tab:related-general} (\S\ref{sec:related-sc}) to the structural
and dynamical details of each simulator, including demand processes,
ordering and control, lead times, capacity and routing, conservation
properties, and scenario generation.

\begin{table}[h]
\centering
\footnotesize
\caption{Structural comparison of \sysname with three related open
supply-chain simulation and learning frameworks.}
\label{tab:related-detailed}
\setlength{\tabcolsep}{3pt}
\renewcommand{\arraystretch}{1.15}
\begin{tabularx}{\linewidth}{@{}l *{4}{>{\raggedright\arraybackslash}X}@{}}
\toprule
Axis
 & \textbf{\sysname} (this work)
 & \textbf{BULL-ODE}~\citep{naik2025bullode}
 & \textbf{deepbullwhip}~\citep{arief2026deepbullwhip}
 & \textbf{OR-Gym}~\citep{hubbs2020orgym} \\
\midrule
Demand process
& Poisson: yearly and weekly seasonality, AR(1) drift, per-item trapezoidal
  bursts, and shared macro-shock $G(t)$ (Eq.~\eqref{eq:intensity})
& AR(1) ($\phi{=}0.6$), i.i.d.\ Gaussian, or lognormal multiplicative
& AR(1) with weekly seasonality and structural shocks; real-data replay
& Stationary Poisson (mean $20$) \\
Ordering / control
& $(s,S)$ at every non-destination node; proactive shipping based on
  demand estimates and pipeline conditions
& Continuous linear order-up-to embedded in the ODE
  ($\tau{=}5$, $\alpha{=}0.8$)
& Pluggable: OUT, POUT, Smoothing OUT, Constant
& Base-stock (IPA / DFO / MIP), shrinking-horizon LP, PPO \\
Lead times
& Stochastic source replenishment; realized path transit time for intermediate
  shipments; deterministic edge transit by default, with stochastic-transit
  extension
& Deterministic per-stage lead time
& Deterministic per-stage lead time
& Deterministic per-stage lead time \\
Capacity / routing
 & Edge-level transportation capacity; capacity-aware shortest-path routing
 & ---
 & Uncapacitated upstream (default)
 & Per-stage production capacity \\
Conservation enforcement
 & Pathwise conservation laws over physical units and demand events
   (Prop.~\ref{prop:conservation}); continuous-time fluid limit
   (\S\ref{sec:conservation-continuous})
 & Inventory conservation $\dot{I}{=}O{-}D$ hard-coded into the UDE residual
 & ---
 & Material balances; no explicit conservation-law framing \\
\bottomrule
\end{tabularx}
\end{table}

Simpler supply-chain configurations can be obtained as reductions of
\sysname{}, allowing direct comparison with the serial-chain settings of
deepbullwhip and OR-Gym. In particular, restricting the catalogue to a single
item and the network to a serial chain preserves the following structural
distinctions.

\begin{itemize}
    \item \textbf{Path-dependent lead times.}
    Source replenishment is stochastic, while intermediate shipments incur
    the realized transit time along their selected paths. This contrasts
    with the deterministic per-stage lead times used by the other
    frameworks.

    \item \textbf{Edge capacity and routing.}
    Transportation capacity is represented on edges rather than solely at
    nodes, allowing lane-level bottlenecks to constrain throughput and
    influence path selection without introducing additional dynamic stock
    variables.

    \item \textbf{Pathwise conservation.}
    Material flow and backlog satisfy explicit pathwise conservation laws,
    enforcing trajectory-level accounting across the network rather than
    relying on local inventory balances.

    \item \textbf{Proactive shipping.}
    Outbound shipping responds to demand estimates and pipeline conditions,
    rather than to currently observed demand alone. An exponentially weighted
    demand estimate provides this history dependence without requiring a
    full demand window in the dynamic state.
\end{itemize}

These choices introduce network-level behavior while keeping the dynamic
state compact and Markovian. BULL-ODE remains a distinct reference point
because its continuous-time single-node formulation does not instantiate a
finite transportation network.

\section{Statistical Comparison with SupplyGraph}
\label{app:supplygraph}

SupplyGraph~\citep{wasi2024supplygraph} provides observations from a consumer
goods company across four supply-chain stages---Sales Order, Delivery to
Distributor, Factory Issue, and Production---covering \(C=41\) SKUs over
221 days. The four stages contain nearly the same aggregate volume
(\(\sim 7.65\times10^6\) cases), indicating that they largely represent
the same material flow observed at different points in the supply chain.
The agreement is not uniform across SKUs: $11$--$12$ of the $41$ differ by
more than $25\%$ between Sales Order and a given downstream stage, so part of
the aggregate match is compensation across SKUs. Since
SupplyGraph is an observational release and does not expose its underlying
data-generating process or network dynamics, we use it as a real-world
statistical reference rather than as a mechanistic validation target.

Three SupplyGraph stages have direct counterparts in \sysname{}: Sales Order
is compared with destination demand, Delivery to Distributor with shipments
into destination nodes, and Factory Issue with shipments from source nodes.
Production is excluded because the base configuration of \sysname{} uses
unbounded source supply. To account for the different units and scales of the
two datasets, the comparison uses dimensionless distributional and temporal
statistics computed over matched 221-day windows. SKUs with insufficient
nonzero observations are excluded from the SupplyGraph statistics.

Table~\ref{tab:supplygraph-stats} summarizes marginal, intermittency, and
temporal statistics for the three comparable stages. Each entry reports the
SupplyGraph median and the range across the 29 \sysname{} rollouts, with
\(\checkmark\) indicating that the observed value falls within the simulator
range. Table~\ref{tab:supplygraph-cross} reports cross-stage amplification
and propagation timing.

\begin{table}[!htb]
\centering
\footnotesize
\caption{Statistical comparison of \sysname{} against SupplyGraph~\citep{wasi2024supplygraph}.
Each cell reports the SupplyGraph median and the range across the 29
\sysname{} rollouts. A \(\checkmark\) indicates that the observed value lies
within the simulator range. All statistics are dimensionless and computed on
matched 221-day windows.}
\label{tab:supplygraph-stats}
\setlength{\tabcolsep}{4pt}
\renewcommand{\arraystretch}{1.15}
\begin{tabularx}{\linewidth}{
@{}>{\raggedright\arraybackslash}p{0.26\linewidth}
*{3}{>{\raggedright\arraybackslash}X}@{}}
\toprule
\textbf{Property} & \textbf{Sales Order} & \textbf{Delivery}
& \textbf{Factory Issue} \\
\midrule

\multicolumn{4}{@{}l}{\emph{Marginal shape}} \\
Coefficient of variation
& 1.03 / [0.16, 0.44] $\times$
& 0.74 / [0.20, 6.59] $\checkmark$
& 0.92 / [1.66, 4.94] $\times$ \\

CV of nonzero values
& 0.79 / [0.16, 0.44] $\times$
& 0.69 / [0.20, 1.12] $\checkmark$
& 0.79 / [0.00, 0.64] $\times$ \\

Skewness
& 1.44 / [0.11, 0.31] $\times$
& 1.04 / [0.15, 6.68] $\checkmark$
& 1.52 / [1.65, 5.06] $\times$ \\

\addlinespace[2pt]
\multicolumn{4}{@{}l}{\emph{Intermittency}} \\

Zero fraction
& 0.21 / [0.00, 0.00] $\times$
& 0.04 / [0.00, 0.96] $\checkmark$
& 0.12 / [0.68, 0.96] $\times$ \\

Average demand interval (d)
& 1.26 / [1.00, 1.00] $\times$
& 1.04 / [1.00, 22.33] $\checkmark$
& 1.12 / [3.12, 23.00] $\times$ \\

\addlinespace[2pt]
\multicolumn{4}{@{}l}{\emph{Temporal dependence}} \\

Lag-1 autocorrelation
& 0.15 / [0.71, 0.98] $\times$
& 0.41 / [0.03, 0.64] $\checkmark$
& 0.35 / [-0.20, -0.03] $\times$ \\

Lag-7 autocorrelation
& 0.46 / [0.72, 0.89] $\times$
& 0.21 / [-0.05, 0.64] $\checkmark$
& 0.11 / [-0.10, 0.05] $\times$ \\

Weekly spectral power ratio
& 7.49 / [0.36, 4.80] $\times$
& 3.32 / [0.06, 5.71] $\checkmark$
& 2.28 / [0.19, 0.79] $\times$ \\

\addlinespace[2pt]
\multicolumn{4}{@{}l}{\emph{After weekly aggregation}} \\

Coefficient of variation
& 0.45 / [0.13, 0.42] $\times$
& 0.43 / [0.16, 2.61] $\checkmark$
& 0.54 / [0.56, 1.93] $\times$ \\

Skewness
& 0.62 / [0.05, 0.36] $\times$
& 0.57 / [0.08, 2.33] $\checkmark$
& 0.45 / [0.15, 1.74] $\checkmark$ \\

Lag-1 autocorrelation
& 0.36 / [0.92, 0.95] $\times$
& 0.39 / [-0.15, 0.88] $\checkmark$
& 0.37 / [-0.20, 0.39] $\checkmark$ \\

\bottomrule
\end{tabularx}
\end{table}

\begin{table}[!htb]
\centering
\footnotesize
\caption{Cross-stage statistics. Bullwhip is the ratio of upstream to
downstream squared coefficients of variation; values above one indicate
amplification. Lead-lag is the lag maximizing the cross-correlation between
the two stages.}
\label{tab:supplygraph-cross}
\setlength{\tabcolsep}{4pt}
\renewcommand{\arraystretch}{1.15}
\begin{tabularx}{\linewidth}{
@{}>{\raggedright\arraybackslash}p{0.36\linewidth}
c >{\raggedright\arraybackslash}X c@{}}
\toprule
\textbf{Property} & \textbf{SupplyGraph} & \textbf{\sysname{} range} & \\
\midrule

Bullwhip, delivery / sales (daily)
& 0.48 & [1.38, 1454.92] & $\times$ \\

Bullwhip, delivery / sales (weekly)
& 1.02 & [1.20, 313.98] & $\times$ \\

Bullwhip, factory / sales (daily)
& 0.79 & [65.03, 944.53] & $\times$ \\

Bullwhip, factory / sales (weekly)
& 1.44 & [10.58, 198.33] & $\times$ \\

Lead-lag, sales to delivery (d)
& 1.00 & [0.00, 2.00] & $\checkmark$ \\

Lead-lag, sales to factory (d)
& 1.00 & [-1.00, 3.00] & $\checkmark$ \\

\bottomrule
\end{tabularx}
\end{table}

Across the 42 reported properties, 18 are contained within the \sysname{}
scenario range. The coverage is strongest at the Delivery stage, where all
reported properties are contained, while the Sales Order and Factory Issue
stages show systematic differences in variability, intermittency, and temporal
dependence. Weekly aggregation improves agreement for several upstream
statistics. The cross-stage comparison shows that \sysname{} reproduces the
observed propagation timing, although its bullwhip ratios are consistently
higher than those observed in SupplyGraph at the resolutions considered.

These results should be interpreted as a statistical coverage check for one
observed supply chain. Containment indicates that a corresponding statistical
regime is reachable within the \sysname{} scenario library; it does not imply
that the released baseline reproduces SupplyGraph or that the observed range
represents the broader population of supply chains.

\section{Reproducing known supply-chain behaviors with \sysname{}
}
\label{app:additional}

\subsection{Bullwhip method}
\label{app:bullwhip-method}
For each non-source node $n$ and item $i$, we compute Cachon's
amplification ratio\citep{cachon2007search}:
\begin{equation*}
B_{n,i} = \frac{\mathrm{Var}\bigl(\mathrm{inflow}_{n,i,\cdot}\bigr)}{\mathrm{Var}\bigl(\mathrm{outflow}_{n,i,\cdot}\bigr)},
\end{equation*}
where $\mathrm{inflow}_{n,i,t}$ is the units of item $i$ arriving
at node $n$ at time unit $t$ and $\mathrm{outflow}_{n,i,t}$ is the
units shipped out of $n$ at the same time unit. At the destination,
there is no downstream node; we use the realized customer demand
$y_{i,t}$ as the boundary outflow, so $B_{\text{dest},i}$ measures
how much the inbound replenishment stream amplifies the demand it
is meant to serve. A ratio $B_{n,i} > 1$ indicates the classical
bullwhip signature — variability grows as one moves upstream from
the customer; $B_{n,i} = 1$ is no amplification, and $B_{n,i} < 1$
means $n$ smooths variability. The three source nodes
(\S\ref{app:replenish}) replenish via a magic-stock mechanism
rather than by inflow from the released network, so $B_{n,i}$ is
not defined for them and they are excluded.

Both series are aggregated into 30-time-unit totals — matching
the monthly granularity at which Cachon reports bullwhip
ratios — before computing variance, after dropping the first
365 time units as a warm-up. The per-node value reported in
Table~\ref{tab:bullwhip-validation} is the mean of $B_{n,i}$
over the $C$ items at that node; the per-tier value is the
mean over nodes in the tier.
\subsection{Per-node bullwhip ratios}
\label{app:additional:bullwhip}

Table~\ref{tab:bullwhip-validation} in
\S\ref{sec:datasets-validation} reports tier-level means. For reproducibility, Table~\ref{tab:bullwhip-pernode} below gives
the per-node values that those tier means are computed from. Each entry is the mean of $B_{n,i}$ over the $50$ items at node $n$,
using the procedure described in \S\ref{sec:datasets-validation};
tier names match Table~\ref{tab:inventory-params}.

\begin{table}[pbht]
\centering\small
\caption{Per-node bullwhip ratio $\bar B_n$ on the $C{=}50$
release. Each value is the mean over the items at that node, averaged
over the $29$-rollout cohort under the rules of
\S\ref{app:bullwhip-method}; the same numbers averaged within tier
appear in Table~\ref{tab:bullwhip-validation}. The wide spread between
Baltimore and Philadelphia in the daily column reflects
sub-monthly batching of arrivals, which the monthly column
smooths out.}
\label{tab:bullwhip-pernode}
\begin{tabular}{@{}l l c c@{}}
\toprule
Node          & Tier        & Daily $\bar B_n$ & Monthly $\bar B_n$ \\
\midrule
NewYork       & Destination & $22.64$ & $1.93$ \\
\midrule
Baltimore     & Tier-5  & $5.79$  & $1.49$ \\
Philadelphia  & Tier-5  & $19.56$ & $1.41$ \\
\midrule
Columbus      & Tier-4      & $1.35$  & $1.23$ \\
Richmond      & Tier-4      & $1.36$  & $1.26$ \\
\midrule
Charlotte     & Tier-3      & $1.14$  & $1.31$ \\
Chicago       & Tier-3      & $1.09$  & $1.17$ \\
Memphis       & Tier-3      & $1.28$  & $1.18$ \\
\midrule
Atlanta       & Tier-2      & $1.30$  & $1.14$ \\
\midrule
Nashville     & Hub         & $0.97$  & $1.07$ \\
\bottomrule
\end{tabular}
\end{table}

\subsection{Quantitative details for observable phenomena}
\label{app:additional:phenomena}

This appendix provides quantitative details supporting the phenomenon-level
figures in \S\ref{sec:datasets-validation:phenomena}. The reported values
are computed directly from the released trajectories and provide additional
measurements for the inventory oscillation, high edge-utilization,
disruption, ripple and cascading disruption, and bottleneck migration
figures. The figures remain the primary visual presentation of these
phenomena; the tables below provide the corresponding numerical summaries
and exact values where useful.

\paragraph{Additional quantitative measurements.}
For inventory oscillation, the detrended $\mathrm{AC}(7)$ of on-hand
inventory across the scenario library ranges from $-0.00$ to $0.98$,
providing the numerical range underlying Figure~\ref{fig:inventory-oscillation}.
For high edge utilization, the fraction of days with
$U_{e,t}\geq0.8$ ranges from the baseline values reported in
Figure~\ref{fig:edge-utilization} to substantially larger values under
uniform capacity reduction. The complete edge-level values are shown in
the corresponding figure and its underlying released trajectories.

For disruption and delivery lateness, Figure~\ref{fig:disruption-lateness}
combines the synthetic Global Supply Chain Disruption dataset with
a matched \sysname{} edge-cut experiment. In that dataset, the
Pearson correlation between monthly disruption-flag share and
late-delivery rate ranges from $0.808$ to $0.978$ for four of the five
lane--event pairs and is $-0.249$ for Tokyo$\rightarrow$Singapore. In the
\sysname{} experiment, removing Nashville$\rightarrow$Atlanta, which has no
alternate route, produces a larger increase in destination
unfilled-on-time demand than removing Atlanta$\rightarrow$Memphis, for
which an alternate route is available.

\paragraph{Ripple and cascading disruption.}
Table~\ref{tab:ripple} provides the exact timing measurements underlying
Figure~\ref{fig:ripple-cascade}. The Nashville$\rightarrow$Atlanta edge
(Hub$\rightarrow$Tier-2) is removed for $60$ days starting at $t$. For each
downstream tier, we report the first day on which incoming edge utilization
falls below $5\%$ of capacity and the first day on which on-hand inventory
falls below $10\%$ of its pre-disruption level. The increasing delays across
tiers quantify the progressive propagation of the disruption through the
network.

\begin{table}[h]
\centering
\footnotesize
\caption{\textbf{Ripple and cascading disruption following a temporary edge
failure.} The edge Nashville$\rightarrow$Atlanta is removed for $60$ days
starting at $t$. For each downstream tier, the table reports the first day
on which incoming edge utilization falls below $5\%$ of capacity and the
first day on which on-hand inventory falls below $10\%$ of its
pre-disruption level.}
\label{tab:ripple}
\setlength{\tabcolsep}{6pt}
\renewcommand{\arraystretch}{1.15}
\begin{tabular}{@{}lccc@{}}
\toprule
Downstream tier & Incoming edge & First day $U_e<0.05$ &
First day $\mathrm{OH}<0.1\,\mathrm{OH}_0$ \\
\midrule
Tier-2 (Atlanta)      & Nashville$\rightarrow$Atlanta     & $t{+}3$  & $t{+}12$ \\
Tier-3 (Chicago)      & Atlanta$\rightarrow$Chicago       & $t{+}9$  & $t{+}44$ \\
Tier-4 (Columbus)     & Chicago$\rightarrow$Columbus     & $t{+}45$ & $t{+}52$ \\
Tier-5 (Baltimore)    & Columbus$\rightarrow$Philadelphia & $t{+}50$ & $t{+}65$ \\
Destination (NewYork) & Philadelphia$\rightarrow$NewYork & $t{+}60$ & --- \\
\bottomrule
\end{tabular}
\end{table}

\paragraph{Bottleneck migration.}
Table~\ref{tab:bottleneck} provides the exact values underlying
Figure~\ref{fig:bottleneck-migration}. The Atlanta$\rightarrow$Memphis
edge is removed for $60$ days. Baltimore has three parallel inbound routes;
for each period, we report the fraction of days on which each inbound edge
is Baltimore's primary supplier, defined as the inbound edge with the
highest daily utilization. The resulting shift in primary-supplier share
quantifies the redistribution of flow toward alternative routes during the
disruption and the subsequent return toward the pre-cut configuration.

\begin{table}[h]
\centering
\footnotesize
\caption{\textbf{Bottleneck migration under a temporary upstream
disruption.} The edge Atlanta$\rightarrow$Memphis is removed for $60$ days.
For each period, the table reports the fraction of days on which each
Baltimore inbound edge is the primary supplier, defined as the inbound edge
with the highest daily utilization. Shares sum to $100\%$ within each
period.}
\label{tab:bottleneck}
\setlength{\tabcolsep}{7pt}
\renewcommand{\arraystretch}{1.15}
\begin{tabular}{@{}lccc@{}}
\toprule
Inbound edge to Baltimore & Pre-cut & During cut & Post-cut \\
& \multicolumn{3}{c}{\emph{Primary supplier share (\% of days)}} \\
\midrule
Memphis$\rightarrow$Baltimore  & \textbf{87} & 18 & \textbf{78} \\
Columbus$\rightarrow$Baltimore & 9 & \textbf{65} & 15 \\
Richmond$\rightarrow$Baltimore & 4 & 17 & 7 \\
\midrule
\textbf{Total} & 100 & 100 & 100 \\
\bottomrule
\end{tabular}
\end{table}

\section{Foundation model evaluation details}
\label{app:hparams}

This appendix collects the implementation details required to
reproduce the tables of \S\ref{sec:results}: the formal MASE definition
(\S\ref{app:mase-formal}), checkpoint identifiers and parameter
counts (\S\ref{app:hparams:models}), per-model inference
hyperparameters (Table~\ref{tab:hparams-infer}), and the full
per-target confidence-interval grids
(\S\ref{app:hparams:networkstate}).

\subsection{MASE: formal definition}
\label{app:mase-formal}

We use the official GIFT-Eval evaluation protocol (\texttt{calculate\_seasonal\_error}
and \texttt{mase} from the gluonts evaluation module). For each
rolling-origin test window $w$ with start time $t_w$ and channel $i$,
the \emph{in-context seasonal error} is
\begin{equation*}
    \mathrm{SE}_{i,w}
    \;=\;
    \tfrac{1}{L-m}\sum_{t=t_w-L}^{t_w-1-m}
    \bigl|\,y_{i,\,t+m} - y_{i,\,t}\,\bigr|,
\end{equation*}
where $L$ is the context length and $m{=}\texttt{get\_seasonality(freq)}$
is the GluonTS seasonal period for each native frequency
($m{=}1$ daily, $24$ hourly, $144$ for 10-minute). The model's
point forecast $\hat{y}^{\rm model}_{i,\,t_w+k}$ is the per-step
median over its sample paths. The per-entry MASE at horizon $h$
is the mean MAE over the first $h$ forecast steps divided
by the in-context seasonal error:
\begin{equation*}
    \mathrm{MASE}_{i,w}(h)
    \;=\;
    \frac{\tfrac{1}{h}\sum_{k=0}^{h-1}
        \bigl|\,y_{i,\,t_w+k} - \hat{y}^{\rm model}_{i,\,t_w+k}\,\bigr|}
        {\mathrm{SE}_{i,w}}.
\end{equation*}
The MASE score for each dataset is computed by avaraging over windows and channels:
\begin{equation}
    \mathrm{MASE}(h)
    \;=\;
    \tfrac{1}{|W|\,C}\sum_{w \in W}\sum_{i=1}^{C} \mathrm{MASE}_{i,w}(h).
    \label{eq:gift-mase}
\end{equation}

\subsection{Models and inference hyperparameters}
\label{app:hparams:models}

Table~\ref{tab:hparams-models} lists the released checkpoint, parameter
count, and pre-train context range for each model, together with the
context length $L$ used in our evaluation. Table~\ref{tab:hparams-infer} lists the inference-time
hyperparameters: the number of probabilistic sample paths drawn
(reduced to the per-step median for the point forecast), the
batching unit each wrapper operates on, and per-model
configuration knobs.

\begin{table}[h]
\centering
\caption{Foundation models evaluated. Pre-train context is the sequence
length used during published pre-training; $L$ used is the context
length in our zero-shot evaluation (see \S\ref{sec:foundation-protocol}).}
\label{tab:hparams-models}
\small
\begin{tabular}{@{}l l r r r@{}}
\toprule
Model & HuggingFace ID & Params & Pre-train ctx & $L$ used \\
\midrule
Chronos    & \texttt{amazon/chronos-t5-base}                  & 200M & 512        & 512 \\
Moirai     & \texttt{Salesforce/moirai-1.1-R-base}            & 91M  & 512--2048  & 512 \\
TimesFM    & \texttt{google/timesfm-2.0-500m-pytorch}         & 500M & up to 2048 & 512 \\
\bottomrule
\end{tabular}
\end{table}

\begin{table}[h]
\centering
\caption{Inference hyperparameters. \emph{Samples} is the number of
probabilistic sample paths drawn; the per-day median is the point
forecast used to compute MASE. TimesFM uses a deterministic quantile
head and reports no sample count. \emph{Batch} indicates the unit
batched: Chronos batches \emph{channels} within a window,
Moirai is multivariate-native and batches \emph{windows}, TimesFM
batches \emph{channels} per call.}
\label{tab:hparams-infer}
\small
\begin{tabular}{@{}l l l l@{}}
\toprule
Model & Samples & Batch & Model-specific \\
\midrule
Chronos    & 20            & channel\_batch$=$16            & bf16 inference \\
Moirai     & 100           & 8 ($C{=}50$) / 1 ($C{=}200$)   & patch\_size$=$32; OOM-bound at $C{=}200$ \\
TimesFM    & ---           & per\_core\_batch$=$32          & horizon\_len$=$128 (decode 1 patch, slice to $H$) \\
\bottomrule
\end{tabular}
\end{table}

\subsection{Network-state forecast targets: full grids}
\label{app:hparams:networkstate}

Tables~\ref{tab:combined-ci-backlog}--\ref{tab:combined-ci-edge_util} give the
complete $7$-model $\times$ $4$-horizon $\times$ $3$-seasonal-lag grids for
backlog, fill rate, and edge utilization, on the same $29$-rollout evaluation
cohort and under the same rolling-window protocol as
Table~\ref{tab:combined-ci-demand} (demand). The row labeled
\emph{SN($m$) MAE} is the seasonal-naive denominator in the units of the
target, and is included so that a MASE value can be converted back to an
absolute error: MASE alone is not comparable across targets whose scales
differ by four orders of magnitude (fill rate is $O(10^{-3})$, backlog is
$O(10^{2})$).

Two entries deserve explicit comment rather than being left for the reader to
find. ARIMA's backlog and fill-rate columns at $h \ge 7$ are reported in
scientific notation because the fitted models diverge, not because of a
formatting choice; the values ($10^{2}$ to $10^{17}$) are the actual computed
MASE and are retained because suppressing them would hide the diagnostic. And
the backlog and fill-rate tables use an effective $n{=}24$ rather than $29$:
five rollouts in the cohort are under-loaded to the point that backlog is
identically zero and fill rate identically one, which makes MASE undefined
(zero denominator), so those rollouts are dropped from those two targets and
retained for demand and edge utilization. They are not dropped for poor
performance.

\begin{table}[t]
\centering\footnotesize
\caption{Univariate MASE on \textbf{backlog}, mean $\pm$ 95\% Student-$t$ CI over the $n{=}29$-rollout cohort. 5 rollouts are excluded from this target's CI because MASE is undefined when the series is constant (under-loaded scenarios yield zero backlog / unit fill rate), giving effective $n{=}24$. \emph{In-domain trained} baselines (ARIMA, ETS, PatchTST) are fit per-rollout; foundation models (Chronos, Moirai, TimesFM) run \emph{zero-shot}. Agg. is the geometric mean across horizons.}
\label{tab:combined-ci-backlog}
\setlength{\tabcolsep}{4pt}\renewcommand{\arraystretch}{1.1}
\begin{tabular}{@{}l ccccc@{}}
\toprule
Model & $h{=}1$ & $h{=}7$ & $h{=}14$ & $h{=}30$ & Agg. \\
\midrule
\multicolumn{6}{@{}l}{\textbf{MASE} ($m{=}1$)} \\
ARIMA & 3.388\,$\pm$\,1.780 & 5.1e+02 & 8.4e+06 & 4.9e+17 & 9.2e+06 \\
ETS & 2.531\,$\pm$\,0.987 & 4.873\,$\pm$\,2.280 & 7.309\,$\pm$\,3.375 & 12.848\,$\pm$\,6.732 & 5.834\,$\pm$\,2.684 \\
PatchTST & 5.622\,$\pm$\,1.630 & 7.209\,$\pm$\,2.430 & 8.443\,$\pm$\,2.888 & 10.218\,$\pm$\,4.168 & 7.690\,$\pm$\,2.642 \\
Chronos & 2.109\,$\pm$\,1.164 & 3.876\,$\pm$\,2.343 & 4.860\,$\pm$\,2.616 & 6.361\,$\pm$\,3.946 & 3.987\,$\pm$\,2.306 \\
Moirai & 1.570\,$\pm$\,0.910 & 3.431\,$\pm$\,2.174 & 4.566\,$\pm$\,2.467 & 6.427\,$\pm$\,3.988 & 3.546\,$\pm$\,2.103 \\
TimesFM & 5.157\,$\pm$\,2.579 & 6.236\,$\pm$\,3.229 & 8.038\,$\pm$\,4.165 & 11.914\,$\pm$\,6.788 & 7.449\,$\pm$\,3.920 \\
\emph{SN($m{=}1$) MAE} & 59.59 & 58.79 & 59.13 & 61.66 & 59.79 \\
\midrule
\multicolumn{6}{@{}l}{\textbf{MASE} ($m{=}7$)} \\
ARIMA & 2.760\,$\pm$\,1.591 & 4.2e+02 & 7.1e+06 & 4.2e+17 & 7.7e+06 \\
ETS & 1.853\,$\pm$\,0.807 & 3.459\,$\pm$\,1.891 & 4.985\,$\pm$\,2.638 & 8.532\,$\pm$\,5.263 & 4.063\,$\pm$\,2.156 \\
PatchTST & 3.455\,$\pm$\,1.003 & 4.634\,$\pm$\,1.719 & 5.321\,$\pm$\,1.986 & 6.435\,$\pm$\,3.309 & 4.839\,$\pm$\,1.862 \\
Chronos & 1.047\,$\pm$\,0.616 & 2.333\,$\pm$\,1.795 & 2.901\,$\pm$\,1.922 & 3.890\,$\pm$\,3.315 & 2.291\,$\pm$\,1.639 \\
Moirai & 0.958\,$\pm$\,0.636 & 2.268\,$\pm$\,1.806 & 2.882\,$\pm$\,1.930 & 3.935\,$\pm$\,3.327 & 2.228\,$\pm$\,1.653 \\
TimesFM & 4.081\,$\pm$\,1.987 & 4.740\,$\pm$\,2.444 & 5.859\,$\pm$\,3.030 & 8.342\,$\pm$\,5.105 & 5.545\,$\pm$\,2.951 \\
\emph{SN($m{=}7$) MAE} & 270.1 & 265.7 & 261.8 & 271.6 & 267.3 \\
\midrule
\multicolumn{6}{@{}l}{\textbf{MASE} ($m{=}30$)} \\
ARIMA & 2.664\,$\pm$\,1.521 & 4e+02 & 6.8e+06 & 4e+17 & 7.3e+06 \\
ETS & 1.801\,$\pm$\,0.776 & 3.308\,$\pm$\,1.803 & 4.744\,$\pm$\,2.570 & 7.964\,$\pm$\,5.045 & 3.873\,$\pm$\,2.076 \\
PatchTST & 3.372\,$\pm$\,1.044 & 4.487\,$\pm$\,1.660 & 5.154\,$\pm$\,1.955 & 6.202\,$\pm$\,3.179 & 4.690\,$\pm$\,1.833 \\
Chronos & 0.897\,$\pm$\,0.573 & 2.134\,$\pm$\,1.714 & 2.686\,$\pm$\,1.852 & 3.553\,$\pm$\,3.181 & 2.067\,$\pm$\,1.559 \\
Moirai & 0.862\,$\pm$\,0.600 & 2.114\,$\pm$\,1.728 & 2.698\,$\pm$\,1.863 & 3.593\,$\pm$\,3.191 & 2.050\,$\pm$\,1.581 \\
TimesFM & 4.422\,$\pm$\,1.974 & 4.924\,$\pm$\,2.347 & 5.993\,$\pm$\,2.940 & 8.343\,$\pm$\,4.895 & 5.744\,$\pm$\,2.868 \\
\emph{SN($m{=}30$) MAE} & 598.2 & 576.7 & 559.3 & 575.5 & 577.4 \\
\bottomrule
\end{tabular}
\end{table}

\begin{table}[t]
\centering\footnotesize
\caption{Univariate MASE on \textbf{fill\_rate}, mean $\pm$ 95\% Student-$t$ CI over the $n{=}29$-rollout cohort. 5 rollouts are excluded from this target's CI because MASE is undefined when the series is constant (under-loaded scenarios yield zero backlog / unit fill rate), giving effective $n{=}24$. \emph{In-domain trained} baselines (ARIMA, ETS, PatchTST) are fit per-rollout; foundation models (Chronos, Moirai, TimesFM) run \emph{zero-shot}. Agg. is the geometric mean across horizons.}
\label{tab:combined-ci-fill_rate}
\setlength{\tabcolsep}{4pt}\renewcommand{\arraystretch}{1.1}
\begin{tabular}{@{}l ccccc@{}}
\toprule
Model & $h{=}1$ & $h{=}7$ & $h{=}14$ & $h{=}30$ & Agg. \\
\midrule
\multicolumn{6}{@{}l}{\textbf{MASE} ($m{=}1$)} \\
ARIMA & 3.763\,$\pm$\,1.766 & 8.1e+02 & 9e+06 & 4e+16 & 5.7e+06 \\
ETS & 3.923\,$\pm$\,1.527 & 7.747\,$\pm$\,3.368 & 11.946\,$\pm$\,5.754 & 20.657\,$\pm$\,10.791 & 9.306\,$\pm$\,4.245 \\
PatchTST & 10.648\,$\pm$\,3.439 & 12.503\,$\pm$\,3.687 & 13.814\,$\pm$\,3.971 & 14.518\,$\pm$\,3.783 & 12.783\,$\pm$\,3.723 \\
Chronos & 2.902\,$\pm$\,0.951 & 3.837\,$\pm$\,1.204 & 4.592\,$\pm$\,1.457 & 5.578\,$\pm$\,2.166 & 4.110\,$\pm$\,1.383 \\
Moirai & 2.308\,$\pm$\,0.993 & 2.939\,$\pm$\,1.249 & 3.391\,$\pm$\,1.284 & 4.185\,$\pm$\,1.853 & 3.132\,$\pm$\,1.312 \\
TimesFM & 4.902\,$\pm$\,2.266 & 5.979\,$\pm$\,2.877 & 7.192\,$\pm$\,3.878 & 9.066\,$\pm$\,5.558 & 6.612\,$\pm$\,3.457 \\
\emph{SN($m{=}1$) MAE} & 0.0027 & 0.0027 & 0.0031 & 0.0033 & 0.0029 \\
\midrule
\multicolumn{6}{@{}l}{\textbf{MASE} ($m{=}7$)} \\
ARIMA & 3.015\,$\pm$\,1.591 & 6.9e+02 & 7.6e+06 & 3.3e+16 & 4.8e+06 \\
ETS & 2.946\,$\pm$\,1.446 & 5.634\,$\pm$\,3.000 & 8.799\,$\pm$\,5.074 & 15.801\,$\pm$\,9.512 & 6.931\,$\pm$\,3.811 \\
PatchTST & 6.554\,$\pm$\,2.788 & 7.663\,$\pm$\,3.162 & 8.300\,$\pm$\,3.322 & 8.495\,$\pm$\,2.831 & 7.714\,$\pm$\,3.027 \\
Chronos & 2.107\,$\pm$\,0.792 & 2.470\,$\pm$\,0.868 & 2.798\,$\pm$\,0.929 & 3.349\,$\pm$\,1.330 & 2.643\,$\pm$\,0.962 \\
Moirai & 1.007\,$\pm$\,0.463 & 1.393\,$\pm$\,0.754 & 1.605\,$\pm$\,0.683 & 2.091\,$\pm$\,1.117 & 1.473\,$\pm$\,0.721 \\
TimesFM & 3.875\,$\pm$\,2.181 & 4.499\,$\pm$\,2.731 & 5.314\,$\pm$\,3.634 & 6.857\,$\pm$\,5.261 & 5.021\,$\pm$\,3.281 \\
\emph{SN($m{=}7$) MAE} & 0.0066 & 0.0073 & 0.0081 & 0.0079 & 0.0075 \\
\midrule
\multicolumn{6}{@{}l}{\textbf{MASE} ($m{=}30$)} \\
ARIMA & 2.894\,$\pm$\,1.513 & 6.6e+02 & 7.3e+06 & 3.2e+16 & 4.6e+06 \\
ETS & 2.833\,$\pm$\,1.381 & 5.380\,$\pm$\,2.850 & 8.377\,$\pm$\,4.838 & 15.146\,$\pm$\,9.049 & 6.631\,$\pm$\,3.630 \\
PatchTST & 6.488\,$\pm$\,2.987 & 7.433\,$\pm$\,3.279 & 8.003\,$\pm$\,3.424 & 8.338\,$\pm$\,2.912 & 7.532\,$\pm$\,3.156 \\
Chronos & 2.047\,$\pm$\,0.764 & 2.391\,$\pm$\,0.829 & 2.691\,$\pm$\,0.892 & 3.317\,$\pm$\,1.331 & 2.571\,$\pm$\,0.933 \\
Moirai & 0.905\,$\pm$\,0.435 & 1.278\,$\pm$\,0.710 & 1.469\,$\pm$\,0.643 & 2.034\,$\pm$\,1.158 & 1.363\,$\pm$\,0.695 \\
TimesFM & 3.775\,$\pm$\,2.109 & 4.333\,$\pm$\,2.613 & 5.092\,$\pm$\,3.474 & 6.663\,$\pm$\,5.023 & 4.854\,$\pm$\,3.144 \\
\emph{SN($m{=}30$) MAE} & 0.0076 & 0.0083 & 0.0095 & 0.0092 & 0.0087 \\
\bottomrule
\end{tabular}
\end{table}

\begin{table}[t]
\centering\footnotesize
\caption{Univariate MASE on \textbf{edge\_util}, mean $\pm$ 95\% Student-$t$ CI over the $n{=}29$-rollout cohort. \emph{In-domain trained} baselines (ARIMA, ETS, PatchTST) are fit per-rollout; foundation models (Chronos, Moirai, TimesFM) run \emph{zero-shot}. Agg. is the geometric mean across horizons.}
\label{tab:combined-ci-edge_util}
\setlength{\tabcolsep}{4pt}\renewcommand{\arraystretch}{1.1}
\begin{tabular}{@{}l ccccc@{}}
\toprule
Model & $h{=}1$ & $h{=}7$ & $h{=}14$ & $h{=}30$ & Agg. \\
\midrule
\multicolumn{6}{@{}l}{\textbf{MASE} ($m{=}1$)} \\
ARIMA & 0.791\,$\pm$\,0.031 & 0.815\,$\pm$\,0.029 & 0.820\,$\pm$\,0.028 & 0.834\,$\pm$\,0.031 & 0.815\,$\pm$\,0.030 \\
ETS & 0.798\,$\pm$\,0.022 & 0.820\,$\pm$\,0.028 & 0.827\,$\pm$\,0.027 & 0.843\,$\pm$\,0.031 & 0.822\,$\pm$\,0.027 \\
PatchTST & 0.898\,$\pm$\,0.036 & 0.910\,$\pm$\,0.046 & 0.909\,$\pm$\,0.045 & 0.921\,$\pm$\,0.046 & 0.910\,$\pm$\,0.043 \\
Chronos & 0.736\,$\pm$\,0.026 & 0.757\,$\pm$\,0.016 & 0.758\,$\pm$\,0.018 & 0.773\,$\pm$\,0.019 & 0.756\,$\pm$\,0.020 \\
Moirai & 0.721\,$\pm$\,0.024 & 0.747\,$\pm$\,0.015 & 0.747\,$\pm$\,0.016 & 0.758\,$\pm$\,0.018 & 0.743\,$\pm$\,0.018 \\
TimesFM & 0.812\,$\pm$\,0.072 & 0.831\,$\pm$\,0.108 & 0.857\,$\pm$\,0.151 & 0.844\,$\pm$\,0.114 & 0.836\,$\pm$\,0.111 \\
\emph{SN($m{=}1$) MAE} & 0.1145 & 0.1140 & 0.1136 & 0.1129 & 0.1137 \\
\midrule
\multicolumn{6}{@{}l}{\textbf{MASE} ($m{=}7$)} \\
ARIMA & 0.779\,$\pm$\,0.040 & 0.798\,$\pm$\,0.031 & 0.801\,$\pm$\,0.029 & 0.814\,$\pm$\,0.030 & 0.798\,$\pm$\,0.033 \\
ETS & 0.777\,$\pm$\,0.016 & 0.797\,$\pm$\,0.024 & 0.802\,$\pm$\,0.021 & 0.816\,$\pm$\,0.023 & 0.798\,$\pm$\,0.021 \\
PatchTST & 0.874\,$\pm$\,0.030 & 0.886\,$\pm$\,0.041 & 0.886\,$\pm$\,0.039 & 0.897\,$\pm$\,0.041 & 0.886\,$\pm$\,0.038 \\
Chronos & 0.718\,$\pm$\,0.024 & 0.736\,$\pm$\,0.013 & 0.737\,$\pm$\,0.013 & 0.750\,$\pm$\,0.013 & 0.735\,$\pm$\,0.016 \\
Moirai & 0.702\,$\pm$\,0.023 & 0.724\,$\pm$\,0.012 & 0.724\,$\pm$\,0.012 & 0.734\,$\pm$\,0.011 & 0.721\,$\pm$\,0.014 \\
TimesFM & 0.794\,$\pm$\,0.071 & 0.812\,$\pm$\,0.107 & 0.837\,$\pm$\,0.149 & 0.823\,$\pm$\,0.113 & 0.817\,$\pm$\,0.109 \\
\emph{SN($m{=}7$) MAE} & 0.1176 & 0.1173 & 0.1168 & 0.1161 & 0.1170 \\
\midrule
\multicolumn{6}{@{}l}{\textbf{MASE} ($m{=}30$)} \\
ARIMA & 0.768\,$\pm$\,0.038 & 0.786\,$\pm$\,0.030 & 0.788\,$\pm$\,0.026 & 0.799\,$\pm$\,0.028 & 0.785\,$\pm$\,0.031 \\
ETS & 0.767\,$\pm$\,0.015 & 0.785\,$\pm$\,0.023 & 0.789\,$\pm$\,0.018 & 0.802\,$\pm$\,0.021 & 0.786\,$\pm$\,0.019 \\
PatchTST & 0.861\,$\pm$\,0.029 & 0.873\,$\pm$\,0.040 & 0.872\,$\pm$\,0.038 & 0.883\,$\pm$\,0.039 & 0.872\,$\pm$\,0.036 \\
Chronos & 0.708\,$\pm$\,0.023 & 0.725\,$\pm$\,0.012 & 0.725\,$\pm$\,0.011 & 0.738\,$\pm$\,0.010 & 0.724\,$\pm$\,0.014 \\
Moirai & 0.692\,$\pm$\,0.022 & 0.712\,$\pm$\,0.010 & 0.712\,$\pm$\,0.009 & 0.721\,$\pm$\,0.008 & 0.709\,$\pm$\,0.012 \\
TimesFM & 0.784\,$\pm$\,0.070 & 0.804\,$\pm$\,0.112 & 0.829\,$\pm$\,0.157 & 0.814\,$\pm$\,0.118 & 0.808\,$\pm$\,0.113 \\
\emph{SN($m{=}30$) MAE} & 0.1193 & 0.1192 & 0.1187 & 0.1181 & 0.1188 \\
\bottomrule
\end{tabular}
\end{table}

\section{Computational resources}
\label{app:compute}

\textbf{Time-series generation.} Simulator runs use a single CPU core and no
GPU. A $C{=}50$ rollout takes roughly 20 minutes;
$C{=}200$ roughly 90 minutes.
The release comprises $34$ generated rollouts: the $C{=}50$ and $C{=}200$
baselines, $24$ one-knob scenario settings at $C{=}50$ (the six sweeps
contribute four non-baseline settings each, and the per-sweep baseline is
shared with the main $C{=}50$ release rather than re-simulated), the two
compound scenarios \texttt{chaos\_compound} and \texttt{chaos\_burst}, the two
dynamic edge-cut scenarios, and the four extension demonstrations of
App.~\ref{app:extensions}. The forward-UQ ensemble of \S\ref{sec:uq} reuses
$20$ of these rollouts and generates none of its own. Wall-clock on the order
of $12$ CPU hours.

\textbf{Foundation-model inference.} All zero-shot inference ran on a
single \textsc{NVIDIA RTX 2080 Ti}, one job per GPU. The pipeline
covers the TSF evaluation of \S\ref{sec:foundation} (6 forecasters
$\times$ 4 targets $\times$ 29 rollouts) and the forward-UQ scoring of
\S\ref{sec:uq} over 3 foundation models $\times$ 20 of those rollouts.
Aggregate wall-clock is on the order of 30 GPU hours:
roughly 22 hours for the TSF evaluation (each job runs the full
$|W|{=}36$ rolling-origin protocol of \S\ref{sec:foundation-protocol} at
stride 30) and roughly 7 hours for the UQ ensemble.

\end{document}